\documentclass{article}

\PassOptionsToPackage{numbers}{natbib}
\usepackage[main, final]{neurips_2026}

\usepackage[utf8]{inputenc} %
\usepackage[T1]{fontenc}    %
\usepackage{hyperref}       %
\hypersetup{
  pdftitle={Your Neighbors Know: Leveraging Local Neighborhoods for Backdoor Detection in Decentralized Learning},
  pdfauthor={Sayan Biswas, Antoine Boutet, Davide Frey, Romaric Gaudel, Rachid Guerraoui, Maxime Jacovella, Anne-Marie Kermarrec, Dimitri Lerévérend, François Taïani, Martijn de Vos},
  pdfsubject={NeurIPS 2026},
}
\usepackage{url}            %
\usepackage{booktabs}       %
\usepackage{amsfonts}       %
\usepackage{nicefrac}       %
\usepackage{microtype}      %
\usepackage{xcolor}         %

\title{Your Neighbors Know: Leveraging Local Neighborhoods for Backdoor Detection in Decentralized Learning}

\author{%
  Sayan Biswas\textsuperscript{1} \And
  Antoine Boutet\textsuperscript{2} \And
  Davide Frey\textsuperscript{3} \And
  Romaric Gaudel\textsuperscript{3} \And
  Rachid Guerraoui\textsuperscript{1} \And
  Maxime Jacovella\textsuperscript{1,}\thanks{Corresponding author: \texttt{maxime.jacovella@epfl.ch}.} \And
  Anne-Marie Kermarrec\textsuperscript{1} \And
  Dimitri Lerévérend\textsuperscript{3}\And
  François Taïani\textsuperscript{3}\And
  Martijn de Vos\textsuperscript{1}\\[0.8em]
  \textsuperscript{1}EPFL \quad
  \textsuperscript{2}Inria, INSA Lyon, CITI \quad
  \textsuperscript{3}Univ. Rennes, Inria, CNRS, IRISA
  \\[0.5em]
  \textit{Authors are listed in alphabetical order.}
}

\usepackage{graphicx}
\graphicspath{ {figures/} }

\usepackage{amsmath,amssymb,amsfonts,amsthm}
\usepackage[ruled,lined,boxed,commentsnumbered, noend, linesnumbered, procnumbered]{algorithm2e}
\usepackage[capitalize,noabbrev]{cleveref}
\usepackage{siunitx}
\usepackage{booktabs}
\usepackage{multirow}
\usepackage{tikz}
\usepackage{diagbox}
\usepackage{makecell}
\usepackage[inline]{enumitem}
\usepackage{subcaption}
\usepackage[abbreviations]{foreign}
\usepackage{colortbl}
\usepackage{float}

\usepackage{wrapfig}
\usepackage{dsfont}

\crefname{assumption}{Assumption}{Assumptions}

\Crefname{theorem}{Theorem}{Theorem}
\Crefname{definition}{Definition}{Definition}
\Crefname{lemma}{Lemma}{Lemma}
\Crefname{equation}{Equation}{Equation}
\Crefname{section}{Section}{Section}
\Crefname{figure}{Figure}{Figure}
\Crefname{algorithm}{Algorithm}{Algorithm}
\Crefname{remark}{Remark}{Remark}
\usepackage{acronym}
\newcommand{\sys}[0]{\textsc{Argus}\xspace}
\acrodef{ML}{machine learning}
\acrodef{FL}{federated learning}
\acrodef{DL}{decentralized learning}
\acrodef{SGD}{stochastic gradient descent}
\acrodef{D-PSGD}{decentralized parallel stochastic gradient descent}
\acrodef{IID}{Independent and Identically Distributed}
\acrodef{NIID}{Non-Independent and Non-Identically Distributed}
\acrodef{ASR}{Attack Success Rate}
\acrodef{FPR}{False Positive Rate}
\acrodef{FNR}{False Negative Rate}
\acrodef{TPR}{True Positive Rate}
\acrodef{TP}{True Positive}
\acrodef{FP}{False Positive}
\acrodef{FN}{False Negative}
\acrodef{CEL}{Cross-entropy loss}
\acrodef{SSIM}{Structural similarity}
\acrodef{EL}{Epidemic Learning}
\acrodef{CNN}{Convolutional Neural Networks}
\acrodef{CA}{Clean Test Accuracy}
\acrodef{Rej. Rate}{Rejection Rate}
\acrodef{EL}{Epidemic Learning}

\newcommand{\cifar}{CIFAR-10\xspace}
\newcommand{\imgnet}{TinyImageNet\xspace}
\newcommand{\femnist}{FEMNIST\xspace}

\newcommand{\Trusted}{\textsc{Trusted}\xspace}
\newcommand{\Suspected}{\textsc{Suspected}\xspace}
\newcommand{\Ejected}{\textsc{Ejected}\xspace}
\newcommand{\Nodef}{\textsc{No defense}\xspace}
\newcommand{\Oracle}{\textsc{Oracle}\xspace}
\newcommand{\Localdef}{\textsc{Local Def.}\xspace}
\newcommand{\badfl}{\textsc{BaDFL}\xspace}
\newcommand{\ppcd}{\textsc{P2PCD}\xspace}
\newcommand{\detrigger}{\textsc{DeTrigger}\xspace}
\newcommand{\multikrum}{\textsc{Multi-Krum}\xspace}
\newcommand{\iid}{\ac{IID}\xspace}
\newcommand{\niid}{\ac{NIID}\xspace}
\usepackage{tikz}
\usepackage{pgfplots}
\usepackage{pgfplotstable}
\usepackage[eulergreek]{sansmath}
\newcolumntype{L}[1]{>{\raggedright\let\newline\\\arraybackslash\hspace{0pt}}m{#1}}
\newcolumntype{C}[1]{>{\centering\let\newline\\\arraybackslash\hspace{0pt}}m{#1}}
\newcolumntype{R}[1]{>{\raggedleft\let\newline\\\arraybackslash\hspace{0pt}}m{#1}}

\pgfplotsset{compat=newest}
\usepgfplotslibrary{external,units,colorbrewer,groupplots,fillbetween}
\usetikzlibrary{patterns, automata, arrows.meta}
\usetikzlibrary{shapes.geometric, arrows, positioning}

\makeatletter

\newcommand{\newgroupwidth}[2]%
{\expandafter\xdef\csname groupwidth#1\endcsname{#2}}

\newcounter{groupwidth}
\newsavebox{\groupwidthbox}
\makeatletter
{\edef\groupnumber{#1}%
	\stepcounter{groupwidth}%
	\@ifundefined{groupwidth\thegroupwidth}{\pgfmathsetlengthmacro{\mywidth}{\linewidth/\groupnumber}}%
	{\expandafter\let\expandafter\mywidth\csname groupwidth\thegroupwidth\endcsname}%
	\begin{lrbox}{\groupwidthbox}%
		\tikzset{/pgfplots/width={\mywidth}}%
		\ignorespaces}%
	{\end{lrbox}%
	\usebox\groupwidthbox
	\pgfmathsetlengthmacro{\mywidth}{\mywidth + (\linewidth - \wd\groupwidthbox)/\groupnumber}
	\immediate\write\@auxout{\string\newgroupwidth{\thegroupwidth}{\mywidth}}}
\makeatother
\usepackage{physics}
\usepackage{thm-restate}
\usepackage{bbm}
\newtheorem{theorem}{Theorem}
\newtheorem{lemma}[theorem]{Lemma}
\newtheorem{corollary}[theorem]{Corollary}
\newtheorem{remark}[theorem]{Remark}
\newtheorem{assumption}[theorem]{Assumption}

\allowdisplaybreaks

\newtheorem{proposition}[theorem]{Proposition}
\newtheorem{observation}[theorem]{Insight}

\newcommand{\R}{\mathbb{R}}
\newcommand{\Pp}{\mathbb{P}}
\newcommand{\E}{\mathbb{E}}

\newcommand{\boundstochasticnoise}[0]{\sigma}
\newcommand{\boundheterogeneity}[0]{\varsigma}
\newcommand{\probareject}[0]{p_{\mathrm{fp}}}
\newcommand{\probafalsenegative}[0]{p_{\mathrm{fn}}}
\newcommand{\corrreject}[0]{\varrho_{\mathrm{fp}}}
\newcommand{\corrfalseneg}[0]{\varrho_{\mathrm{fn}}}
\newcommand{\phiprobareject}[0]{\phi_{\probareject}}
\newcommand{\phiprobarejectat}[1]{\phiprobareject\!\left(#1\right)}
\newcommand{\psiprobareject}[0]{\psi_{\probareject}}
\newcommand{\psiprobarejectat}[1]{\psiprobareject\!\left(#1\right)}
\newcommand{\losssampled}[0]{\mathcal{L}}
\newcommand{\locallosssampled}[1]{\losssampled_{#1}}
\newcommand{\model}{\theta}

\newcommand{\half}{\nicefrac{1}{2}}
\newcommand{\nodeset}[0]{\mathcal{V}}
\newcommand{\maliciousset}[0]{\mathcal{M}}
\newcommand{\honestset}[0]{\mathcal{H}}
\newcommand{\condesp}[2]{\E\left[#1 \mid #2\right]}

\newcommand{\receivedset}[1]{\mathcal{A}_{#1}}
\newcommand{\nbnodes}[0]{n}
\newcommand{\len}[1]{\left|#1\right|}
\newcommand{\nbreceived}[1]{\len{\receivedset{#1}}}
\newcommand{\1}{\mathbf{1}}
\newcommand{\lrmain}{\eta}
\newcommand{\lrmask}{\tilde{\eta}} 
\begin{document}

\maketitle

\begin{abstract}
Decentralized learning (DL) is an emerging machine learning paradigm where nodes collaboratively train models without a central server.
However, the collaborative nature of DL makes it vulnerable to backdoor attacks, where a model is taught to behave normally on standard inputs while executing hidden, malicious actions when encountering data with specific triggers.
Backdoor attacks in DL remain understudied and existing defenses often overlook DL constraints.
We introduce \textsc{Argus}, a novel backdoor detection framework native to DL that requires neither a central coordinator nor prior knowledge of the trigger.
In \textsc{Argus}, honest nodes locally analyze received model updates to identify potential backdoor triggers.
Nodes then collectively share their triggers with their neighbors and use a structural similarity metric to separate true backdoors from false alarms induced by data heterogeneity.
A key insight is that false positive triggers exhibit inconsistencies across participants while true positive ones show consistent patterns.
Model updates that fail this collaborative test are rejected, and persistently malicious senders are eventually evicted.
We provide the first theoretical convergence guarantees for a DL-specific backdoor detection mechanism, showing that filtering out suspicious model updates with high probability preserves a convergence rate comparable to standard DL.
We implement and evaluate \textsc{Argus} on three standard datasets and against three state-of-the-art baselines.
Across settings, \textsc{Argus} reduces attack success rates by up to 90 points compared to no defense, while preserving model utility within 5 percentage points of an omniscient oracle.
Furthermore, the effectiveness of \textsc{Argus} compared to baselines improves as data heterogeneity increases.
\end{abstract}

\section{Introduction}\label{sec:introduction}

\Acf{DL} is a collaborative \ac{ML} approach that enables a set of nodes to train an \ac{ML} model
collectively without the need for any central server
\cite{lian2017advances,beltran2023decentralized,devosEpidemicLearningBoosting2023}.
In each round of \ac{DL}, nodes independently train their model on their local private datasets and exchange the resulting model updates with their neighbors according to a topology.
Each node then aggregates the received updates locally, and the resulting model is used as a starting point for the next round.
This process repeats until the models converge.
Because \ac{DL} requires no central coordinator, it is inherently scalable and fault-tolerant, and removes the need to trust any single party.
\ac{DL} is increasingly adopted in high-stakes domains such as healthcare, energy and finance~\cite{beltran2023decentralized,shiranthika2023decentralized}.

Unfortunately, the collaborative nature of \ac{DL} makes it vulnerable to attacker nodes attempting to undermine model training by injecting adversarial updates that hurt model utility for other nodes~\cite{Troncoso_2017,fang2022bridge}.
One such example is \emph{backdoor attacks}, by which attacker nodes manipulate their local models to misclassify inputs containing a specific trigger pattern, redirecting them to a target label of their choice.
\begin{wrapfigure}{r}{0.35\textwidth}
  \centering
  \begin{tikzpicture}
    \begin{axis}[
        width=5.3cm,
        height=4cm,
        xlabel={Round},
        ylabel={ASR (\%)},
        grid=major,
        grid style={gray!20},
        tick label style={font=\scriptsize},
        label style={font=\small},
        legend style={
            font=\scriptsize,
            at={(0.37,0.37)},
            anchor=north west,
            row sep=-2pt,
            legend cell align=left,
            /tikz/every even column/.append style={column sep=0pt},
            legend image post style={scale=0.5}
        },
        ymin=0, ymax=100,
        xmin=0, xmax=160,
    ]

    \addplot[thick, color=orange!80!black, densely dashed]
        table[x=round, y=none_hop1] {plots/propagation_asr.dat};
    \addlegendentry{Direct neighbors}

    \addplot[thick, color=orange!80!black]
        table[x=round, y=none_avg] {plots/propagation_asr.dat};
    \addlegendentry{All honest nodes}
    
    \end{axis}
    \end{tikzpicture}
  \caption{The average \acf{ASR} of the backdoor attack in a 16-node network with one attacker, using the \cifar dataset in a NIID setting.}
  \label{fig:propagation_intro}
  \vspace{-1em}
\end{wrapfigure}
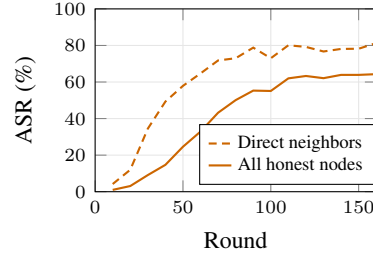
This is done without compromising the model's accuracy on normal, clean inputs~\cite{gu2019badnets}.
In \ac{DL}, such backdoor triggers are continuously injected and spread across the network through seemingly legitimate model updates.
\Cref{fig:propagation_intro} illustrates the effectiveness of backdoor attacks in \ac{DL} with 16 nodes and just one attacker (see \Cref{subsec:exp_prop} for further experimental details) in a \ac{NIID} setting on the \cifar dataset.
The models of honest nodes are significantly compromised, reaching over 60\% average \acf{ASR} and thus showing that backdoors remain a serious threat in \ac{DL}.
The average \ac{ASR} for nodes directly connected to the attacker is even more pronounced, ultimately exceeding 80\%.
While backdoor attacks and defenses have been extensively studied in \ac{FL}~\cite{gong2022backdoor,li2025backdoor}, these solutions are largely inapplicable in \ac{DL} because of the absence of a central server.
Therefore, backdoor defense in \ac{DL} remains unsolved.

We address this shortcoming and introduce \sys, a novel framework for backdoor detection in \ac{DL} that operates without any central coordinator and without requiring prior knowledge of the backdoor trigger. 
With \sys, nodes detect and filter backdoored model updates before they are aggregated into local models, and eventually identify and evict attacker nodes from participating.
\Cref{fig:sys_overview} shows the overall workflow of \sys as executed by an honest node.
\sys works in two phases:
\begin{enumerate*}[label= \emph{(\roman*)}]
    \item a \emph{local trigger detection} phase (Step 3), in which each node independently attempts to reverse-engineer a backdoor trigger from each received update, using only their local dataset; and
    \item a \emph{collaborative verification} phase (Steps 4--5), in which nodes cross-validate their recovered triggers with the ones found by their neighbors.
\end{enumerate*}
Depending on the similarity between triggers, incoming updates are accepted or rejected (Step 6).
This collaborative verification phase is crucial: data heterogeneity causes local detection to produce some false positives, and those triggers recovered from honest nodes are structurally inconsistent across detecting nodes whereas true backdoor triggers are not, allowing \sys to cleanly separate the two.
By leveraging the collaborative nature of \ac{DL}, \sys effectively suppresses backdoor attacks without compromising model utility.
We implement \sys and conduct extensive experiments on three standard vision datasets (\cifar, \femnist, and \imgnet), against three state-of-the-art baselines (\multikrum, a Byzantine-robust aggregation, and the two existing backdoor defenses targeting \ac{DL}: \badfl and the clipping defense of \citet{syros2025backdoor}), and across varying levels of data heterogeneity, network topologies, and attacker node counts.
We find that \sys reduces the \ac{ASR} from over $70\%$ to below $7\%$ across all settings, while retaining test accuracy on non-backdoored samples within $5$ percentage points of an omniscient oracle defense.

\textbf{Contributions.} In summary, our main contributions are:
\begin{enumerate}
    \item We introduce \sys, the first backdoor detection framework where nodes in \ac{DL} reverse-engineer and collaborate to identify potential backdoor triggers (\Cref{sec:system}). 
    \item We provide a convergence analysis of \sys, showing a competitive convergence rate alongside a characterization of the effect of false positives introduced by \sys (\Cref{sec:theory}).
    To the best of our knowledge, this is the first analysis for \ac{DL} with \emph{asymmetric} link failure that captures the impact of the failure probability in the convergence bound.

    \item We implement \sys and evaluate it on three standard datasets and against three state-of-the-art baselines (\Cref{sec:experiments}). Our experiments demonstrate the effectiveness of \sys in detecting backdoor attacks while preserving model utility, especially in \ac{NIID} settings.
\end{enumerate}

\begin{figure*}[t]
    \centering
    \includegraphics[width=0.95\textwidth]{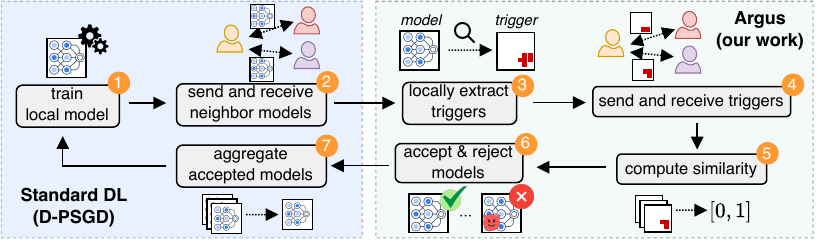}
    \caption{The workflow of \sys during a single round in \ac{DL} as executed by an honest node.}
    \label{fig:sys_overview}
\end{figure*}

\section{Background and problem formulation}\label{sec:formalization}

We now introduce the standard \ac{DL} algorithm and backdoors, and then elaborate on the shortcomings of existing approaches that defend against backdoor attacks in collaborative \ac{ML} algorithms.

\textbf{Decentralized learning.}
A set of $n$ nodes collaboratively trains a model $f_\theta$ to minimize the global risk $\losssampled(\theta) = \frac{1}{n}\sum_{i=1}^n \E_{(x,y)\sim\mathcal{P}_i}\ell(f_\theta(x), y)$, where $\ell$ is the point-wise loss and each node $i$ holds a local dataset $\mathcal{D}_i$ drawn from $\mathcal{P}_i$, with $\mathcal{P}_1,\ldots,\mathcal{P}_n$ \emph{a priori} \ac{NIID}. Communication follows an undirected graph $\mathcal{G}=(\nodeset,\mathcal{E})$ in synchronous rounds.
At round $t$, every node $i$, owning a model $\model_i^t$, performs an arbitrary number $G$ of \ac{SGD} steps on $\model_i^t$ to obtain $\theta_i^{t+\nicefrac{1}{2}}$ (Step 1 in \Cref{fig:sys_overview}), sends it to its neighbors $N(i)$ (Step 2), and aggregates the received updates (Step 7) as $\theta_i^{t+1}=\sum_j\mathcal{W}_{ij}\theta_j^{t+\nicefrac{1}{2}}$ with $\mathcal{W}$ a row-stochastic gossip matrix consistent with $\mathcal{G}$ ($\mathcal{W}\1 = \1$ and $\mathcal{W}_{ij}\neq0 \iff \{i,j\}\in\mathcal{E}\cup\{(i,i)\}$).
This follows the standard D-PSGD algorithm~\citep{lian2017advances}.

In this paper, we consider a $K$-class image classification task with an input space $\mathcal{X} \subseteq \mathbb{R}^{C \times H \times W}$, where $C$, $H$, and $W$ denote the number of channels, height, and width of the input images, respectively. The output space is defined over $K$ discrete classes as $\mathcal{Y} = \{1, \dots, K\}$.

\textbf{Backdoor attacks.}
A backdoor attack \citep{gu2019badnets,bagdasaryan2020backdoor} aims to obtain a model $\theta^b$ that retains a high clean accuracy (\ie accuracy on clean data) but predicts a fixed target label $t^*$ on any input to which a pre-defined trigger $\tau^*$ has been applied: $\Pp_{(x, y)\sim\mathcal{P}} [f_{\theta^b}(x) = y] \approx 1$ and $\Pp_{(x, y)\sim\mathcal{P}} [f_{\theta^b}(x \oplus \tau^*) = t^*] \approx 1$ where $x\oplus \tau^*$ denotes applying $\tau^*$ to $x$. Following \cite{gu2019badnets,wang2019neural,chulin2021crfl,cao2024advances}, we focus on \emph{spatially localized}, patch-like triggers occupying a contiguous region of at most $p\cdot H\cdot W$ pixels with $p \ll 1$. We consider a non-contiguous trigger in \Cref{subsec:exp_triggers} and discuss semantic backdoors~\citep{bagdasaryan2020backdoor} in \Cref{sec:conclusion}.

\textbf{Shortcomings of existing solutions.}
Many works address backdoor attacks in \ac{FL}, where a central server aggregates client updates.
This is done via robust aggregation~\citep{blanchard2017machine,yin2018byzantine,chulin2021crfl,nguyenFLAMETamingBackdoors2022}, server-side anomaly detection~\citep{rieger2022deepsight,nguyen2023fedgrad,lin2023mitigating,huang2024parameter,wang2024boosting,binbin2025fedlad,yuan2025multi,xu2025detecting}, training-time hardening~\citep{huang2023lockdown}, or trigger-based detection~\citep{andreina2021baffle,jia2023fedgame,liBackdoorIndicatorLeveragingOOD2024,riefer2024crowdguard,zhengHoneyFLUsingHoneypots2025,lee2025detrigger} (see \Cref{appendix:related_work} for a thorough discussion).
All of them rely in some way on a central orchestrator with a global view of every update of each round.
Removing this assumption is far from trivial: each honest node now only sees its own neighborhood and the local statistics are biased by its own data, which starves distance-based aggregators, weakens clusterings, and breaks voting quorums.
\detrigger~\citep{lee2025detrigger} is the most adaptable approach for decentralization and tries to reverse-engineer a candidate trigger from incoming model updates.
It inspires our local-detection step (\Cref{subsec:sys_local}), but we show in \Cref{subsec:exp_ablation} that, if applied without any adaptation, it leads to excessive rejections and sharp accuracy drops in \ac{NIID} settings.
In \ac{DL}, on the other hand, prior work has mostly focused on Byzantine robustness \citep{fang2022bridge,elmhamdi2021collaborative} and privacy \citep{cyffers2022muffliato,biswas2025noiseless, biswas2025low}.
\badfl \citep{yuan2025badfl} is framed as a backdoor defense for \ac{DL}, but actually only considers the untargeted gradient-perturbation attack of~\citep{sun2021flwbc}, a model poisoning attack.
To the best of our knowledge, only \citet{syros2025backdoor} (referred to hereafter as \ppcd for brevity) explicitly aims for backdoor robustness in fully decentralized settings.
It proposes a two-norm clipping scheme that applies separate clipping bounds to peer updates and to the local model,  thus not actually detecting backdoors but rather trying to limit their spread.
However, \ppcd heavily relies on well-chosen clipping thresholds, with no heuristic given to set them. 
Crucially, \ac{DL} opens an opportunity: an attacker's update is observed by \emph{all} its neighbors, who can collaborate 
to detect backdoors.
\sys explores this avenue.

\section{Design of \sys}
\label{sec:system}

We now first formalize the system and threat model in \Cref{subsec:threat_model} and then outline the \sys workflow in \Cref{subsec:workflow}.

\subsection{System and threat model}
\label{subsec:threat_model}
A subset $\maliciousset \subset \nodeset$ of nodes is malicious, sharing a common target label $t^*$ and backdoor trigger $\tau^*$. We refer to the nodes in $\maliciousset$ as \emph{attacker nodes}. They use a single, shared trigger and target label for their backdoor, a standard assumption in this domain~\cite{bagdasaryan2020backdoor}.
Attacker nodes know each other, as well as both the communication graph $\mathcal{G}$ and the algorithm implemented by \sys, but do not know the honest nodes' local data.
They comply with the protocol from an external perspective: they participate in every communication round and never abstain from sending a model update.
Internally, however, they may deviate from honest training to pursue their goal, \ie, they may modify their local training data and training procedure.
They may also downscale or disregard some of the updates they receive to preserve the backdoor, as integrating them all is known to %
weaken the attack's effectiveness~\citep{bagdasaryan2020backdoor}.

Let $\honestset:=\mathcal{V}\backslash\mathcal{M}$ denote the set of \emph{honest nodes}, who adhere to the \sys protocol and do not know which node is a potential attacker.
Honest nodes know the \emph{local} graph topology up to distance 2 and, hence, may communicate lightweight information with neighbors of their neighbors.
This is a common assumption in \ac{DL} literature~\citep{menegatti2023discrete,biswas2024secure,zhang2026decentralized}. We also assume that, for every node $i\in \mathcal{V}$, at least $\kappa + 1$ of its neighbors are honest, where $\kappa$ will control the minimum number of confirmations required for collaborative verification (\Cref{subsec:sys_collab}).
This ensures the collaborative step always has enough witnesses.
Note that this assumption is local, and so is the impact of its violation.
We explore adversarial settings where this assumption does not hold in \Cref{subsec:connectivity_assumption_violation}.

\DontPrintSemicolon 
\begin{algorithm}[t]
\footnotesize
    \caption{\sys: Protocol executed by node $i$ at round $t$}
    \label{alg:main}
    \KwIn{Local model $\theta_i^t$, trust states $\{\mathcal{S}_j\}_{j\in N(i)}$, learning rate $\lrmain$, no. \ac{SGD} steps $G$, thresholds $\gamma$, $\xi$, min. confirmations $\kappa$}
    \smallskip
    Compute gradient $g_i^t$ on a batch from $\mathcal{D}_i$;\, Update  $\theta_i^{t+\nicefrac{1}{2}} \leftarrow \theta_i^{t} - \lrmain g_i^{t}$ \tcp*{$G$ local \ac{SGD} steps}
    Send $\theta_i^{t+\nicefrac{1}{2}}$ to $N(i)$; receive $\{\theta_j^{t+\nicefrac{1}{2}}\}_{j\in N(i)}$\tcp*{Communication}\smallskip
                                    \textit{\fbox{$\triangleright$ Stage 1 - Backdoor mitigation}}\smallskip\\
    $\receivedset{i}^t \leftarrow \emptyset$ \tcp*{Set of accepted neighbors}
    \ForEach{$j \in N(i)$ s.t. $\mathcal{S}_j\neq\Ejected$}{
        $(\mathrm{flagged}, \hat\tau_j^i) \leftarrow \textsc{LocalDetect}(\theta_j^{t+\nicefrac{1}{2}}, \gamma)$\tcp*{\Cref{subsec:sys_local}}
        $\mathrm{reject}\leftarrow \mathrm{flagged}\land\textsc{CollaborativeVerify}(j, \hat\tau_j^i, \xi, \kappa)$\tcp* {\Cref{subsec:sys_collab}}
        \lIf{not \textnormal{reject} and $\mathcal{S}_j = \Trusted$ }{
                $\receivedset{i}^t \leftarrow \receivedset{i}^t\cup\{j\}$
        }
        $\mathcal{S}_j\leftarrow \textsc{UpdateTrustState}(\mathcal{S}_j, \mathrm{reject})$\tcp*  {\Cref{subsec:sys_state_machine}}
        
    }\smallskip
    
  \fbox{\textit{$\triangleright$ Stage 2 - Average over accepted neighbors only}}\smallskip\\
    $\theta_i^{t+1} \leftarrow \frac{1}{\len{\receivedset{i}^t}+1 }  \Big(\theta_i^{t+\nicefrac{1}{2}} + \sum_{j\in\receivedset{i}^t} \theta_j^{t+\nicefrac{1}{2}}\Big)$\tcp*[f]{Re-scaled averaging}\label{line:averaging}\\

    \Return $\theta_i^{t+1} \text{ and updated } \{\mathcal{S}_j\}_{j\in N(i)}$
    
\end{algorithm}

\subsection{\sys workflow}
\label{subsec:workflow}

\sys provides a three-component defense against backdoor attacks in \ac{DL}.
Its architecture is summarized in \Cref{fig:sys_overview} and the full per-round protocol is available in \Cref{alg:main}.
At a high level, \sys comprises three key components:
\begin{enumerate*}[label= \emph{(\roman*)}]
    \item \textbf{local trigger detection} (Step 3 in~\cref{fig:sys_overview}), in which each honest node independently scans its neighbors' updates for potential backdoors, returning a binary decision on whether or not to reject the update and a possible trigger pattern; 
    \item \textbf{collaborative verification} (Steps 4--6 in~\cref{fig:sys_overview}), in which flagged updates are cross-validated by querying the sender's other neighbors and comparing their recovered triggers via a structural similarity metric; and 
    \item \textbf{a trust state machine} (omitted from~\cref{fig:sys_overview} for brevity), used by honest nodes to maintain a per-neighbor trust level and eject nodes detected as persistently malicious, thus reducing compute overhead.
\end{enumerate*}
We note that \sys is modular: the exact local detection and collaborative verification implementations may be altered without changing the overall protocol.
We describe below a possible instantiation.
Similarly, \sys is agnostic to the underlying \ac{DL} algorithm and can be layered on top of any gossip-based training procedure.
We instantiate it on D-PSGD~\cite{lian2017advances} throughout this work, as it is the most widely adopted baseline in the \ac{DL} literature.

\subsubsection{Local trigger detection}
\label{subsec:sys_local}

When node $i$ receives a model update $\theta_j^{t+\nicefrac{1}{2}}$, it first tries to reverse-engineer, for each output label $y\in \mathcal{Y}$, a small input perturbation that misclassifies inputs to class $y$ using a small validation set $\mathcal{D}_i^{val} \subseteq \mathcal{D}_i$.
For this, we adopt the per-update recovery step of \detrigger~\citep{lee2025detrigger} proposed for \ac{FL}, but adapt it to a decentralized setting.
This proceeds in three steps:
\begin{enumerate*}[label= \emph{(\roman*)}]
    \item \textit{Initial candidate:} for each source label $z\in \mathcal{Y}\setminus\{y\}$ and $(x,z) \in \mathcal{D}_i^{val}$, we compute the input-layer gradient $g_y(x, z) = \nabla_x[f_\theta(x)_y - f_\theta(x)_z]$, average it across these examples, min-max normalize to $[-1,1]$ and select the source label whose average has highest energy to obtain an initial trigger $\hat\tau^0$ with associated binary mask $m_y$;
    \item \textit{Optimizing the trigger:} we perform $R$ projected gradient steps: $\hat\tau^{(r+1)} ={\mathrm{clip}}\left[\hat\tau^{(r)} -\lrmask\tanh\Big(\nabla_{\hat\tau}\mathcal{L}_{CE}(f_\theta(x\oplus \hat\tau^{(r)}), y)\Big)\right]_{[-1,1]} \odot m_y$ where $\odot$ is the Hadamard product;
    \item \textit{Flagging suspicious updates:} we compute the resulting trigger's \ac{ASR} on $\mathcal{D}_i^{val}$, return the target label with the highest \ac{ASR} and the corresponding $\hat\tau_j^i$ if its \ac{ASR} exceeds a threshold $\gamma$.
\end{enumerate*}

It is noteworthy that a recovered trigger might not correspond to an actual backdoor trigger.
Because false positives (FP) will be filtered by collaborative verification, $\gamma$ can be set aggressively low. Similarly, $|\mathcal{D}_i^{val}|$ can be chosen small, or step \textit{(i)} can be performed on a small subset of it (in our experiments in \Cref{sec:experiments}, $\mathcal{D}_i^{val}$ contains at most one sample per class).

\paragraph{Local detection alone is not enough.} A node whose local data is dominated by class $y$ naturally produces updates that shift the decision to $y$. A neighbor may interpret this change as a backdoor targeting $y$, producing many FP detections under \ac{NIID} data (see ablation in \Cref{subsec:exp_ablation} where local detection alone achieves low \ac{ASR} but severely damages accuracy when heterogeneity increases). Crucially, these \ac{FP}-triggers are \textit{inconsistent across detecting nodes}: each detecting node's input gradient depends on its local data, so two neighbors of the same node end up with two independent random-looking triggers. 
True backdoors, in contrast, are \textit{structurally consistent}\begin{wrapfigure}{r}{0.35\textwidth}
  \centering
    \begin{tabular}{>{\centering\arraybackslash}m{0.32\linewidth}>{\centering\arraybackslash}m{0.32\linewidth}}
        {\scriptsize\textcolor{blue}{\textbf{True Positive}}} &
        {\scriptsize\textcolor{red!70!black}{\textbf{False Positive}}} \\
        \includegraphics[width=1.1\linewidth]{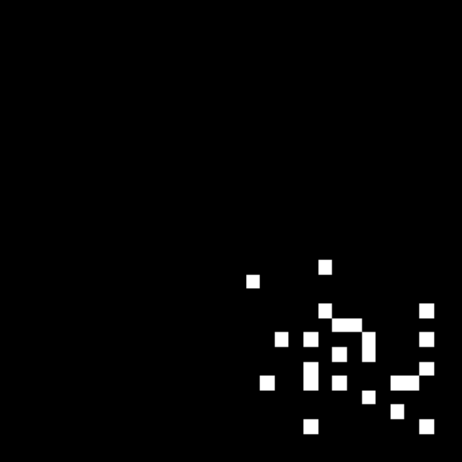}  &
        \includegraphics[width=1.1\linewidth]{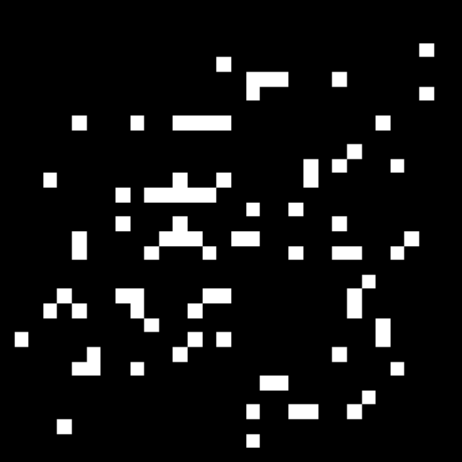}  \\
    \end{tabular}
  \caption{Example \ac{TP} and \ac{FP} reverse-engineered triggers from \cifar when the real backdoor trigger is a bottom-right $3\times3$ pixel square.}
  \label{fig:ex_triggers}
  \vspace{-1em}
\end{wrapfigure} since detecting nodes are tracing the same implanted trigger.
This is formulated in the insight below:
\begin{observation}[Structural similarity of recovered triggers]
\label{obs:main}
Let $\hat\tau_a$ and $\hat\tau_b$ be triggers reverse-engineered by two honest nodes given the \emph{same} update. \textbf{(TP)} If the update contains a genuine backdoor with trigger $\tau^*$, then $\hat\tau_a$ and $\hat\tau_b$ will tend to be structurally similar to each other (and to $\tau^*$). \textbf{(FP)} Otherwise, $\hat\tau_a$ and $\hat\tau_b$ will likely appear as independent random masks.
\end{observation}

\Cref{fig:ex_triggers} motivates this insight
on \cifar with two attacker nodes trying to inject a $3\times3$ bottom-right square backdoor: the pixels in \ac{TP} triggers concentrate near the true backdoor while those in \ac{FP} triggers are scattered and look random.

\subsubsection{Collaborative verification}
\label{subsec:sys_collab}

When node $i$ flags an update from neighbor $j$, it asks $j$'s neighbors to confirm. Each $\ell \in N(j)\backslash\{i\}$ either replies with its own candidate $\hat\tau_j^\ell$ or with \textsc{Not-Suspicious}. Node $i$ rejects the update if, and only if, at least $\kappa$ peers returned a trigger $\hat\tau_{j}^\ell$ 
with $\texttt{sim}(\hat\tau_{j}^i, \hat\tau_{j}^\ell) \ge \xi$, where \texttt{sim} is a similarity metric. 

\textbf{Trigger similarity metric.} Comparing raw triggers pixel-by-pixel is unreliable because candidate triggers may differ in exact pixel positions and intensity even when sharing the same spatial structure.
\sys assumes a similarity metric that \textit{(i)} is sensitive to high-level spatial structure, \textit{(ii)} can abstract away from background noise and intensity scaling, and \textit{(iii)} provides a principled way to set the threshold $\xi$ independently of the ground-truth backdoor trigger. We instantiate \texttt{sim} as the \ac{SSIM} \citep{wang2004image} of the two trigger energy maps after retaining only the top-$k$ highest-energy pixels.
We provide the full mathematical definition of the trigger similarity metric in \Cref{subsec:sim_metric}.

\textbf{Resilience to attackers.} Attacker nodes may try to interfere in two ways: \textit{(i) Shielding a fellow attacker node} by sending a \textsc{Not-Suspicious} message or a random trigger when queried about it. Detection still succeeds as long as at least $\kappa$ of the other neighbors confirm, which is possible since the threat model gives at least $\kappa+1$ honest neighbors.
\textit{(ii) Framing an honest node} by attempting to confirm an \ac{FP} trigger. This gives a lower bound on $\kappa$, above the expected number of attacker nodes.

\subsubsection{Trust state machine}
\label{subsec:sys_state_machine}

The collaborative verification protocol involves nodes detecting backdoors independently every round.
To leverage the history of collaborative decisions, each node in \sys maintains a per-neighbor trust state machine.
In addition to reducing computations for nodes that are already known to be backdoored, this helps reduce the impact of occasional misses in backdoor detection.
Each node $i$ maintains a state $\mathcal{S}_j\in\{\Trusted, \Suspected, \Ejected\}$ per neighbor $j$.
$j$ moves from \Trusted to \Suspected after $k_1$ consecutive confirmed rejections by $i$ (its updates are then rejected by default), and is permanently \Ejected if at least $k_2$ further rejections occur in the next $k_3$ rounds (otherwise it returns to \Trusted). More details and a visual representation are available in \Cref{subsec:trust_state_machine_details}.

\subsection{Setting the \sys hyperparameters}
\label{subsubsec:sys_threshold_selection}

\sys hinges on two main hyperparameters, both related to the collaborative verification phase: \textit{(i)} the minimum number of required similarity confirmations $\kappa$, and \textit{(ii)} the similarity threshold $\xi$. $\kappa$ can be set directly from the graph degree and the threat model, given the expected minimum number of honest neighbors.
$\xi$ is more sensitive since setting it too high lets backdoors pass, while too low rejects honest updates.
Crucially, we show that $\xi$ can be calibrated \textit{before training} using \begin{wrapfigure}{r}{0.35\textwidth}
  \centering
    \includegraphics[width=\linewidth]{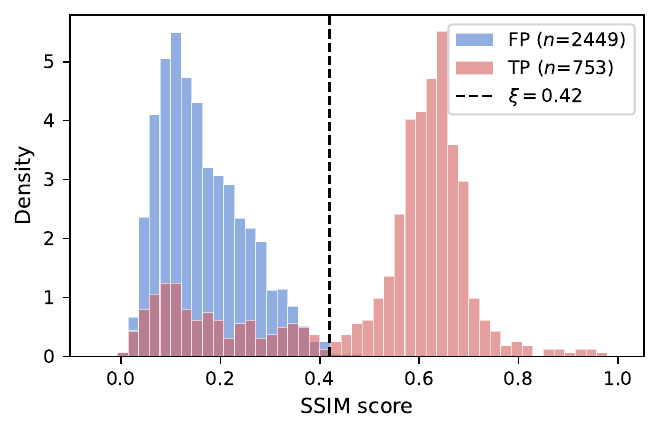}
\caption{\textbf{Empirical \ac{FP} vs. \ac{TP} trigger similarities} on \cifar ($\alpha=0.5$, $m=2$ attacker nodes, $n=16$ nodes) and calibrated threshold $\xi=0.42$.}
\label{fig:pairwise_ssim}
  \vspace{-1em}
\end{wrapfigure} only the image dimensions $(H, W)$, the clipping size $k$, the SSIM window size $w$ and the model architecture, with no access to the trigger or training data. 
Motivated by \Cref{obs:main}, we model \ac{FP} triggers as independent samples from a Gaussian random field with correlation $\sigma$ on the $H\times W$ grid, where $\sigma$ depends solely on the model architecture.
We set $\xi$ as a quantile of the resulting null SSIM distribution, estimated by Monte Carlo sampling over $M=10^4$ pairs. 
This is a one-time, offline cost that merely takes a few seconds on CPU. On \cifar with our parameters ($H=W=32, k=51, w=11, \sigma=2$), we obtain $\xi\approx0.42$. 
Per-dataset values are in \Cref{tab:calibration_per_dataset} and the full derivation in \Cref{sec:threshold_calibration}.

\Cref{fig:pairwise_ssim} validates this calibration: across all rounds and pairs of detecting nodes flagging the same source, \ac{FP} and \ac{TP} similarities are cleanly separated by $\xi=0.42$, with limited overlap.

\section{Convergence analysis}
\label{sec:theory}

We now analyze the convergence of \sys. We proceed through the following steps. First, we make a simple \iid model for accept/reject events in~\sys (\cref{ass:two_rate_detection}). Then, we provide theoretical guarantees on malicious and honest node ejection probabilities (\cref{prop:node_ejection_bounds}). We then highlight a gossip contraction factor induced by  randomized filtering (\cref{prop:spectral_gap_mixing_matrix}). Finally, we use this contraction factor with a standard \ac{DL} convergence analysis to derive convergence guarantees for \sys (\cref{thm:convergence}) and an end-to-end guarantee (\cref{cor:end_to_end}). For brevity, the proofs for the node-ejection bounds (\cref{subsec:node ejection bounds}) are postponed to~\cref{appendix:generic_ejection_theorem}, and the proofs of the convergence-rate properties (\cref{subsec:convergence rates}) to~\cref{appendix:convergence_proof}.

\subsection{Node ejection bounds}%
\label{subsec:node ejection bounds}
Intuitively, if all nodes reject all external updates then learning fails, while if malicious updates are often accepted then backdoors persist.
We model this trade-off with a two-rate assumption.

\begin{assumption}[\iid two-rate detection model]\label{ass:two_rate_detection}
    For $i\in\honestset, j\in N(i)$, let $\delta_{i,j}^t:=\mathbb{I}_{\{j\in\receivedset{i}^t\}}$ denote whether node $i$ accepts the update from node $j$ at round $t$ (and set $\delta_{i,i}^t:=1$).
    Assume the random variables $\{\delta_{i,j}^t\}$ are independent over $t$ and satisfy $\delta_{i,j}^t \sim \operatorname{Bern}(\probafalsenegative)\ \ (j\in\maliciousset)$ and $
        \delta_{i,j}^t \sim \operatorname{Bern}(1-\probareject)\ \ (j\in\honestset)$,
    where $\probafalsenegative$ is the per-round probability of accepting a malicious update and $\probareject$ is the per-round probability of rejecting an honest update.
\end{assumption}

\begin{remark}
In practice, model rejections are temporally correlated due to the inherent dependence of model updates over time. However, when a backdoor is detected from an honest node, we observe random behavior that aligns well with this assumption (\cf \cref{fig:ex_triggers}).
In addition, our threat model assumes that attacker nodes may behave arbitrarily; thus, attackers could send uncorrelated updates, which is handled by~\cref{ass:two_rate_detection}.
Moreover, \cref{appendix:markov_ejection_bounds} relaxes independence to a first-order Markov model: at the correlations measured in our experiments ($0.09$ for honest and $0.31$ for attacker neighbors), the ejection bounds below remain under $1\%$.

\end{remark}
\Cref{ass:two_rate_detection} allows us to derive bounds on the probability of ejecting honest and attacker nodes:
\begin{restatable}{proposition}{NodeEjectionBounds}\label{prop:node_ejection_bounds}%
    In any round $T$, for thresholds $k_1, k_2 \leq k_3$, let
    $
        \pi(p):=p^{k_1}\sum_{r=k_2}^{k_3}\binom{k_3}{r}p^r(1-p)^{k_3-r}
    $, 
    and define $A_T^{\mathrm{hon}}$ (resp. $B_T^{\mathrm{mal}}$) as the event where an honest node reaches (resp. an attacker node does not reach) the state \Ejected{} by the end of round $T$.
    With \cref{ass:two_rate_detection} and $q_{\mathrm{fn}} := 1-\probafalsenegative$, we have 
    \[\mathbb{P}(B_T^{\mathrm{mal}}=1)\le \left(1-\probafalsenegative\pi(q_{\mathrm{fn}})\right)^{\left\lfloor T/(k_1+k_3+1)\right\rfloor}; \mathbb{P}(A_T^{\mathrm{hon}}=1)\le \max\{T-k_1-k_2+1,0\}\,\pi(\probareject).\]
    The first bound is stated for $k_2=1$, as in our experiments; \cref{appendix:markov_ejection_bounds} covers all thresholds.
\end{restatable}

\begin{remark}\label{rem:node_ejection_bounds}
   In our experiments on \cifar, we obtain $\probareject = 1.1 \%$ and $\probafalsenegative = 20.0 \%$. After $50$ iterations, this upper-bounds the probability that a node has ejected a given honest neighbor by $0.02\%$, and the probability that it has not ejected a given attacker neighbor by $33.7\%$; the tighter bound of \cref{appendix:markov_ejection_bounds} brings the latter below $0.01\%$. Moreover, the values of $k_1$, $k_2$, and $k_3$ can be modified to adjust these probabilities depending on the use case. 
\end{remark}

\subsection{Convergence rates}\label{subsec:convergence rates}
Using~\cref{prop:node_ejection_bounds}, we can show that, with high probability, attacker nodes are removed after a few rounds while honest nodes remain in the system.
After these few rounds, \sys behaves like standard \ac{DL} with random link failures.
Similar scenarios have been studied in the literature~\citep{larssonUnifiedAnalysisDecentralized2025}, but without considering rescaling of the received models as we do in Stage 2 of~\cref{alg:main}, and those scenarios fail to derive the impact of the failure probability on the convergence term.

Let $(S^t)_{i,j} := \nicefrac{\delta_{i,j}^t}{\left(1+\len{\receivedset{i}^t}\right)}$, where $\delta_{i,j}^t$ is as defined in \Cref{ass:two_rate_detection}.
Thus, $S^t$ is the random mixing matrix induced by Stage $2$ of~\cref{alg:main}, obtained by masking the base gossip matrix $\mathcal{W}$ according to rejected updates and re-normalizing.
Using this notation, we analyze how link failures affect the spectral gap, a key quantity controlling the convergence of decentralized optimization algorithms~\citep{devosEpidemicLearningBoosting2023,koloskovaUnifiedTheoryDecentralized2020}:
\begin{restatable}[Spectral gap of $\E\!\left((S^t)^\top S^t\right)$]{proposition}{MixingSpectrum}\label{prop:spectral_gap_mixing_matrix}
For an \emph{undirected} and $d$-regular communication graph $\mathcal{W}$, if the entries of $S^t$ satisfy~\cref{ass:two_rate_detection}, let $\phiprobarejectat{\nu}:=(a-cd)+2b\,\nu+c\,\nu^2$ and $\psiprobarejectat{\mu}:=\phiprobarejectat{(d+1)\mu-1}$
for some constants $a,b,c$ (\cf\cref{eq:a_def,eq:b_def,eq:c_def}) depending on $\probareject$.
Let $1=\mu_1\ge\mu_2\ge\cdots\ge\mu_n$ be the eigenvalues of $\mathcal{W}$.
Then, we obtain the following identities:
\[
\quad \lambda_1\!\left(\E[(S^t)^\top S^t]\right) = 1
 \quad\text{ and } \quad
\rho :=\lambda_2\!\left(\E[(S^t)^\top S^t]\right)
= \max\!\left\{\psiprobarejectat{\mu_2},\ \psiprobarejectat{\mu_n}\right\}\in[0,1).
\]
\end{restatable}

We can now state the main convergence result of \sys{} by combining the mean-square contraction factor $\rho$ from~\cref{prop:spectral_gap_mixing_matrix} with the standard descent-consensus decomposition used in \ac{DL} convergence analysis.
\begin{restatable}[Convergence rate of~\sys{}]{theorem}{ConvergenceTheorem}%
    \label{thm:convergence}
    Consider \cref{ass:two_rate_detection}, $G=1$ and that all attackers have been ejected while honest neighbors remain \Trusted.
    Moreover, assume that the setting of~\cref{prop:spectral_gap_mixing_matrix} holds, and that mixing matrices $S^t$ are independent of the gradients computed in the same round.
    Assume each local objective $\locallosssampled{i}$ is $L$-smooth and that \ac{SGD} steps satisfy $\forall x \in \R^d, t \in \mathbb{N}:
        \E\big\|g_i^t(x)-\nabla\locallosssampled{i}(x)\big\|^2\le \boundstochasticnoise^2$ and $
        \frac{1}{n}\sum_{i=1}^n\big\|\nabla\locallosssampled{i}(x)-\nabla\losssampled(x)\big\|^2\le \boundheterogeneity^2$. 
    Define $\Delta_0:= \locallosssampled{}\left(\overline{\model}^0\right) - \locallosssampled{}^*$.
    Then, for a constant stepsize $\lrmain$ (\cref{eq:stepsize constant}), and constants $C_{\rho}$ and  $\beta(d, \probareject)$, $\frac{1}{n}\sum_{i\in\nodeset} \frac{1}{T}\sum_{t=0}^{T-1}\E\left[\norm{\nabla \locallosssampled{}(\model^t_i)}^2\right]$ is bounded by:
    \begin{align*}         
        \mathcal{O}\left(
            \frac{L}{T}\Delta_0
            + \sqrt{\frac{L\Delta_0C_{\rho}\beta(d,\probareject) (\boundheterogeneity^2 + \boundstochasticnoise^2)}{nT}}
            + \sqrt[3]{\frac{L^2\Delta_0^2 C_{\rho} (\boundheterogeneity^2 + \boundstochasticnoise^2)}{T^2}}
        \right).
    \end{align*}
\end{restatable}

\cref{thm:convergence} shows how the probability $\probareject$ of rejecting honest updates affects the convergence rate, captured by the variables $C_\rho$ and $\beta(d,\probareject)$ (defined in~\cref{eq:c rho definition,eq:beta definition} for brevity).

Finally, both results combine into a guarantee that covers the whole run, attack phase included:
\begin{restatable}[End-to-end convergence under attack, informal]{corollary}{EndToEndConvergence}\label{cor:end_to_end}
    Fix $\delta\in(0,1)$ and $\epsilon>0$.
    Under the assumptions of \cref{prop:node_ejection_bounds,thm:convergence}, and if the loss and the disagreement of honest models remain bounded while attackers are present (\cref{ass:bounded_transient}), there exist $T_{\mathrm{ej}}=\mathcal{O}(\log(m/\delta))$ and $T_{\mathrm{cv}}=\mathcal{O}(1/(\delta^2\epsilon^2))$ such that, running \sys for $T=T_{\mathrm{ej}}+T_{\mathrm{cv}}$ rounds,
    \[
        \mathbb{P}\bigg(\frac{1}{\len{\honestset}T_{\mathrm{cv}}}\sum_{i\in\honestset}\sum_{t=T_{\mathrm{ej}}}^{T-1}\norm{\nabla\locallosssampled{}(\model_i^t)}^2\le\epsilon\bigg)\ \ge\ 1-\delta.
    \]
\end{restatable}
The confidence $1-\delta$ and the accuracy $\epsilon$ can thus be set independently.
The closed-form version, proof, and the case of honest neighbors leaving \Trusted{} are in \cref{appendix:end_to_end}.
In particular, the formal version of \cref{cor:end_to_end} lies in \cref{cor:end_to_end_formal}.

\section{Experimental evaluation}
\label{sec:experiments}

We next present our experimental evaluation and answer the following two questions:
\textbf{(Q1)} How effective is \sys at detecting backdoors across datasets and compared to baselines (\Cref{subsec:exp_main})?
\textbf{(Q2)} What does collaborative verification add beyond local detection alone, in terms of balancing attack suppression and model utility, using the \cifar dataset (\Cref{subsec:exp_ablation})?

We provide additional experiments in \Cref{sec:exp_additional} where we further analyze the effectiveness of \sys, \eg, by experimenting with different trigger shapes, a moving trigger designed to evade detection, different topologies, network sizes and attacker node counts.
We also discuss the computational and communication overhead of \sys.

\subsection{Experimental setup}
\label{subsec:exp_setup}

We now outline the key aspects of our experiment setup and provide additional setup details in \Cref{sec:exp_additional_details}.
We also make our code available for reproducibility.\footnote{Code available at \url{https://github.com/sacs-epfl/argus}.}

\textbf{Datasets, models and topologies.}
We evaluate \sys on \cifar \citep{krizhevsky2012learning} with a ResNet-8 model \citep{he2016deep}, \femnist \citep{caldas2019leaf} with a 2-layer \ac{CNN} model \citep{lecun1998gradient} (two $5\times5$ convolutional blocks followed by $2\times2$ max-pooling and a 512-unit fully connected layer) and \imgnet \citep{le2015tiny} with a ResNet-18 model, covering image sizes $32\times32$, $28\times28$ and $64\times64$ and 10, 62 and 200 classes respectively. The default topology is a 3-regular graph with $n=16$ nodes, but we vary both the degree and the graph family in \Cref{subsec:exp_topologies}. In line with recent work \citep{aketi2023global,aketi2024averaging}, we realize label heterogeneity across nodes with the Dirichlet distribution parametrized by $\alpha>0$: a lower $\alpha$ represents higher heterogeneity. We consider $\alpha\in\{0.25, 0.5, \infty\}$ on \cifar and $\alpha=1$ on \imgnet.
\femnist incorporates real-world heterogeneity through its feature distribution.

\textbf{Attack configuration.} We consider $m\in \{1, 2, 3, 4\}$ attacker nodes sharing the same target label and backdoor trigger.
They poison a random fraction of their training set by adding the backdoor and changing the target label, and ignore incoming updates to preserve the backdoor \citep{bagdasaryan2020backdoor}.
Attacker nodes know the defense: during collaborative verification, they shield other attacker nodes and try to frame honest nodes by sending a random trigger.
The default trigger is a $3\times3$ ($6\times6$ for \imgnet) bottom-right-corner square.
We experiment with different trigger characteristics in \Cref{subsec:exp_triggers}.

\textbf{Baselines.} We compare \sys against:
\begin{enumerate*}[label=\emph{(\roman*)}]
    \item \Nodef, the standard D-PSGD algorithm without any defense,    
    
    \item \Oracle, an idealized defense that knows the attacker nodes and rejects only their updates. It is not practical but gives an upper bound on achievable clean accuracy with \ac{ASR}$\approx0$,

    \item \multikrum \citep{blanchard2017machine}, a Byzantine-robust aggregation method commonly used as a baseline in backdoor defenses \citep{nguyenFLAMETamingBackdoors2022,liBackdoorIndicatorLeveragingOOD2024,xu2025detecting}, applied locally over each neighborhood,
    
    \item \badfl \citep{yuan2025badfl}, an approach introduced as a backdoor defense for \ac{DL}, but evaluated only against untargeted poisoning attacks (see \Cref{sec:formalization}) that we re-implement and evaluate against true backdoors, and
    
    \item \ppcd (Peer-to-Peer Clipping Defense), the two-norm clipping defense proposed for \ac{DL} \citep{syros2025backdoor}, which applies a tighter clipping norm to peer updates than to local ones, and skips clipping during the early \emph{agreement phase}. 
\end{enumerate*}

More details about the baselines and their hyperparameters are provided in \Cref{sec:exp_baselines}.

\textbf{Metrics.}
We evaluate all methods using the following two metrics, averaged across honest nodes: \emph{(i)} \ac{CA}, measuring model utility on non-backdoored inputs; and \emph{(ii)} \acf{ASR}, the fraction of non-target test images misclassified as $t^*$ when the trigger is applied, excluding images already misclassified without it.
We also report the \acf{Rej. Rate}, the fraction of peer updates rejected during training, and for \sys specifically, the per-update \acf{FPR} and \acf{TPR} to characterize detection quality.
All experiments are run with 3 random seeds and we report aggregated values and associated standard deviations.

\subsection{The effectiveness of \sys against baselines}
\label{subsec:exp_main}

\begin{table}[t]
\centering
\small
\caption{\textbf{Main results.} \acf{CA} ($\uparrow$ is better) and \acf{ASR} ($\downarrow$ is better) with $m=2$ attacker nodes out of $n=16$ nodes, reporting mean and std over 3 seeds.
}
\label{tab:main_results}
\setlength{\tabcolsep}{3pt}
\resizebox{.95\linewidth}{!}{
\begin{tabular}{@{}l cc cc cc cc@{}}
\toprule
 & \multicolumn{2}{c}{\cifar, $\alpha=0.25$} & \multicolumn{2}{c}{\cifar, $\alpha=0.5$} &  \multicolumn{2}{c}{\femnist} & \multicolumn{2}{c}{\imgnet} \\
\cmidrule(lr){2-3}\cmidrule(lr){4-5}\cmidrule(lr){6-7}\cmidrule(lr){8-9}
\textbf{Defense} & CA [\%] & ASR [\%] & CA [\%] & ASR [\%] & CA [\%] & ASR [\%] & CA [\%] & ASR [\%] \\
\midrule
\Oracle     & 48.4$\pm$1.3 & 0.1$\pm$0.0 & 54.7$\pm$0.8 & 0.1$\pm$0.0 & 73.5$\pm$0.4 & 0.0$\pm$0.0 & 48.2$\pm$0.7 & 0.0$\pm$0.0 \\
\midrule

\Nodef      & 35.3$\pm$2.2 & 78.0$\pm$6.2 & 43.2$\pm$1.3 & 72.4$\pm$4.7 & 69.4$\pm$0.6 & 99.8$\pm$0.1 & 28.3$\pm$0.3 & 46.9$\pm$2.7 \\
\badfl      & 30.4$\pm$2.1 & 50.4$\pm$5.2 & 37.4$\pm$1.2 & 45.7$\pm$5.4 & 69.2$\pm$0.5 & 99.8$\pm$0.1 & 27.1$\pm$0.3 & 45.6$\pm$3.1 \\
\ppcd       & 31.0$\pm$2.0 & 29.3$\pm$8.0 & 38.4$\pm$0.7 & 24.9$\pm$12.2 & 68.0$\pm$0.5 & 57.7$\pm$2.8 & 20.0$\pm$1.0 & 9.4$\pm$1.2 \\
\multikrum  & 41.5$\pm$3.5 & 17.8$\pm$8.5 & 47.4$\pm$0.5 & 15.1$\pm$11.0 & \textbf{73.2$\pm$0.4} & \textbf{0.0$\pm$0.0} & 37.3$\pm$0.6 & \textbf{0.0$\pm$0.0} \\

\midrule
\textbf{\sys (Ours)}  & \textbf{45.6$\pm$2.7} & \textbf{6.9$\pm$3.2} & \textbf{53.3$\pm$1.0} & \textbf{1.7$\pm$1.0} & 72.8$\pm$0.4 & \textbf{0.0$\pm$0.0} & \textbf{43.3$\pm$2.0} & \textbf{0.0$\pm$0.0} \\
\bottomrule
\end{tabular}}
\end{table}

\Cref{tab:main_results} reports the \ac{CA} and \ac{ASR} with $m=2$ attackers in a 16-node 3-regular graph, for all three datasets.
For \cifar, we include two \niid heterogeneity levels ($\alpha = 0.25 $ and $\alpha = 0.5$) and report IID results in \Cref{subsec:exp_cifar10_iid}.
Across every configuration, \sys stays within 5 percentage points of the \Oracle baseline's \ac{CA}, demonstrating that backdoor filtering does largely preserve model utility.
By contrast, \badfl and \ppcd lose up to 20 \ac{CA} percentage points compared to \Oracle, while \multikrum, the strongest competitor, is 10.9 points below \Oracle on \imgnet (while showing minimal loss on the easier \femnist task).
On the attack side, \sys reduces \ac{ASR} to below 7\% on \cifar and to near-zero on \femnist and \imgnet.
\multikrum matches \sys on the latter two, which we attribute to its distance-based rejection: attackers do not integrate honest updates in order to preserve their backdoor, so their model updates drift away, which is exactly what \multikrum is built to detect. Under heterogeneity, however, \multikrum performs worse and leaves a residual \ac{ASR} of 15.1\% on \cifar ($\alpha=0.5$) (vs. 1.7\% for \sys) and 17.8\% for $\alpha=0.25$.
Overall, \sys is the only defense that simultaneously achieves low \ac{ASR} and near-\Oracle{} \ac{CA} across all evaluated settings, and its margin over the strongest baseline widens as data heterogeneity increases.

\subsection{The contribution of collaborative verification}
\label{subsec:exp_ablation}

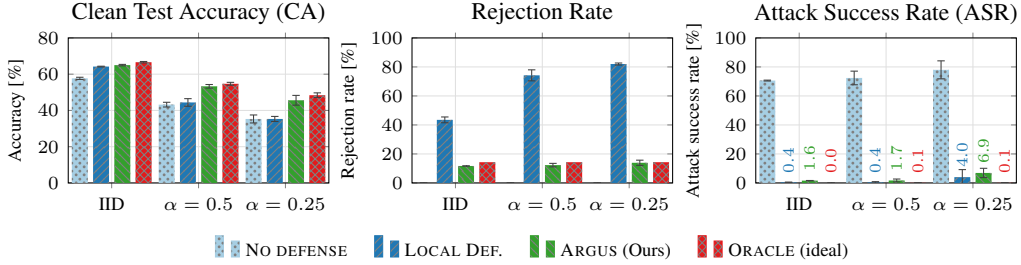
\begin{figure}[t]
\centering
\begin{tikzpicture}
\pgfplotsset{cycle list/Paired}
\pgfplotsset{
    localdet bar/.style={
        ybar,
        bar width=6pt,
        width=0.36\linewidth,
        height=3.5cm,
        enlarge x limits=0.25,
        ymin=0,
        symbolic x coords={iid,a05,a025},
        xtick=data,
        xticklabels={\iid,$\alpha=0.5$,$\alpha=0.25$},
        tick label style={font=\scriptsize},
        label style={font=\scriptsize},
        title style={font=\small, yshift=-1ex},
        ylabel style={yshift=-2pt},
        grid=major,
        grid style={gray!25, line width=0.2pt},
        cycle list name=Paired,
        major tick length=2pt,
        every axis plot/.append style={fill,draw=none,no markers},
    },
    bar with err/.style={
        /pgfplots/error bars/y dir=both,
        /pgfplots/error bars/y explicit,
        /pgfplots/error bars/error bar style={black!75, line width=0.4pt},
        /pgfplots/error bars/error mark=-,
        /pgfplots/error bars/error mark options={
            black!75, line width=0.4pt, mark size=1.4pt, rotate=90
        },
    },
}

\begin{groupplot}[
    group style={group size=3 by 1, horizontal sep=1.1cm},
    localdet bar,
]

\nextgroupplot[title={\acf{CA}}, ymax=80, ylabel={Accuracy [\%]}]
\addplot+[color=Paired-A, postaction={pattern=crosshatch dots, pattern color=gray}, bar with err]
coordinates {(iid,57.7) +- (0,0.6) (a05,43.2) +- (0,1.3) (a025,35.3) +- (0,2.2)};
\addplot+[color=Paired-B, postaction={pattern=north east lines, pattern color=gray}, bar with err]
coordinates {(iid,64.2) +- (0,0.2) (a05,44.4) +- (0,2.1) (a025,35.3) +- (0,1.4)};
\addplot+[color=Paired-D, postaction={pattern=north west lines, pattern color=gray}, bar with err]
coordinates {(iid,65.0) +- (0,0.4) (a05,53.3) +- (0,1.0) (a025,45.6) +- (0,2.7)};
\addplot+[color=Paired-F, postaction={pattern=crosshatch, pattern color=gray}, bar with err]
coordinates {(iid,66.6) +- (0,0.4) (a05,54.7) +- (0,0.8) (a025,48.4) +- (0,1.3)};

\nextgroupplot[title={Rejection Rate}, ymax=100, ylabel={Rejection rate [\%]}]
\addplot+[color=Paired-A, postaction={pattern=crosshatch dots, pattern color=gray}, bar with err]
coordinates {(iid,0.0) +- (0,0.0) (a05,0.0) +- (0,0.0) (a025,0.0) +- (0,0.0)};
\addplot+[color=Paired-B, postaction={pattern=north east lines, pattern color=gray}, bar with err]
coordinates {(iid,43.5) +- (0,2.0) (a05,74.2) +- (0,3.8) (a025,82.0) +- (0,0.7)};
\addplot+[color=Paired-D, postaction={pattern=north west lines, pattern color=gray}, bar with err]
coordinates {(iid,11.7) +- (0,0.0) (a05,12.3) +- (0,1.2) (a025,13.9) +- (0,1.8)};
\addplot+[color=Paired-F, postaction={pattern=crosshatch, pattern color=gray}]
coordinates {(iid,14.3) +- (0,0.0) (a05,14.3) +- (0,0.0) (a025,14.3) +- (0,0.0)};

\nextgroupplot[
    title={\acf{ASR}},
    ymax=100,
    ylabel={Attack success rate [\%]},
    legend style={
        font=\scriptsize,
        legend columns=4,
        at={(-0.85,-0.32)}, anchor=north,
        draw=none, fill=none,
        /tikz/every even column/.append style={column sep=8pt},
    },
    every node near coord/.append style={
        font=\tiny, inner sep=1pt,
        rotate=90, anchor=west, xshift=2.5pt,
    },
]

\addplot+[
    color=Paired-A,
    postaction={pattern=crosshatch dots, pattern color=gray},
    bar with err,
    point meta=y,
]
coordinates {(iid,70.6) +- (0,0.3) (a05,72.4) +- (0,4.7) (a025,78.0) +- (0,6.2)};

\addplot+[
    color=Paired-B,
    postaction={pattern=north east lines, pattern color=gray},
    bar with err,
    nodes near coords={\pgfmathprintnumber[fixed,precision=1,fixed zerofill]{\pgfplotspointmeta}},
    point meta=y,
]
coordinates {(iid,0.4) +- (0,0.2) (a05,0.4) +- (0,0.5) (a025,4.0) +- (0,5.2)};

\addplot+[
    color=Paired-D,
    postaction={pattern=north west lines, pattern color=gray},
    bar with err,
    nodes near coords={\pgfmathprintnumber[fixed,precision=1,fixed zerofill]{\pgfplotspointmeta}},
    point meta=y,
]
coordinates {(iid,1.6) +- (0,0.1) (a05,1.7) +- (0,1.0) (a025,6.9) +- (0,3.2)};

\addplot+[
    color=Paired-F,
    postaction={pattern=crosshatch, pattern color=gray},
    bar with err,
    nodes near coords={\pgfmathprintnumber[fixed,precision=1,fixed zerofill]{\pgfplotspointmeta}},
    point meta=y,
]
coordinates {(iid,0.0) +- (0,0.0) (a05,0.1) +- (0,0.0) (a025,0.1) +- (0,0.1)};

\legend{\Nodef, \Localdef, \sys (Ours), \Oracle (ideal)}

\end{groupplot}
\end{tikzpicture}
\caption{\textbf{Local detection alone is unsatisfactory.}
The \ac{CA} (left, $\uparrow$ is better), rejection rate (middle, $\downarrow$ is better) and \ac{ASR} (right, $\downarrow$ is better) for \cifar with $m=2$ attackers out of $n=16$ nodes, and for varying heterogeneity levels. We consider different baseline settings and report std over 3 seeds.
}
\label{fig:local_det_only}
\end{figure}

We now isolate the contribution of collaborative verification by introducing \Localdef, an \sys variant that rejects every update flagged by local detection, without cross-validating with neighbors.
\Cref{fig:local_det_only} reports the \ac{CA} (left), rejection rate of all model updates (middle) and \ac{ASR} (right) on \cifar for three heterogeneity levels: \ac{IID} ($\alpha\to\infty$), moderate ($\alpha=0.5$) and high ($\alpha=0.25$).
With \Nodef, the backdoor reaches high \ac{ASR} across all settings ($>70\%$), confirming the intrinsic vulnerability of \ac{DL} to backdoors.
Its \ac{CA} remains low as malicious updates deviate from others because attacker nodes do not incorporate the update they receive to preserve the backdoor (see \Cref{subsec:exp_setup}). \Localdef successfully suppresses the backdoor, leading to a near-zero \ac{ASR}, but its rejection rate is excessively high, especially in \ac{NIID} settings: 43.5\% with \ac{IID} data and over 80\% under high heterogeneity.
As discussed in \Cref{subsec:sys_local}, data heterogeneity introduces \ac{FP}s, resulting in accuracy loss: for $\alpha=0.25$, \Localdef is 13 \ac{CA} points lower than \Oracle.
\sys' collaborative verification step helps close that gap across heterogeneity levels.

\section{Conclusions}
\label{sec:conclusion}

We presented \sys, a novel backdoor detection framework for \ac{DL} that requires neither a central coordinator nor prior knowledge of the trigger.
\sys combines local trigger reverse-engineering with collaborative cross-validation among neighboring nodes to cleanly separate true backdoors from false positives induced by data heterogeneity.
We provide the first convergence analysis for a \ac{DL}-native backdoor defense that captures the impact of rejecting updates, modeled as asymmetric link failures, within the convergence bound.
Our experimental evaluation highlights that across three datasets and three baselines, \sys reduces attack success rates to below 7\% while preserving clean accuracy within 5 points of an omniscient oracle.
Thus, \sys is a key step toward backdoor-resilient \ac{DL}.

\paragraph{Limitations and future work.}
We identify two limitations motivating future work.
First, rejected updates may still carry a useful learning signal.
Collaborative verification already limits this loss (for $\alpha=0.25$ on \cifar, it lowers the rejection rate from 82\% with \Localdef to 14\%, on par with \Oracle; \Cref{fig:local_det_only}).
However, \emph{beyond rejection}, a collaboratively verified trigger could be used to \emph{filter} or \emph{unlearn} the backdoor from an update rather than just discarding it, thereby retaining its useful information.
Feasibility in a centralized~\citep{wang2019neural} and \ac{FL} setting~\citep{wu2022toward,wu2024unlearning} has been shown.
Second, \emph{semantic backdoor triggers}~\citep{bagdasaryan2020backdoor} fall outside the considered patch model: our local detection searches for a small additive input-space perturbation, whereas semantic triggers are neither additive nor directly comparable by structural similarity.
Studying these triggers in \ac{DL}, even just their applicability, remains an open problem.

\subsection*{Acknowledgments}
The work of the co-authors affiliated with EPFL has been funded by the Swiss National Science Foundation, under the project ``FRIDAY: Frugal, Privacy-Aware and Practical Decentralized Learning'', SNSF proposal No. 10.001.796, as well as the project ``TruBrain'', proposal No. 20CH21\_218778.
This work has also been supported in part by the Groupe La Poste, sponsor of the Inria Foundation, in the framework of the FedMalin Inria Challenge, and by the ANR 22-PECY-0002 IPOP (Interdisciplinary Project on Privacy) project of the Cybersecurity PEPR.

\bibliographystyle{unsrtnat}
\bibliography{references}

\clearpage
\appendix
\crefalias{section}{appendix}
\crefalias{subsection}{appendix}
\crefalias{subsubsection}{appendix}

\section{Symbol table}\label{appendix:symbol_table}
We provide a table summarizing all the symbols used throughout this work in~\cref{tab:symbols}.
\begin{table}[h!]
    \centering
    \caption{List of symbols used in this work.}%
    \label{tab:symbols}
    \begin{tabular}{|p{0.18\linewidth}|p{0.57\linewidth}|p{0.14\linewidth}|}
        \hline
        \textbf{Symbol} & \textbf{Description} & \textbf{Source} \\
        \hline

        \multicolumn{3}{|l|}{\textbf{Learning setup and threat model}} \\
        \hline
        $\mathcal{X},\mathcal{Y},K$ & Input space, output space, and number of classes & \Cref{sec:formalization} \\
        \hline
        $\nodeset,\ \nbnodes$ & Node set and total number of nodes & \Cref{sec:formalization} \\
        \hline
        $\mathcal{D}_i,\ \mathcal{P}_i$ & Local dataset and local data distribution at node $i$& \Cref{sec:formalization} \\
        \hline
        $\model^t\in\R^{n \times d}$ & Parameter matrix (local instances are $\model_i^t\in \R^d$) & \Cref{sec:formalization} \\
        \hline
        $\losssampled,\ \locallosssampled{i},\ \ell$ & Global objective, local objective, and point-wise loss & \Cref{sec:formalization} \\
        \hline
        $g_i^t,\ \lrmain$ & Stochastic gradient at node $i$ and learning rate & \Cref{sec:formalization} \\
        \hline
        $G$ & Number of local SGD steps & \Cref{alg:main} \\
        \hline
        $\mathcal{G}=(\nodeset,\mathcal{E}), N(i)$ & Communication graph and neighborhood of node $i$ & \Cref{sec:formalization} \\
        \hline
        $\mathcal{W}$ & Base gossip/mixing matrix & \Cref{sec:formalization} \\
        \hline
        $\maliciousset,\ \honestset$ & Malicious and honest node sets & \Cref{subsec:threat_model} \\
        \hline
        $t^*,\ \tau^*,\ \model^b$ & Target label, true backdoor trigger, and backdoored model & \Cref{subsec:threat_model} \\
        \hline
        $x\oplus\tau^*,\ p$ & Trigger application operator and trigger area ratio in $p\cdot H\cdot W$. & \Cref{subsec:threat_model} \\
        \hline

        \multicolumn{3}{|l|}{\textbf{Local detection and collaborative verification}} \\
        \hline
        $\mathcal{D}_i^{val}$ & Local validation subset used by node $i$ & \Cref{subsec:sys_local} \\
        \hline
        $\lrmask$ & Learning rate of the local trigger detection & \Cref{subsec:sys_local}  
        \\
        \hline
        $\hat\tau_j^i,\ m_y,\ \gamma$ & Trigger recovered by node $i$ from sender $j$, binary support mask, and local detection threshold & \Cref{subsec:sys_local} \\
        \hline
        $\tilde\tau,\ k$ & Top-$k$ clipped trigger map.& \Cref{subsec:sys_collab}\\
        \hline
        $\mathrm{sim}(\cdot,\cdot),\ \xi,\ \kappa$ & Trigger similarity metric, similarity threshold, and minimum confirmations & \Cref{subsec:sys_collab} \\
        \hline
        $H,\ W,\ w$ & Image height, image width, and window size in SSIM & \Cref{subsec:sys_collab} \\
        \hline
        $\mu_x,\ \sigma_x^2,\ \sigma_{ab}$ & Local SSIM moments/covariance (with stabilizing constants $C_1,C_2$) & \Cref{subsec:sys_collab} \\
        \hline

        \multicolumn{3}{|l|}{\textbf{Trust state machine and ejection events}} \\
        \hline
        $\mathcal{S}_j$ & Trust state assigned by a node to neighbor $j$ 
            & \Cref{subsec:sys_state_machine} \\
        \hline
        $\receivedset{i}^t$ & Set of accepted neighbors at round $t$ & \Cref{alg:main} \\
        \hline
        $\delta_{i,j}^t$ & Acceptance indicator for edge $(i,j)$ & \Cref{ass:two_rate_detection} \\
        \hline
        $\probareject,\ \probafalsenegative$ & Honest false-positive rejection probability and attacker false-negative acceptance probability & \Cref{ass:two_rate_detection} \\
        \hline
        $k_1,\ k_2,\ k_3,\ \pi(p)$ & Ejection rule thresholds and window success probability & \Cref{prop:node_ejection_bounds} \\
        \hline
        $\corrreject,\ \corrfalseneg$ & Lag-1 correlation of rejections on honest and attacker edges & \Cref{appendix:markov_ejection_bounds} \\
        \hline

        \multicolumn{3}{|l|}{\textbf{Convergence analysis}} \\
        \hline
        $S^t$ & Random mixing matrix after model rejection & \Cref{subsec:convergence rates} \\
        \hline
        $\phiprobareject,\ \psiprobareject$ & Polynomial function linking $\probareject$ to the gossip spectrum & \Cref{prop:spectral_gap_mixing_matrix} \\
        \hline
        $T_{\mathrm{ej}},\ T_{\mathrm{cv}}$ & Ejection round and number of further rounds of the end-to-end guarantee & \Cref{cor:end_to_end} \\

        \hline

        \multicolumn{3}{|l|}{\textbf{Threshold calibration}} \\
        \hline
        $\tilde E=G_\sigma\star Z$ & Smooth Gaussian-field null model for false-positive trigger energy maps & \Cref{eq:gaussian_field} \\
        \hline
        $\sigma$ & Correlation in the Gaussian-field model& \Cref{subsec:sys_collab} \\
        \hline
        $\alpha$ & Dirichlet (data) heterogeneity parameter & \Cref{subsec:exp_setup} \\
        \hline
        
    \end{tabular}
\end{table} %
\section{Related work on backdoor attacks and defenses in FL} %
\label{appendix:related_work}

A large body of work addresses backdoor attacks in \ac{FL}, where a central server aggregates client updates. They can be broadly grouped into four categories. \textit{Robust aggregations} have been proposed to replace the server's average rule with a Byzantine-tolerant aggregator, like Krum or coordinate-wise trimmed mean/median \citep{blanchard2017machine,yin2018byzantine}. FLAME \citep{nguyenFLAMETamingBackdoors2022} combines clustering with adaptive clipping and calibrated Gaussian noise to remove potential backdoors, while CRFL \citep{chulin2021crfl} adds norm clipping and noise to the aggregated model, yielding certified robustness for backdoors with limited magnitude. All of these methods rely on access to the full set of client updates collected at the server in each round, to compute pairwise distances \citep{blanchard2017machine,nguyenFLAMETamingBackdoors2022}, adapt to parameter distributions \citep{yin2018byzantine} or to reach an aggregated global model \cite{chulin2021crfl}. In \ac{DL}, each node could attempt to run them locally over its neighborhood, but the resulting sample size and bias introduced by local data heterogeneity severely damage their statistical robustness.

Another line of work focuses on \textit{server-side model inspection and clustering} \citep{rieger2022deepsight,nguyen2023fedgrad,lin2023mitigating,binbin2025fedlad,yuan2025multi}. These approaches all require a single entity to collect every client's update and compute global statistics and/or clusterings. For example, AlignIns \citep{xu2025detecting} flags updates based on their pairwise cosine similarity with the global aggregate and sign agreement, FDCR \citep{huang2024parameter} clusters clients by Fisher-information-weighted gradient discrepancies from the global update, and BoBa \citep{wang2024boosting} relies on a server with simultaneous visibility of client updates to infer data distributions, perform overlapping clustering, compute trust scores, and then carry out weighted aggregation.

\textit{Training-time hardening} adapts the local training protocol so that backdoors are less likely to form or transfer. Lockdown \citep{huang2023lockdown} restricts each client to train in an isolated sparse subspace, which could be transferred to \ac{DL}. Its per-round fusion step, however, assumes a coordinator can compare clients' subspace masks and purge parameters that lack quorum support, which is not transferable to a fully serverless setting.

\textit{Trigger-based and proactive detection} methods attempt to detect backdoors by actively probing the received updates. Some approaches rely on applying dummy backdoor triggers and studying how they evolve through the training process~\citep{liBackdoorIndicatorLeveragingOOD2024,zhengHoneyFLUsingHoneypots2025}. This would require a central authority or give attacker nodes knowledge of the dummy trigger, rendering it both impractical and ineffective. Others, like CrowdGuard \citep{riefer2024crowdguard} and BaFFLe \citep{andreina2021baffle}, are conceptually closer to a collaborative setting, as they partly rely on client-side validation and voting. However, CrowdGuard still requires a central aggregation, and sends the aggregated model to a random subset of clients, not just neighbors, and BaFFLe's quorum protocol requires server coordination. Neither can be transferred to a purely decentralized setting. Finally, some approaches rely on direct trigger reverse-engineering. FedGame \citep{jia2023fedgame} models defense as a minimax game, using server-side trigger reverse-engineering to compute a \emph{genuine score}. \detrigger \citep{lee2025detrigger} reverse-engineers a candidate trigger from each update via input-layer gradient analysis and verifies it on a server-held validation set. Among these, \detrigger is the most amenable to decentralization, as it operates on each single update independently. We draw inspiration from it in the local detection phase of \sys (\Cref{subsec:sys_local}) but show that, when used alone, it produces unacceptable drops in accuracy under \ac{NIID} data (\Cref{subsec:exp_ablation}).

Finally, the interaction between differential privacy (DP) and backdoor robustness has also been studied. While the approach proposed in \citep{jagielski2021differential} provides mixed results on robustness and utility, \citet{du2020robust} have shown that it can improve outlier and novelty detection, and extended this idea to backdoor-poisoning detection, but their approach is formulated for centralized learning.

Unlike \ac{FL} defenses, \sys operates without a central server or any view of updates beyond the local neighborhood.  It directly identifies attacker nodes and the backdoor mechanism, rather than relying on model-level anomaly score which may be harmed by data heterogeneity. Most importantly, \sys \textit{leverages} the peer-to-peer topology of DL as a \textit{feature} rather than treating it as a limitation: cross-validation of recovered triggers enables separating true backdoors from false alarms induced by data heterogeneity. This collaborative verification enables \sys to achieve near-oracle accuracy while maintaining low \ac{ASR}, without requiring any trusted entity, global aggregation or knowledge of the trigger. 
\section{Additional design details}
\label{sec:design_details}

In this appendix, we provide additional design details of \sys that were omitted from \Cref{sec:system} for brevity.
Specifically, we cover the formal definition of the trigger similarity metric (\Cref{subsec:sim_metric}) and the trust state machine (\Cref{subsec:trust_state_machine_details}).

\subsection{Formal definition of the trigger similarity metric}
\label{subsec:sim_metric}

We give below the full definition of the trigger similarity metric introduced in \Cref{subsec:sys_collab}. It consists of two steps:

\textbf{Step 1: Top-$k$ energy clipping.} For a trigger $\hat\tau \in \R^{C\times H\times W}$, we define the per-pixel energy as the mean across channels of the trigger pixels: $E(\hat\tau)_{hw} = \frac{1}{C}\sum_{c=1}^{C} \hat\tau_{chw}$. We retain at most $k$ pixels with the highest energy, yielding the clipped trigger $\tilde \tau$. This ensures the similarity value depends only on the spatial arrangement of the $k$ most influential pixels, not on trigger density which may vary across detecting nodes (this is required for the threshold calibration in \Cref{sec:threshold_calibration}).

\textbf{Step 2: \ac{SSIM} comparison.} The similarity metric between two triggers $\tilde\tau_a$ and $\tilde\tau_b$ is the mean \textit{local} \ac{SSIM} of the two clipped energy maps $E(\tilde\tau_a)$ and $E(\tilde\tau_b)$, as introduced by \cite{wang2004image}:
\[\texttt{sim}(\tilde\tau_a, \tilde\tau_b) = \frac{1}{HW} \sum_{(h, w)\in H\times W} \frac{(2\mu_a\mu_b + C_1)(2\sigma_{ab}+C_2)}{(\mu_a^2+\mu_b^2+C_1)(\sigma_a^2+\sigma_b^2+C_2)}\bigg|_{(h,w)}\]
where each term is taken over a window $w\times w$ using the uniform averaging kernel $K$ where $K(x, y) = \frac{1}{w^2}$ for $(x,y)\in [-w/2,+w/2]^2$ and $K(x, y)=0$ otherwise, and where:
\[\mu_x = K \star E(\tilde\tau_x),\quad \sigma_{ab}=K\star(E(\tilde\tau_a)E(\tilde\tau_b)) -\mu_a\mu_b, \quad \sigma_x^2=K\star E(\tilde\tau_x)^2 - \mu_x^2\]
with $\star$ being the convolution operator and $C_1 = (0.01\cdot L)^2$, $C_2 = (0.03 \cdot L)^2$ with $L=\max(\max(E(\tilde\tau_a)), \max(E(\tilde\tau_b)))$.

\subsection{Specifications of the trust state machine}
\label{subsec:trust_state_machine_details}

\begin{figure}[t]
\centering
\begin{tikzpicture}[
  >=Stealth,
  node distance=2.6cm and 2.0cm,
  on grid,
  auto,
  state/.style={
    draw,
    rounded corners=6pt,
    thick,
    minimum width=2.2cm,
    minimum height=0.9cm,
    align=center,
    fill=gray!10,
    font=\sffamily\small
  },
  trusted/.style={state, fill=green!8, draw=green!50!black},
  suspected/.style={state, fill=yellow!12, draw=orange!75!black},
  ejected/.style={state, fill=red!10, draw=red!70!black, text=red!70!black},
  transition/.style={->, thick, draw=black!75}
]

\node[trusted] (trusted) {\Trusted};
\node[suspected, right=4cm of trusted] (suspected) {\Suspected};
\node[ejected, right=5cm of suspected] (ejected) {\Ejected};

\path[transition]
  (trusted) edge[bend left=18] node[above,font=\footnotesize] {$k_1$ in a row} (suspected)
  (suspected) edge[bend left=18] node[below,font=\footnotesize] {$<k_2$ in next $k_3$ rounds} (trusted)
  (suspected) edge node[below,font=\footnotesize] {$\ge k_2$ in next $k_3$ round}(ejected);

\end{tikzpicture}
\caption{The per-neighbor trust state machine maintained by each (honest) node.}
\label{fig:state_machine}
\end{figure}
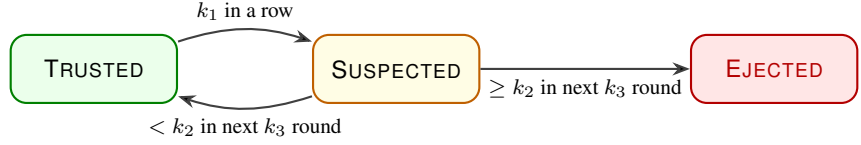

In \sys, each node maintains a per-neighbor trust state machine to reduce computations for nodes that are very likely to be malicious.
The specification of this trust state machine is illustrated in \Cref{fig:state_machine}.
We now explain these states and transitions in more detail.

\paragraph{States.} Each node $i$ maintains a \textit{trust state} $\mathcal{S}_j\in\{\Trusted, \Suspected, \Ejected\}$ for each of its neighbors $j\in N(i)$, governed by a simple state machine with 3 states:
\begin{enumerate}[label=(\roman*)]
    \item \textbf{\Trusted:} This is the default case.
    \item \textbf{\Suspected:} Model updates from a node in the \Suspected state are rejected regardless of the current round's decision outcome.
    \item \textbf{\Ejected} (terminal state): Model updates originating from a node $j$ in the \Ejected state are permanently rejected. Node $i$ keeps $j$'s latest collaboratively verified trigger mask in memory so that it can answer possible future \textsc{CollaborativeVerify} requests.
    These trigger masks are very lightweight and efficient to store.
\end{enumerate}

\paragraph{Transitions.} Transitions may be governed by any rule. In experiments, we chose $(k_1, k_2, k_3) = (2, 1, 3)$, with trade-offs characterized in \Cref{prop:node_ejection_bounds}, and a numerical application in~\cref{rem:node_ejection_bounds}.

\paragraph{Honest nodes that become backdoored.} A consequence of our trust state machine is that an honest node that has been backdoored, \eg, because it integrated a few malicious updates before \sys caught the attacker, may be forever ignored.
This is by design: our experiments show that backdoors in \ac{DL} are remarkably persistent (see \Cref{subsec:exp_prop}) and are present for many rounds even after attacker nodes are ejected from the network.
An infected node would thus keep spreading the backdoor for many rounds.
\section{Calibrating the similarity threshold}
\label{sec:threshold_calibration}

We next provide the details behind the principled similarity threshold selection introduced in \Cref{subsubsec:sys_threshold_selection}.
Recall that we seek a threshold $\xi$ such that two independent \ac{FP} triggers under \Cref{obs:main} are unlikely to exceed it while \ac{TP} triggers do.
We aim to only use quantities known before training, \ie, the image dimensions $H$ and $W$, the clipping parameter $k$, the \ac{SSIM} window size $w$ and the model architecture.

\subsection{Smooth null model}
\label{sec:threshold_calibration_gaussian}

\paragraph{Motivation.} According to \Cref{obs:main}, \ac{FP} triggers are caused by input-layer gradients being biased by each origin and detecting node's heterogeneous local data. These gradients are not pixel-wise independent: the structure of the first convolutional layer correlates nearby pixels. Intuitively, the gradient at each pixel aggregates contributions from all positions whose receptive field covers that pixel.

\paragraph{Gaussian random field model.} We thus model each \ac{FP} trigger energy map as a sample from an isotropic Gaussian random field on the $H\times W$ grid. Concretely:
\begin{equation}
    \tilde E = G_\sigma \star Z \qquad \text{where} \qquad G_\sigma(x, y) = \frac{1}{2\pi\sigma^2}\exp\left(-\frac{x^2+y^2}{2\sigma^2}\right) \label{eq:gaussian_field}
\end{equation}

with $Z\sim \mathcal{N}(0, I)$ of size $H\times W$ and $\star$ the convolution. Two \ac{FP} triggers are modeled as two independent realizations of this process, after which the top-$k$ clipping and \ac{SSIM} computation from \Cref{subsec:sys_collab} are applied.

\paragraph{Choice of $\sigma$ from the architecture.} Based on the motivation above, $\sigma$ should reflect the scale at which input-layer gradients vary, which for convolutional architectures depends on the first layer's receptive field. For our \cifar experiment, we use a ResNet-8 which begins with a $3\times3$ convolutional layer with stride 1, giving a $3\times 3$ receptive field. A natural choice is thus $\sigma =1.5$ (the receptive field half-width). However, we note that: \textit{(i)} a larger $\sigma$ produces smoother triggers, which raises $\xi$ and reduces \ac{FP}s at the cost of potentially more \ac{FN} flaggings, but \textit{(ii)} in \sys, having a few \ac{FN}s is not catastrophic, since a single missed backdoor will be diluted by accepted honest updates. We thus propose to round $\sigma$ up for all datasets. A more formal analysis of the relationship between architecture and the optimal $\sigma$ is left for future work.

\subsection{Monte Carlo calibration procedure}
Given $(H, W, k, w, \sigma)$, the threshold $\xi$ is estimated from the Gaussian distribution as follows:

\begin{enumerate}
    \item Sample $M$ independent pairs of null triggers $(\tau_a^{(m)}, \tau_b^{(m)})$ for $m=1, \ldots, M$ following \Cref{eq:gaussian_field}.

    \item For each pair, apply top-$k$ clipping and compute $\texttt{sim}(\tau_a^{(m)}, \tau_b^{(m)})$ with window size $w$ (also see \Cref{subsec:sys_collab}). 

    \item Set $\xi$ at the desired quantile $q$ of the resulting similarity distribution.
\end{enumerate}

We use $M=10\;000$ and $q=0.99$ for our experiments. This is a one-time, offline, pre-run computation that does \textit{not} depend on any training data or the model weights, and runs in a few seconds on a single CPU.

\subsection{Calibration across datasets}
\label{sec:threshold_calibration_per_dataset}

The clipping parameter $k$ and window size $w$ are set consistently across datasets:
\begin{itemize}
    \item \textbf{Clipping parameter $k$.} We set $k\approx 0.05\cdot H\cdot W$ (rounded to the nearest integer), an order of magnitude retaining the 5\% of pixels with the highest energy. This is motivated by our threat model (see \Cref{subsec:threat_model}), which assumes spatially localized triggers that occupy a small fraction of the image.

    \item \textbf{Window size $w$.} We set $w \approx H/3$, so that \ac{SSIM} windows are large enough to capture meaningful spatial structure while remaining local.
\end{itemize}

\Cref{tab:calibration_per_dataset} reports the resulting per-dataset parameters and threshold.

\begin{table}[t]
\centering
\caption{\textbf{Similarity threshold calibration across datasets.} \ac{SSIM} parameters and resulting threshold $\xi$ for each dataset, following the heuristic above. All use $\sigma=2$ and $M=10\;000$ Monte Carlo samples.}
\label{tab:calibration_per_dataset}
\setlength{\tabcolsep}{5pt}
\begin{tabular}{@{}l cc cc cc c@{}}
\toprule 
\textbf{Dataset} & $H\times W$ & \textbf{Architecture} & $k$ & $w$ & $\E[\texttt{sim}]$ & std[$\texttt{sim}$] & $\xi$ ($q_{99}$) \\

\midrule
\cifar & $32 \times32$ & ResNet-8 & 51 & 11 & 0.195 & 0.081 & 0.42 \\

\midrule

\femnist & $28\times28$ & \ac{CNN} & 39 & 9 & 0.279& 0.089& 0.51 \\

\midrule

\imgnet & $64\times64$ & ResNet-18 & 205 & 19 & 0.058 & 0.031 & 0.15 \\

\bottomrule
\end{tabular}
\end{table}

Several points are worth noting. First, the null \ac{SSIM} distributions vary significantly across image resolutions, confirming the need for a per-dataset thresholding rather than a universal constant. Larger images (as in \imgnet) lead to much lower null similarities since two Gaussian random fields are less likely to align, $\sigma$ staying the same. Second, the null distributions are relatively well concentrated, which validates the idea to discriminate based on similarities. Third, the thresholds $\xi$ all sit well below 1 (perfect similarity), leaving a large margin for \ac{TP} triggers.

\begin{figure}[t]
\centering
\includegraphics[width=0.85\linewidth]{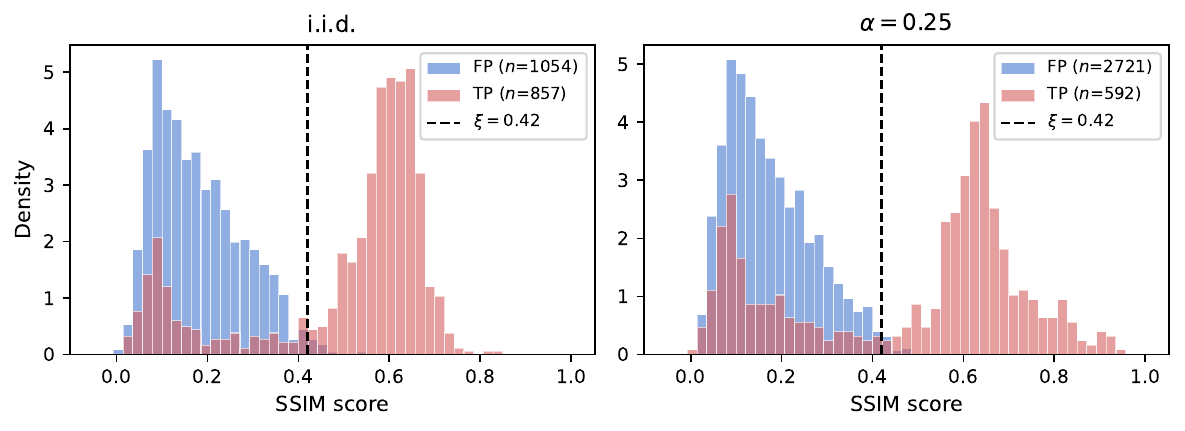}
\caption{\textbf{Empirical \ac{FP} vs. \ac{TP} similarities (additional settings).} We consider \ac{IID} (left) and $\alpha=0.25$ (right) heterogeneity levels, using the \cifar dataset.
}
\label{fig:fp_tp_distrib_appendix}
\end{figure}

\paragraph{Empirical validation.}
\Cref{fig:pairwise_ssim} in \Cref{subsec:sys_collab} already shows the \ac{FP}/\ac{TP} trigger similarity distributions for $\alpha=0.5$ on the \cifar dataset.
We additionally show in \Cref{fig:fp_tp_distrib_appendix} the same distributions for additional heterogeneity levels: \ac{IID} (left) and $\alpha=0.25$ (right).
Across all data distributions, $\xi=0.42$ cleanly separates the two distributions.
Lower \ac{TP} \ac{SSIM} values correspond to early iterations, where the backdoor is yet to be fully implanted, as well as occasional trigger reversal errors.\footnote{Note that the distributions in the figure are normalized, which may exaggerate the perceived \ac{TP} trigger counts.} \Cref{subsec:varying_xi} explores the effectiveness of \sys when the threshold varies.

\section{Additional experimental setup details}
\label{sec:exp_additional_details}

In this appendix, we provide additional details on the experimental setup given in \Cref{subsec:exp_setup}.

\paragraph{Training hyperparameters.}
On the \cifar and \femnist datasets, nodes train for 160 and 80 rounds, respectively, with the \ac{SGD} optimizer and learning rate $\eta = 0.01$.
Each round with \cifar and \femnist consists of 5 and 15 local batches, respectively, of 64 images using the cross-entropy loss.
On the \imgnet dataset, we train for 160 rounds with learning rate $\eta=0.1$, momentum 0.9 and weight decay $10^{-4}$.
Each round with \imgnet consists of 60 batches of 64 images.

\paragraph{\sys hyperparameters.}
Because \acp{FP} will be filtered by collaborative verification, we set the local-detection threshold aggressively low to $\gamma=0.5$ (also see \Cref{subsec:sys_local}). The number of refinement steps is $5$ for \cifar and \femnist and $3$ for \imgnet, with a step size of 0.2 in both cases. Each node's validation set is made of 1 sample per class, for classes available in the local dataset.
The similarity threshold $\xi$ and other collaborative verification parameters are calibrated per dataset as described in \Cref{sec:threshold_calibration_per_dataset} and the minimum number of confirmations is $\kappa=1$, unless specified otherwise. The \Trusted $\to$ \Suspected transition is triggered after $k_1 = 2$ consecutive rejections, and \Suspected $\to$ \Ejected after $k_2=1$ rejection in the following $k_3=3$ rounds. 

\paragraph{Attack configuration.}
Attacker nodes poison a fixed fraction $p_{\text{poison}}=0.3$ ($p_{\text{poison}}=0.25$ for \imgnet) of their training set, by adding the appropriate backdoor trigger and changing the target label to $t^*=7$.
They do not incorporate incoming peer updates during averaging \citep{bagdasaryan2020backdoor}.
During collaborative verification, they send a \textsc{Not-Suspicious} message to every query about another attacker node, and a \textsc{Suspicious} message with a random trigger of size $k$ for queries about honest nodes.

\paragraph{Seeds and reproducibility.} Each configuration in \Cref{tab:main_results} and \Cref{fig:local_det_only} is run with random seeds $\{1, 2, 3\}$. We report mean and standard deviation.
All our code and run scripts are released as a public repository (see \Cref{subsec:exp_setup}).

\textbf{Hardware and implementation.}
We run all experiments on our compute cluster.
Nodes in this cluster are equipped with an NVIDIA A100 GPU with 40GB of VRAM.
We implement \sys using the \textsc{DecentralizePy} framework \citep{dhasade2023decentralized}. %
\section{Baseline implementation details}
\label{sec:exp_baselines}

We now describe the baselines introduced in \Cref{subsec:exp_setup} in more detail.
All of them are applied locally at each node, since no central coordinator is available in our \ac{DL} setting. When possible, we reuse the hyperparameters from the respective papers, otherwise we tune them.

\paragraph{\Nodef.} This baseline represents standard \ac{D-PSGD} in which each node trains locally, broadcasts its update, and averages all incoming updates with uniform weights.
There is no defense against any backdoor attacks.

\paragraph{\Oracle.} This baseline represents the ground-truth defense, where honest nodes reject all updates from $\mathcal{M}$ and accept all updates from $\mathcal{H}$.
This defense is not practical to deploy since it assumes the identities of attacker nodes are known in advance by all honest nodes, but it provides a lower bound for \ac{ASR} and an upper bound for \ac{CA}.\footnote{Although, due to heterogeneity, this \ac{CA} could possibly be exceeded while keeping a low \ac{ASR} by accepting very few malicious model updates.}

\paragraph{\multikrum \citep{blanchard2017machine}.} This is a Byzantine-robust aggregation rule and is commonly used as a baseline in related work on backdoor defenses \citep{nguyenFLAMETamingBackdoors2022,liBackdoorIndicatorLeveragingOOD2024,xu2025detecting}.
We adapt this rule to \ac{DL} by running it independently over the model updates received by each node from its neighbors.
More precisely, given received updates $\{\theta_j\}_{j\in N(i)}$, each node $i$ computes score $s^{(i)}_j = \sum_{k\in A_j}\|\theta_k - \theta_j\|^2$, where $A_j$ is the set of $|N(i)| - r - 1$ closest updates and $r$ is the number of updates to reject.
The $r$ updates with the largest score are rejected, while the rest is averaged with node $i$'s own model.
We set $r$ to be the maximum number of attackers in the neighborhood of any node (\eg, $r=1$ in the experiments in \Cref{sec:experiments}).  

\paragraph{\badfl \citep{yuan2025badfl}.} This baseline corrects model weights after each local training round. Given post-aggregation snapshots $X_{t-1}$ and $X_t$, and post-training weights $\tilde{X}_t$, it forms the diagonal Hessian estimate $\hat{H} = ((\tilde{X}_{t+1} - X_t) - (X_t - X_{t-1}))/\gamma$, clips $(I - \gamma \hat{H})$ element-wise to $[-q, q]$ to obtain $G$ and corrects its local update: $X_t = \tilde{X}_t - \gamma\alpha G$. We adopt the hyperparameters from their work and set $\alpha = 0.08$ and $q = 0.1$.

\paragraph{\ppcd \citep{syros2025backdoor}.} With this baseline, at round $t$, node $i$ rescales each incoming neighbor update $U_j^t = \theta_j^t - \theta_j^{t-1}$ to satisfy $\|U_j^t\|_2 \leq C_{\text{neigh}}$ and applies the same rescaling to its own local update $U_i^t$ with bound $C_{\text{local}}$, where $C_{\text{neigh}} < C_{\text{local}}$, before aggregating. Clipping is disabled during an initial agreement phase to improve accuracy. We adopt $C_{\text{neigh}}=0.1$ and $C_{\text{local}} = 1.0$ as reported by the authors.
We also adopt their choice of 50 agreement rounds for the \cifar and \imgnet datasets, but set it to 0 on \femnist since we found non-zero values allow the backdoor to propagate through the agreement phase and persist.  %
\section{Additional experimental results}
\label{sec:exp_additional}

We provide additional experiments that complement the main results presented in \Cref{sec:experiments}.
These additional experiments address the following questions:
\begin{enumerate}[label=\textbf{(Q\arabic*)}, leftmargin=*, align=left]
    \item How severe and persistent is the backdoor threat in \ac{DL}, even from a single attacker node?
    \item Does \sys remain effective in the IID setting, where data heterogeneity is not a confounding factor?
    \item How effective is \sys when varying the backdoor trigger shape, size, and position?
    \item Can attacker nodes evade \sys by breaking the cross-node consistency of recovered triggers?
    \item How does \sys perform under different network topologies and attacker counts?
    \item Does \sys degrade gracefully when the connectivity assumption stated in \Cref{subsec:threat_model} is violated?
    \item How sensitive is \sys to the choice of similarity threshold $\xi$?
    \item What is the computational and communication overhead introduced by \sys, and when is it acceptable?
\end{enumerate}

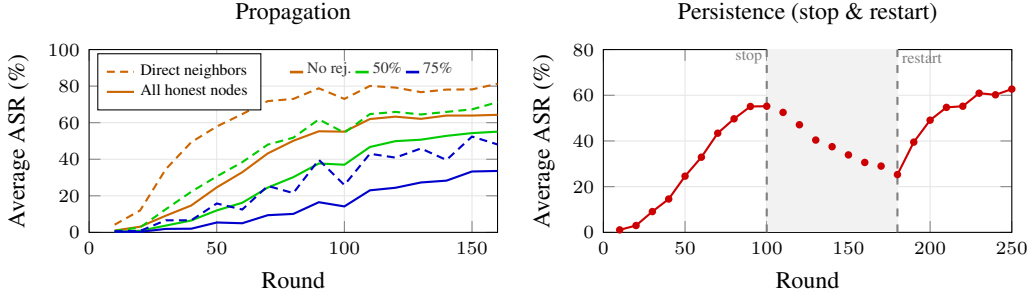
\begin{figure}[t]
\centering
\begin{tikzpicture}
\begin{groupplot}[
    group style={
        group size=2 by 1,
        horizontal sep=1.4cm,
    },
    width=0.5\linewidth,
    height=4cm,
    xlabel={Round},
    ylabel={Average \ac{ASR} (\%)},
    grid=major,
    grid style={gray!20},
    tick label style={font=\scriptsize},
    label style={font=\small},
    legend style={font=\scriptsize, at={(0.98,0.02)}, anchor=south east},
    ymin=0, ymax=100,
    xmin=0,
]

\nextgroupplot[title={\small Propagation}, xmax=160,
    legend style={font=\tiny, at={(0.02,0.98)}, anchor=north west, row sep=-2pt, legend cell align=left, /tikz/every even column/.append style={column sep=0pt}, legend image post style={scale=0.6}}]
\addplot[thick, color=orange!80!black, densely dashed] table[x=round, y=none_hop1] {plots/propagation_asr.dat};
\addplot[thick, color=orange!80!black] table[x=round, y=none_avg] {plots/propagation_asr.dat};
\addplot[thick, color=green!80!black, densely dashed, forget plot] table[x=round, y=rej50_hop1] {plots/propagation_asr.dat};
\addplot[thick, color=green!80!black, forget plot] table[x=round, y=rej50_avg] {plots/propagation_asr.dat};
\addplot[thick, color=blue!80!black, densely dashed, forget plot] table[x=round, y=rej75_hop1] {plots/propagation_asr.dat};
\addplot[thick, color=blue!80!black, forget plot] table[x=round, y=rej75_avg] {plots/propagation_asr.dat};
\legend{Direct neighbors, All honest nodes}
\node[font=\tiny, anchor=north west, align=left, text=black!80] at (axis cs:75,98) {%
    \textcolor{orange!80!black}{\rule{5pt}{1pt}} No rej.\
    \textcolor{green!80!black}{\rule{5pt}{1pt}} 50\%\
    \textcolor{blue!80!black}{\rule{5pt}{1pt}} 75\%};

\nextgroupplot[title={\small Persistence (stop \& restart)}, xmax=250,ymax=80]
\addplot[thick, color=red!80!black, mark=*, mark size=1pt] table[x=round, y=asr] {plots/persistence_asr.dat};
\fill[gray!10] (axis cs:100,0) rectangle (axis cs:180,80);

\draw[dashed, gray, thick] (axis cs:100,0) -- (axis cs:100,80);
\draw[dashed, gray, thick] (axis cs:180,0) -- (axis cs:180,80);
\node[font=\tiny, anchor=south, text=gray] at (axis cs:89,70) {stop};
\node[font=\tiny, anchor=south, text=gray] at (axis cs:195,70) {restart};

\end{groupplot}
\end{tikzpicture}
\caption{\textbf{Backdoor propagation and persistence in \ac{DL}.} Experiments on \cifar with a 3-regular network of 16 nodes ($\alpha=0.5$) and $m=1$ attacker node. \textit{Left}: Average \ac{ASR} of honest nodes with varying malicious updates rejection rates, and when rejecting no malicious updates. The solid line indicates the average \ac{ASR} across all honest nodes whereas the dashed line considers the \ac{ASR} of nodes that are directly connected to the attacker node. \textit{Right}: The average \ac{ASR} of honest nodes where the attacker poisons during rounds 0--100, then goes silent (shaded), and resumes its malicious behavior at round 180.}
\label{fig:propagation}
\end{figure}

\subsection{Backdoor propagation and persistence in DL}
\label{subsec:exp_prop}

To motivate the need for a defense against the backdoor attack in \ac{DL}, we run an experiment to quantify the disruption caused by a single attacker node.
We train an \ac{ML} model using a 3-regular, 16-node network with heterogeneity $\alpha=0.5$ on \cifar, where a single node (node 0) injects a backdoor in their model updates.
We consider a setting where honest nodes do not use any defense to detect or prevent backdoors, and aggregate all incoming model updates (\textit{no rejection}).
Additionally, to study the effectiveness of rejecting backdoored model updates, we also consider two settings where malicious model updates are rejected in a structured, periodic manner: with a 50\% and 75\% rejection rate, respectively.

\paragraph{Propagation.}
\Cref{fig:propagation} (left) shows the average \ac{ASR} across honest nodes as training progresses when different proportions of malicious updates are rejected throughout, and without any rejection.
With \textit{no rejection}, the backdoor by the single attacker node propagates rapidly through the network, reaching over 50\% average \ac{ASR} for all nodes within the first 80 rounds.
Separating direct neighbors of the attacker node (dashed lines) from the whole network (solid) shows the backdoor propagation dynamics: direct neighbors are infected first and reach above 80\% \ac{ASR}.
The other nodes in the network are infected with a delay but still reach an average \ac{ASR} around 60\%.
Even when honest nodes can somehow identify and reject 50\% or 75\% of the malicious updates, the backdoor still propagates and is effective, with the average \ac{ASR} for all nodes reaching 55\% and 34\% on average, respectively.
Thus, \Cref{fig:propagation} (left) demonstrates that even when rejecting a majority of backdoored model updates, the backdoor is able to spread and manifest in the network.

\paragraph{Persistence.}
\Cref{fig:propagation} (right) shows the average \ac{ASR} of honest nodes when the attacker stops injecting and spreading its backdoor at round 100 and later resumes at round 180.
One would expect that the backdoor effectively gets overridden and erased through the model updates of honest nodes.
However, even after the attacker goes silent, the \ac{ASR} only decays slowly, decreasing by just 30 percentage points in 80 rounds and rebounds rapidly once poisoning resumes.
This demonstrates that backdoors in \ac{DL} are remarkably persistent: once injected, they remain in the network for many rounds, even when the attacker stops spreading malicious model updates.

\subsection{The effectiveness of \sys and baselines on \cifar (IID)}
\label{subsec:exp_cifar10_iid}

\begin{table}[t]
\centering
\caption{\textbf{\cifar \ac{IID} results.} The \acf{CA} ($\uparrow$ is better) and \acf{ASR} ($\downarrow$ is better) with $m=2$ attacker nodes out of $n=16$ nodes, reporting mean and std over 3 seeds. We use the \cifar dataset in an IID setting ($\alpha=\infty$).}
\label{tab:iid_cifar}
\setlength{\tabcolsep}{3pt}
\begin{tabular}{@{}l cc@{}}
\toprule
\textbf{Defense} & CA [\%] & ASR [\%] \\
\midrule
\Oracle     & 66.6$\pm$0.4 & 0.0$\pm$0.0  \\
\midrule

\Nodef      & 57.7$\pm$0.6 & 70.6$\pm$0.3 \\
\badfl     & 53.0$\pm$0.8 & 52.7$\pm$1.3 \\
\ppcd      & 58.8$\pm$0.8 & 32.9$\pm$0.3\\
\multikrum  & 65.0$\pm$0.2 &0.1$\pm$0.0 \\

\midrule
\textbf{\sys (Ours)}  & 65.0$\pm$0.4 & 1.6$\pm$0.1 \\
\bottomrule
\end{tabular}
\end{table}

\Cref{tab:iid_cifar} shows the \ac{CA} and \ac{ASR} on the \cifar dataset in the \ac{IID} setting ($\alpha=\infty$).
These results are obtained using the same setting as the experiment described in \Cref{subsec:exp_main}.
\Cref{tab:iid_cifar} shows that both \sys and \multikrum achieve near-oracle performance in this setting.
We attribute the effectiveness of \multikrum to the fact that attackers do not integrate model updates from their neighboring nodes to make their backdoor as effective as possible.
This makes their model updates diverge geometrically from those of honest nodes, whose model updates remain close to one another under an IID data distribution.
Consequently, this gives \multikrum a clean distance-based signal to identify and reject the model updates from attackers.
As shown earlier in \Cref{tab:main_results}, the effectiveness of \multikrum in \ac{NIID} settings degrades since the model updates from honest nodes are further apart, making it increasingly difficult to distinguish attacker from honest nodes as data heterogeneity increases.
Nevertheless, \Cref{tab:iid_cifar} shows that \sys compared to \multikrum achieves similar \ac{CA} and competitive \ac{ASR} in an IID setting, demonstrating that our trigger-based detection does not meaningfully degrade \ac{CA} even in this regime that structurally favors distance-based methods.
Notably, \badfl and \ppcd fail to suppress the backdoor even under IID conditions, with average \ac{ASR} of 52.7\% and 32.9\% respectively, and incur more pronounced degradations in \ac{CA} compared to the \Oracle baseline.

\begin{figure}[t]
\includegraphics[width=\linewidth]{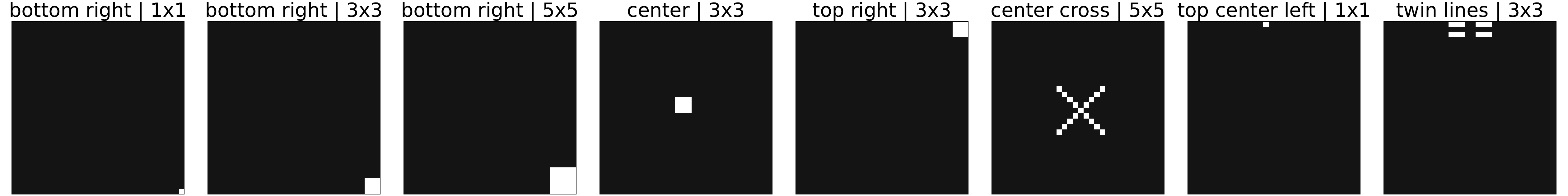}
\caption{\textbf{Various trigger types.} We vary their shape, position and size. Actual triggers are colored in gray to be more inconspicuous.}
\label{fig:trigger_patterns}
\end{figure}

\begin{table}[t]
\centering
\caption{\textbf{Robustness to different trigger patterns}. The \acf{CA} ($\uparrow$ is better), \acf{ASR} ($\downarrow$ is better) and \acf{FPR} ($\downarrow$ is better) with $m=2$ attacker nodes out of $n=16$ nodes. We use the \cifar ($\alpha=0.5$) and \femnist datasets. Triggers vary in position, size and shape and are visualized in \Cref{fig:trigger_patterns}.}
\label{tab:trigger_patterns}
\setlength{\tabcolsep}{4pt}
\begin{tabular}{@{}l ccc ccc@{}}
\toprule
 & \multicolumn{3}{c}{\femnist} & \multicolumn{3}{c}{\cifar ($\alpha=0.5$)} \\
\cmidrule(lr){2-4} \cmidrule(lr){5-7}
\textbf{Trigger (position, size)} & \ac{CA} [\%] & \ac{ASR} [\%] & \ac{FPR} [\%] & \ac{CA} [\%] & \ac{ASR} [\%] & \ac{FPR} [\%] \\
\midrule

bottom right, $1{\times}1$           & 72.8 & 0.0 & 34.6 & 53.4 & 1.8 & 0.7\\
bottom right, $3{\times}3$           & 73.0 & 0.0 & 26.8 & 54.2 & 2.9 & 2.4\\
bottom right, $5{\times}5$         & 72.5 & 3.7 & 31.9 & 54.9 & 3.6 & 1.2 \\

\midrule

center, $3{\times}3$                & 72.9 & 0.5 & 29.4 & 55.4 & 10.3 & 1.0 \\
top right, $3{\times}3$             & 72.9 & 0.1 & 31.7 & 52.8 &6.0 & 1.0 \\
center cross, $5{\times}5$           & 72.8 & 0.1 & 39.5 & 55.1 & 13.1 & 0.9 \\
top center left, $1{\times}1$       & 72.9 & 0.0 & 29.0 & 54.6 &6.8& 0.7 \\
twin lines, $3{\times}3$ & 72.9 & 0.0 & 29.1 & 55.1 & 12.1 & 0.7\\

\bottomrule
\end{tabular}
\end{table}

\subsection{Robustness of \sys to different trigger patterns}
\label{subsec:exp_triggers}

In \Cref{sec:experiments} we analyzed the effectiveness of \sys using a single, fixed trigger.
To see if \sys remains effective with different trigger characteristics, we now evaluate \sys on eight triggers with varying size, position and shape, including a non-contiguous pattern (\emph{twin lines}).
This is consistent with the variety of triggers considered in \cite{lee2025detrigger}.
We visualize these triggers in \Cref{fig:trigger_patterns}.

\Cref{tab:trigger_patterns} reports final \ac{CA}, \ac{ASR} and \ac{FPR} for \sys on the \femnist and \cifar ($\alpha=0.5$) datasets, with $m=2$ attacker nodes in a 3-regular graph with $n=16$ nodes. 
\sys effectively mitigates the backdoor attack in all configurations while keeping a \ac{CA} close to the \Oracle baseline (see \Cref{tab:main_results}) regardless of the trigger.
The \ac{ASR} is consistently low for both datasets and across tested triggers.
The \ac{FPR} is clearly higher on \femnist than on \cifar, where it is close to 0\%, even if the accuracy remains high.
This is consistent with the higher null \ac{SSIM} predicted by our analysis in \Cref{sec:threshold_calibration}.
Triggers located in the center of the image (\eg, \emph{center} and \emph{center cross}) lead to a higher \ac{ASR} on \cifar compared to other trigger variations, but are still effectively detected, with a \ac{TPR} of 85.0\% and 99.3\% respectively.

\subsection{Robustness of \sys to adaptive attackers}
\label{subsec:exp_adaptive}

\begin{table}[t]
\centering
\caption{\textbf{Robustness to a moving trigger.} The \acf{CA} ($\uparrow$ is better), \acf{ASR} ($\downarrow$ is better), \acf{FPR} ($\downarrow$ is better) and \acf{TPR} ($\uparrow$ is better) when attacker nodes re-position their $3\times3$ trigger uniformly at random within a $16\times16$ region at every round. We use a 3-regular graph with $n=16$ nodes, $m=2$ attacker nodes and $\kappa=1$, on the \cifar ($\alpha=0.5$) dataset.}
\label{tab:moving_trigger}
\setlength{\tabcolsep}{4pt}
\begin{tabular}{@{}l cccc@{}}
\toprule
\textbf{Defense} & \ac{CA} [\%] & \ac{ASR} [\%] & \ac{FPR} [\%] & \ac{TPR} [\%] \\
\midrule
\Oracle & 55.7 & 0.5 & 0.0 & 100 \\
\Nodef & 45.0 & 49.7 & -- & -- \\
\sys & 54.0 & 1.7 & 1.2 & 72.9 \\
\bottomrule
\end{tabular}
\end{table}

All experiments in this paper already consider attacker nodes that are adaptive at the protocol level: they know each other, the communication graph and the \sys algorithm, and actively interfere with collaborative verification by shielding fellow attacker nodes and framing honest ones (\Cref{subsec:sys_collab} and \Cref{sec:exp_additional_details}).
We now discuss how attacker nodes could instead target the property on which collaborative verification relies, namely the consistency of the triggers recovered by different neighbors of the same node (\Cref{obs:main}).

We first note that breaking this consistency is not trivial: the consistency \sys tests is not on something directly transmitted by the attacker, but a property of the received model, which each neighbor derives independently from its own data.
Making two honest neighbors recover inconsistent triggers from the \emph{same} model requires a weak or ambiguous backdoor, which in turn reduces the effectiveness of the attack.
This leaves three natural strategies:
\begin{enumerate}[label=\emph{(\roman*)}]
    \item \emph{Using a distinct trigger per attacker node.} This does not help, since collaborative verification is performed per sender: all the neighbors of a given attacker node inspect the same model, and thus the same trigger.
    \item \emph{Sending different models to different neighbors}, each carrying a different trigger. This can be detected at negligible cost by attaching a hash of the model received to each verification query and reply: mismatching hashes reveal that the sender did not send the same model to all its neighbors and thus either a faulty or malicious behavior.
    \item \emph{Changing the trigger across rounds}, to disperse the energy of the recovered trigger and harm reverse-engineering, but this weakens the backdoor.
\end{enumerate}

We evaluate strategy \emph{(iii)} with a \emph{moving trigger}: at each iteration, attacker nodes re-position their $3\times3$ trigger uniformly at random within a $16\times16$ region.
\Cref{tab:moving_trigger} reports the results on \cifar ($\alpha=0.5$) with $m=2$ attacker nodes in our default 3-regular graph of $n=16$ nodes.
We make two observations.
First, the evasion strategy weakens the attack itself: without defense, the moving trigger reaches an \ac{ASR} of 49.7\%, compared to 79.0\% for the static one (\Cref{tab:varying_xi}).
Second, \sys still defends against it, with an \ac{ASR} of 1.7\% and a \ac{CA} within 1.7 points of \Oracle.
Reverse-engineering still recovers a localized trigger that neighbors confirm, with a \ac{TPR} of 72.9\% compared to 82.8\% for the static trigger (\Cref{tab:topologies}).

\begin{table}[t]
\centering
\caption{\textbf{Robustness to different topologies and number of attacker nodes.} The \acf{CA} ($\uparrow$ is better), \acf{ASR} ($\downarrow$ is better), \acf{FPR} ($\downarrow$ is better) and \acf{TPR} ($\uparrow$ is better) with varying topologies, number of attackers $m$, total number of nodes $n$ and confirmation thresholds $\kappa$. We use the \cifar ($\alpha=0.5$) dataset. The backdoor trigger is a $3\times3$ square at the bottom-right corner.}
\label{tab:topologies}
\setlength{\tabcolsep}{4pt}
\begin{tabular}{@{}l cccc@{}}
\toprule
\textbf{Topology} & \ac{CA} [\%] & \ac{ASR} [\%] & \ac{FPR} [\%] & \ac{TPR} [\%] \\
\midrule 
3-regular, $m=2, n=16, \kappa=1$ & 54.2 & 2.9 & 2.4 & 82.8 \\

Fully-connected, $m=4, n=8, \kappa=3$ & 52.8 &  26.0 & 0.0 & 16.4 \\

Erdős–Rényi-Gilbert, $m=3, n=32, \kappa=1$ & 51.8 & 0.6 & 0.3&99.7 \\

\bottomrule
\end{tabular}
\end{table}

\subsection{Robustness of \sys to different topologies and numbers of attacker nodes}
\label{subsec:exp_topologies}

To further understand the robustness of \sys, we vary both the type of the communication topology and the number of attackers.
We consider three topologies: our default 3-regular graph with $m=2$, an 8-node fully connected graph with $m=4$ (run for 250 rounds since only 4 honest nodes contribute), and an Erdős–Rényi-Gilbert random graph of $n=32$ nodes including $m=3$ malicious ones, where the probability of an edge being present is set to $p=0.103$ so that the expected degree matches the 3-regular graph.\footnote{The expected degree is $(n-1)p$, we set it to 3 to allow $\kappa=1$.}
We also vary the number of required peer confirmations $\kappa$ for the fully-connected graph, setting it to 3.
\Cref{tab:topologies} reports final \ac{CA}, \ac{ASR}, \ac{FPR} and \ac{TPR} on the \cifar dataset ($\alpha=0.5$).
In general, \sys remains effective across topologies and attacker counts.
However, we observe an \ac{ASR} of 26\% with the fully-connected graph, which we explain by the following two factors.
First, the attacker ratio is substantially higher (50\% compared to 12.5\% in the default setting), meaning honest nodes absorb proportionally more malicious updates before ejection.
Second, the higher confirmation threshold $\kappa=3$, which is necessary to maintain correctness guarantees under this attacker ratio, reduces the \ac{TPR} to 16.4\%.
This difficulty in detecting actual triggers slows down the ejection of attacker nodes.
However, training for additional rounds after ejection of the attacker nodes is likely to reduce the \ac{ASR}.

\begin{table}[t]
\centering
\caption{\textbf{Scaling the number of attacker nodes.} The \acf{CA} ($\uparrow$ is better) of \Oracle and \sys, and the \acf{ASR} ($\downarrow$ is better), \acf{FPR} ($\downarrow$ is better) and \acf{TPR} ($\uparrow$ is better) of \sys, with $m$ attacker nodes (nodes $0$ to $m-1$) in a 3-regular graph with $n=16$ nodes and $\kappa=1$. We use the \cifar ($\alpha=0.5$) dataset. All runs use the same seed.}
\label{tab:scaling_attackers}
\setlength{\tabcolsep}{4pt}
\begin{tabular}{@{}c c cccc@{}}
\toprule
 & \Oracle & \multicolumn{4}{c}{\sys} \\
\cmidrule(lr){2-2}\cmidrule(lr){3-6}
$m$ & \ac{CA} [\%] & \ac{CA} [\%] & \ac{ASR} [\%] & \ac{FPR} [\%] & \ac{TPR} [\%] \\
\midrule
1 & 56.4 & 55.3 & 0.7 & 3.1 & 83.8 \\
2 & 55.7 & 54.2 & 2.9 & 2.4 & 82.8 \\
3 & 55.9 & 52.4 & 6.5 & 4.6 & 83.3 \\
4 & 53.9 & 49.5 & 6.4 & 3.1 & 83.6 \\
5 & 52.8 & 51.1 & 6.1 & 0.5 & 80.8 \\
\bottomrule
\end{tabular}
\end{table}

\begin{table}[t]
\centering
\caption{\textbf{Scaling the network size and sparsity.} The \acf{CA} ($\uparrow$ is better), \acf{ASR} ($\downarrow$ is better) and \acf{FPR} ($\downarrow$ is better) of \sys with $n$ nodes, $m$ attacker nodes (keeping the attacker ratio constant) and $\kappa=1$. We use the \cifar ($\alpha=0.5$) dataset. On the ring, attacker nodes are placed apart so that each honest node keeps an honest witness.}
\label{tab:scaling_network}
\setlength{\tabcolsep}{4pt}
\begin{tabular}{@{}l cc ccc@{}}
\toprule
\textbf{Topology} & $n$ & $m$ & \ac{CA} [\%] & \ac{ASR} [\%] & \ac{FPR} [\%] \\
\midrule
3-regular & 16 & 2 & 54.2 & 2.9 & 2.4 \\
3-regular & 32 & 4 & 53.4 & 1.6 & 0.1 \\
Ring & 16 & 2 & 49.4 & 2.1 & 0.8 \\
Ring & 32 & 4 & 44.2 & 3.5 & 0.0 \\
\bottomrule
\end{tabular}
\end{table}

\paragraph{Scaling the number of attacker nodes.}
In \Cref{tab:topologies}, the topology and the number of attacker nodes vary jointly.
To isolate the effect of the latter, we fix our default 3-regular graph with $n=16$ nodes and $\kappa=1$, and increase the number of attacker nodes from $m=1$ to $m=5$.
\Cref{tab:scaling_attackers} reports the results on \cifar ($\alpha=0.5$).
Detection remains effective across all configurations, with an \ac{ASR} of at most 6.5\% (compared to 79.0\% without defense for $m=2$, see \Cref{tab:varying_xi}) and stable \ac{FPR} and \ac{TPR}.
The mild decrease in \ac{CA} as $m$ grows is independent of \sys and also observed for \Oracle: each additional attacker node replaces an honest contributor, which lowers the accuracy achievable by any defense, especially with \ac{NIID} data.

\paragraph{Larger and sparser networks.}
We also double the network size to $n=32$ nodes for both the 3-regular graph and the ring topology (the sparsest connected topology, with degree 2), keeping the attacker ratio fixed ($m=2$ for $n=16$ and $m=4$ for $n=32$).
On the ring, we place attacker nodes apart so that each honest node keeps an honest witness for its malicious neighbors (we study adjacent attacker nodes in \Cref{subsec:connectivity_assumption_violation}).
\Cref{tab:scaling_network} shows that the \ac{ASR} stays low as the network grows and becomes sparser, at most 3.5\%, with an \ac{FPR} of at most 2.4\%.
The lower \ac{CA} on the ring stems from its slower mixing, which slows down \ac{DL} convergence independently of the defense.
Such sparse topologies are particularly unfavorable to distance-based defenses like \multikrum, which must estimate outlier statistics from only two received updates per round.

\begin{table}[t]
\centering
\caption{\textbf{Connectivity assumption violations on \cifar.} The \acf{CA} ($\uparrow$ is better) and \acf{ASR} ($\downarrow$ is better) on the \cifar dataset with $\alpha=0.5$. We use a 3-regular graph with $n=16$ nodes, $\kappa=1$ and $m=4$ attackers (nodes \{0, 1, 6, 8\}). Node 3 is the only honest node whose malicious neighbor has no other honest witnesses (Node 1's other neighbors are malicious too).}
\label{tab:connectivity_assump}
\small
\setlength{\tabcolsep}{4pt}
\begin{tabular}{@{}l ccccccccccccc@{}}
\toprule
Honest node & 2 & 3 & 4 & 5 & 7 & 9 & 10 & 11 & 12 & 13 & 14 & 15 & \textbf{Avg.} \\
\midrule
\ac{CA} [\%] & 50.6 & 46.8 & 51.0 & 50.4 & 50.9 & 28.9 & 53.6 & 53.7 & 49.1 & 45.8 & 43.6 & 55.3 & $48.3\pm6.7$\\
\ac{ASR} [\%] & 1.1 & \textbf{15.8} & 5.5 & 2.5 & 5.8 & 0.0 & 1.7 & 3.2 & 11.0 & 3.2 & 7.5 & 7.7 & $5.4\pm4.4$\\
\bottomrule
\end{tabular}
\end{table}

\subsection{The performance of \sys when violating the connectivity assumption}
\label{subsec:connectivity_assumption_violation}

We assume in \Cref{subsec:threat_model} that each node has at least $\kappa +1 $ honest neighbors.
When this fails for a particular node, \sys may be unable to collect enough confirmations to confirm a true backdoor during the collaborative verification phase.
We evaluate the impact on \sys' efficiency if this assumption is violated.
Specifically, we consider a $n=16$ node setting, connected with a 3-regular graph and $m=4$ attackers.
We additionally set $\kappa=1$.
\Cref{tab:connectivity_assump} shows the per-node and averaged \ac{CA} and \ac{ASR} in a configuration where node 3 is the only honest neighbor of attacker node 1.
As expected, node 3 is significantly more affected by the backdoor than its peers.
The \ac{ASR} of node 3 reaches 15.8\%, which is the highest among all honest nodes and strongly above the average \ac{ASR} of 5.4\%.
Interestingly, even though node 3 is significantly more affected by the backdoor than its peers, its honest neighbors are not in turn corrupted by it: the collaborative verification mechanism allows them to identify and reject node 3's compromised updates, effectively containing the spread of the backdoor and preventing it from propagating further into the network.
Even in this pessimistic setting, \sys still effectively protects the rest of the network (the average \ac{ASR} among honest nodes is much lower than the 72.4\% with \Nodef in \Cref{fig:local_det_only}), since the remaining nodes still have enough honest neighbors to collectively identify malicious updates.
Even on node 3, the backdoor remains far weaker than without defense, since the malicious update is averaged with those of its two honest neighbors.
Degradation under violation of the assumption is thus graceful and localized.

\paragraph{Worst case on a ring.}
We push this setting further with the sparsest connected topology: a ring of $n=16$ nodes with $\kappa=1$ on \cifar ($\alpha=0.5$), where the $m=2$ attacker nodes (nodes 0 and 1) are adjacent.
Nodes 2 and 15 therefore have no honest witness for their malicious neighbor.
Both these nodes reach an \ac{ASR} of 48.5\%, while the remaining honest nodes average at 5.7\%: only the nodes that cannot verify are affected, and the rest of the network stays protected.
The higher \ac{ASR} of nodes 2 and 15 compared to node 3 above is expected, since on a ring the malicious update is averaged with a single honest update instead of two.
Finally, we recall that $\kappa$ is a design parameter that should be set from the graph degree and the expected fraction of attacker nodes (\Cref{subsubsec:sys_threshold_selection}), and may even vary across honest nodes.

\subsection{The effect of varying the similarity threshold}
\label{subsec:varying_xi}

\begin{table}[t]
\centering
\caption{\textbf{Varying the similarity threshold $\xi$}. The \acf{CA} ($\uparrow$ is better), \acf{ASR} ($\downarrow$ is better) and rejection rate ($\downarrow$ is better) using a 3-regular graph, $n=16$ nodes, $m=2$ attackers, and $\kappa=1$. We use the \cifar ($\alpha=0.5$) dataset. The calibrated value of $\xi$ is $\xi^*=0.42$ (see \Cref{sec:threshold_calibration_per_dataset}).}
\label{tab:varying_xi}
\setlength{\tabcolsep}{4pt}
\begin{tabular}{@{}cccc@{}}
\toprule
\textbf{Similarity threshold $\xi$} & \ac{CA} [\%] &\ac{ASR} [\%] & \ac{Rej. Rate} of all model updates [\%] \\
\midrule
0 (\Localdef) & 45.3 & 1.1 & 75.7 \\
0.27 & 44.1 & 1.5 & 37.7 \\
\textbf{0.42} & 54.2 & 2.9& 13.3 \\
0.57 & 55.3 & 5.7 & 11.4\\
0.72 & 50.9 & 47.7 & 8.4\\
1 (\Nodef) & 45.0 &79.0 & 0.0\\
\bottomrule
\end{tabular}
\end{table}

To quantify the effect of the similarity threshold $\xi$ on the effectiveness of \sys, we vary $\xi$ beyond the calibrated value $\xi^*=0.42$ (see \Cref{sec:threshold_calibration_per_dataset}).
We consider a setting with a 3-regular graph, $n=16$ nodes, $m=2$ attackers, and $\kappa=1$. We use the \cifar dataset with heterogeneity $\alpha=0.5$.
\Cref{tab:varying_xi} shows the \ac{CA}, \ac{ASR} and rejection rate (\ac{Rej. Rate}) of \emph{all} model updates in this setting for different similarity thresholds $\xi$.
We remark that $\xi=0$ corresponds to the \Localdef baseline while $\xi=1$ is equivalent to the \Nodef baseline.
Increasing the value of $\xi$ beyond the calibrated value decreases \ac{CA} and increases \ac{ASR} as fewer malicious model updates will be rejected.
On the other hand, a lower value of $\xi$ also decreases \ac{CA} since an increasing number of honest updates are rejected, but also reduces \ac{ASR}.
We also observed in our experiments that a small offset from the calibrated value of $\xi$ does not significantly impact the \ac{CA} nor the \ac{ASR}, therefore providing some flexibility to \sys users to estimate $\xi$.
Overall, these results suggest that the calibrated threshold $\xi$ sits in a wide basin of near-optimal performance, making \sys robust to moderate miscalculations of $\xi$ in practice.

\begin{table}[t]
\centering
\caption{\textbf{Computation and communication overhead of \sys.} Wall-clock training time ($\downarrow$ is better) and maximum bytes sent per round ($\downarrow$ is better) averaged over honest nodes and 3 seeds for \Nodef and \sys. We report the relative slowdown with regard to \Nodef in parentheses.}
\label{tab:overhead}
\setlength{\tabcolsep}{4pt}
\begin{tabular}{@{}l ccc ccc@{}}
\toprule
 & \multicolumn{3}{c}{Bytes per round [MB]} &\multicolumn{3}{c}{Time per round [s]} \\
\cmidrule(lr){2-4}\cmidrule(lr){5-7}
\textbf{Defense} & \cifar & \femnist & \imgnet & \cifar & \femnist & \imgnet \\
\midrule
\Nodef &59.25&20.28&135.37& 4.18 (1.0$\times$) & 21.1 (1.0$\times$) & 32.5 (1.0$\times$)\\
\sys &59.32&20.28&135.38& 22.5 (5.4$\times$) & 77.4 (3.7$\times$) & 291.6 (9.0$\times$)\\
\bottomrule
\end{tabular}
\end{table}

\subsection{Computation and communication overhead of \sys}
\label{subsec:exp_cost}

Finally, we quantify the computation and communication overhead introduced by \sys.
We adopt the same setting as used in the experiments described in \Cref{subsec:exp_main} and consider $m=2$ attacker nodes in a network with $n=16$ nodes, connected using a 3-regular graph.
\Cref{tab:overhead} reports the wall-clock time in seconds and the bytes exchanged per training round in MB for each dataset, with and without \sys.
All times and communication volumes are averaged across iterations over honest nodes and over three seeds. 
For \cifar, we also compute the mean values over all considered heterogeneity levels ($\alpha \in \{0.25, 0.5, \infty \}$).

\Cref{tab:overhead} shows that communication overhead induced by \sys is negligible, since each neighbor confirmation requires the exchange of small trigger masks.
The additional communication volume introduced by exchanging these masks is negligible compared to the size of exchanged models.
The computational overhead is, however, more pronounced and is primarily caused by the local trigger reverse-engineering step, where each honest node tries to reconstruct a trigger for each target class.
Since \sys has to optimize trigger masks for each possible target class, the computational cost grows with the number of classes.
\Cref{tab:overhead} shows that the increase in time per round is between 3.7$\times$ (for \femnist) and 9.0$\times$ (for \imgnet).
We remark that the \imgnet dataset contains 200 classes, compared to 10 and 62 classes for the \cifar and \femnist datasets, respectively.
Furthermore, larger sizes of input images also result in additional computation per optimization step.

\paragraph{Cost structure.}
Local trigger reverse-engineering accounts for more than 90\% of the overhead on \imgnet.
At each round, for each non-ejected neighbor and each of the $K$ candidate target classes, an honest node computes an initial candidate and performs $R$ refinement steps, each on its small validation set $\mathcal{D}_i^{val}$ (at most one sample per class, see \Cref{sec:exp_additional_details}).
The cost per round thus amounts to roughly $|N(i)|\cdot K\cdot R$ gradient steps over $\mathcal{D}_i^{val}$.
It grows with the number of classes and the node degree, but neither with the network size $n$ nor with the size of local datasets, two quantities that grow in real deployments.
It also decreases over the course of training, since ejected neighbors are skipped entirely and attacker nodes are ejected with high probability (\Cref{prop:node_ejection_bounds}).
Finally, our setup accentuates the relative overhead: nodes deliberately perform little local work per round (\eg, 5 batches of 64 samples on \cifar), whereas the detection cost does not depend on the local workload, so its share decreases as local workloads grow.

\paragraph{Reducing the overhead.}
We identify two promising avenues to reduce this computational cost.
First, since trigger recovery is performed independently per target class, these computations can be trivially parallelized, with each class assigned to a separate thread or device.
Second, we can significantly reduce the computational cost by exploiting the structure of realistic attack scenarios.
In practice, an attacker is unlikely to switch target classes mid-training, as doing so would require re-injecting an entirely new backdoor into the network from scratch, which is a costly and slow process.
\sys can exploit this temporal consistency: once a target class has been flagged as suspicious in a given round, subsequent rounds can concentrate their optimization budget on that class, while only spot-checking remaining classes with a reduced number of gradient steps or at a lower sampling rate.
This yields a significant reduction in per-round compute cost in the presence of an attacker, with little risk of missing a class switch.
On \imgnet ($K=200$), replacing the full scan over all classes by roughly one class could reduce the dominant cost term by up to two orders of magnitude.
In the benign case where no attacker is present, we can apply a complementary strategy: more compute budget can be spent at the beginning of training, when the model is still evolving rapidly and a backdoor could more easily take hold, and gradually reduced as training progresses and the model stabilizes.
Together, these two strategies form an adaptive compute schedule that naturally allocates more resources where and when they are most needed.
We leave a thorough investigation of these optimizations and other efficiency improvements to future work.

\section{Proof of Section~\ref{subsec:convergence rates}}\label{appendix:convergence_proof}

\subsection{Notations}

Recall that the averaging rule of~\Cref{alg:main} can be written as
\[
\model_i^{t+1} = \frac{1}{\nbreceived{i}+1}\left(\model_i^{t+\half} + \sum_{j\in N(i)} \delta_{i,j}^t \model_j^{t+\half}\right).
\]
Equivalently, in matrix form,
let
\[
\model^t:=\begin{bmatrix}(\model_1^t)^\top \\ \vdots \\ (\model_n^t)^\top\end{bmatrix}\in\R^{n\times d},
\qquad
G^t:=\begin{bmatrix}(g_1^t)^\top \\ \vdots \\ (g_n^t)^\top\end{bmatrix}\in\R^{n\times d},
\]
where $g_i^t$ is the stochastic gradient at node $i$ and round $t$.
\[
\model^{t+\half} = \model^t - \lrmain\nabla\losssampled(\model^t),
\qquad
\model^{t+1} = S^t\model^{t+\half},
\]
where $S^t$ is the random row-stochastic mixing matrix induced by the accepted links at round $t$ (and matches Stage $2$ of~\cref{alg:main}, see~\cref{sec:theory}). We use
\[
\Pi := I - \frac{1}{n}\mathbf{1}\mathbf{1}^\top,
\qquad
\widetilde{\model}^t := \Pi\model^t = \model^t-\mathbf{1}\,\overline{\model}^t,
\qquad
\overline{\model}^t := \frac{1}{n}\mathbf{1}^\top\model^t\in\R^d.
\]
In words, $\widetilde{\model}^t$ is the \emph{disagreement (centered) iterate}: its $i$-th row is $\model_i^t-\overline{\model}^t$, hence $\|\widetilde{\model}^t\|_F^2=\sum_{i=1}^n\|\model_i^t-\overline{\model}^t\|^2$.

\subsection{Auxiliary lemmas}

We state and prove here useful lemmas that will be used throughout the proofs.
\begin{lemma}\label{lem:distance to average decomposition}
    For all $x\in\R^{n}$, we have
    \begin{align}\label{eq:distance to average decomposition}
        \frac{1}{n}\sum_{i=1}^n \norm{x_i - \bar{x}}^2 = \frac{1}{2n^2}\sum_{i,j=1}^n \norm{x_i - x_j}^2,
    \end{align}
\end{lemma}
\begin{proof}
    We have
    \begin{align*}
        \frac{1}{2n^2}\sum_{i,j=1}^n \norm{x_i - x_j}^2
        =& \frac{1}{2n^2}\sum_{i,j=1}^n \norm{(x_i-\bar{x}) - (x_j-\bar{x})}^2\\
        =& \frac{1}{2n^2}\sum_{i,j=1}^n \norm{x_i-\bar{x}}^2 + \norm{x_j-\bar{x}}^2 - 2\langle x_i-\bar{x}, x_j-\bar{x}\rangle\\
        =& \frac{1}{n}\sum_{i=1}^n \norm{x_i - \bar{x}}^2,
    \end{align*}
    where the last equality uses \[\sum_{i,j=1}^n\langle x_i-\bar{x}, x_j-\bar{x}\rangle =\sum_{j=1}^{n} \langle \sum_{i=1}^{n}x_i-\bar{x}, x_j-\bar{x}\rangle = \sum_{j=1}^{n} \langle 0, x_j-\bar{x}\rangle = 0\].
\end{proof}

Moreover, we provide a characterization of how the link failures affect the expected gossip matrix $\E\left[S^t\right]$:
\begin{lemma}[Expected weights of the random mixing matrix]\label{lem:expected_mixing_matrix}
Let $d_i = |\mathcal{N}(i)|$ and $M_i^t = \sum_{k\in\mathcal{N}(i)} \delta_{i,k}^t$. Under \Cref{ass:two_rate_detection}, the entries of $S^t$ are given by:
\[
S_{i,i}^t = \frac{1}{1+M_i^t},
\qquad
S_{i,j}^t = \frac{\delta_{i,j}^t}{1+M_i^t}\ \text{for } j\in\mathcal{N}(i),
\]
and
\[
q_i := \E\left[S_{i,i}^t\right]
= \E\left[\frac{1}{1+M_i^t}\right]
= \frac{1-{\probareject}^{d_i+1}}{(d_i+1)(1-\probareject)}.
\]
Moreover,
\[
\bar{S}_{i,j}:=\E[S_{i,j}^t]= \begin{cases}
q_i & \text{if } j=i,\\
\frac{1-q_i}{d_i} & \text{if } j\in\mathcal{N}(i),\\
0 & \text{otherwise}.
\end{cases}
\]
In particular, $\bar{S}$ is row-stochastic. Moreover, if $\mathcal{W}$ is $d$-regular ($\forall i,j\in\nodeset, d_i = d_j = d$), then $\bar{S}$ is symmetric and doubly-stochastic.
\end{lemma}

\begin{proof}[Proof of~\cref{lem:expected_mixing_matrix}]
The identities for $S_{i,i}^t$ and $S_{i,j}^t$ follow directly from row normalization and $\delta_{i,i}^t=1$.

We now focus on the expectation derivations. Denote $p = 1-\probareject$. Since $M_i^t\sim\text{Binomial}(d_i,p)$,
\[
\E\left[S_{i,i}^t\right] 
= \E\left[\frac{1}{1+M_i^t}\right]
= \sum_{m=0}^{d_i}\frac{1}{m+1}\binom{d_i}{m}p^m{(1-p)}^{d_i-m}
= \frac{1-{(1-p)}^{d_i+1}}{(d_i+1)p}.
\]
All off-diagonal expectations on $\mathcal{N}(i)$ sum to $1-q_i$: $\sum_{j\in\mathcal{N}(i)}\E[S_{i,j}^t] = 1-q_i$.
By symmetry across neighbors, all off-diagonal expectations on $\mathcal{N}(i)$ are equal, hence $\E[S_{i,j}^t] = \frac{1-q_i}{d_i}$ for $j\in\mathcal{N}(i)$.
\end{proof}

\subsection{Proof of Proposition~\ref{prop:spectral_gap_mixing_matrix}}\label{subsec:proof of spectral mixing}
We now restate and prove~\cref{prop:spectral_gap_mixing_matrix}:
\MixingSpectrum*
\begin{proof}[Proof of~\cref{prop:spectral_gap_mixing_matrix}]

\textbf{Proof for the largest eigenvalue:}
We have:
\begin{align}
    \E[(S^t)^\top S^t]\1 = \E[(S^t)^\top S^t\1] = \E[(S^t)^\top\1] = \E[(S^t)^\top] \1 = \E[(S^t)]^\top \1 = \E[(S^t)]\1  = \1
\end{align}
where we used that $\E[S^t]$ is symmetric when $\mathcal{W}$ is $d$-regular thanks to~\cref{lem:expected_mixing_matrix}. This completes the proof for the largest eigenvalue.

\textbf{Proof for the second largest eigenvalue}
For simplicity, denote $q= \left(1-\probareject\right)$, and let us write:
\begin{align}    
    a := \E\!\left[\frac{1}{1+M}\right],\quad M\sim\mathrm{Binomial}(d,1-\probareject), \label{eq:a_def}
\end{align}
\begin{align}
    b := \left(1-\probareject\right)\,\E\!\left[\frac{1}{(2+U)^2}\right],\quad U\sim\mathrm{Binomial}(d-1,1-\probareject),\label{eq:b_def}
\end{align}
\begin{align}    
    c := \left(1-\probareject\right)^2\,\E\!\left[\frac{1}{(3+V)^2}\right],\quad V\sim\mathrm{Binomial}(d-2,1-\probareject) \label{eq:c_def}
    \quad (c:=0\text{ if }d=1).
\end{align}

Let
\[
Y := \E\left[(S^t)^\top S^t\right],
\qquad
Y_{j,k}=\sum_{i=1}^n \E\left[S^t_{i,j}S^t_{i,k}\right].
\]
Write $K_i:=1+M_i^t$ the number of effective neighbors for node $i$ at step $t$. 
For each row $i$, nonzero entries of $S^t$ are
$S^t_{i,i}=1/K_i$ and $S^t_{i,j}=\delta_{i,j}^t/K_i$ for $j\in\mathcal N(i)$.

For $j\neq k$, a term $S^t_{i,j}S^t_{i,k}$ is nonzero only if both columns $j,k$
appear in row $i$. This gives three cases:
\begin{enumerate}
    \item $i=j$ and $k\in\mathcal N(j)$: contribution $\E[\delta^t_{j,k}/K_j^2]=b$;
    \item $i=k$ and $j\in\mathcal N(k)$: contribution $\E[\delta^t_{k,j}/K_k^2]=b$, and is symmetric to the previous case;
    \item $i\notin\{j,k\}$ and $j,k\in\mathcal N(i)$: contribution
    $\E[\delta^t_{i,j}\delta^t_{i,k}/K_i^2]=c$.
\end{enumerate}
Hence, for $j\neq k$,
\begin{align}
Y_{j,k}
&= b\,\mathbbm{1}_{\{k\in\mathcal N(j)\}} + b\,\mathbbm{1}_{\{j\in\mathcal N(k)\}}
   + c\,\#\{i:\ j,k\in\mathcal N(i)\}.\label{eq:off-diagonal-terms_counting}
\end{align}

Since $\mathcal{W}$ represents a $d$-regular undirected graph, we can write $A$ to be its adjacency matrix with $A_{ii}=0$ for all $i$, and we have:
\begin{align}
    \mathcal{W} = \frac{1}{d+1}(I + A).\label{eq:W_from_A}
\end{align}

For an undirected $d$-regular graph with adjacency matrix $A$ following~\cref{eq:W_from_A}, \cref{eq:off-diagonal-terms_counting} becomes:
\[
Y_{j,k}=2b\,A_{j,k}+c\,(A^2)_{j,k},\qquad j\neq k.
\]

For the diagonal terms, we have:
\[
Y_{j,j}=\E\left[\frac{1}{K_j^2}\right]+\sum_{i\in\mathcal N(j)}\E\left[\frac{\delta_{i,j}^t}{K_i^2}\right].
\]
By $d$-regularity and exchangeability across neighbors (symmetry), we have that for any fixed $\ell\in\mathcal N(j)$:
\[
\sum_{i\in\mathcal N(j)}\E\left[\frac{\delta_{i,j}^t}{K_i^2}\right]
= d\,\E\left[\frac{\delta_{j,\ell}^t}{K_j^2}\right]
= \E\left[\frac{M_j^t}{K_j^2}\right].
\]
Thus, we have:
\[
Y_{j,j}=\E\left[\frac{1+M_j^t}{K_j^2}\right]=\E\left[\frac{1}{1+M_j^t}\right]=a.
\]

Since $(A^2)_{j,j}=d$, the compact matrix identity is
\[
Y = (a-cd)I + 2bA + cA^2
\]

Now, following~\cref{eq:W_from_A}, the matrices $\mathcal{W}$ and $A$ share eigenvectors. Using the eigenvalues $\mu_1,\ldots,\mu_n$ defined above, define
\[
\nu_k:=(d+1)\mu_k-1,
\]
so that $\nu_k$ is the corresponding eigenvalue of $A$.
Because $\mathcal{W}$ is symmetric (undirected gossip matrix), it has an orthonormal eigenbasis
$\{v_k\}_{k=1}^n$, which is also an eigenbasis of $A$. For each $k$:
\[
Yv_k = \big((a-cd)+2b\nu_k+c\nu_k^2\big)v_k = \phiprobarejectat{\nu_k}v_k.
\]
So the eigenvalues of $Y$ are exactly $\{\phiprobarejectat{\nu_k}\}_{k=1}^n$.
Using $\nu_k=(d+1)\mu_k-1$, define
\[
\psiprobarejectat{\mu}:=\phiprobarejectat{(d+1)\mu-1}.
\]
Then equivalently the eigenvalues of $Y$ are $\{\psiprobarejectat{\mu_k}\}_{k=1}^n$.

Finally, for $k\ge2$, we have $\mu_n\le\mu_k\le\mu_2$.
Also, $\psiprobareject''(\mu)=2c(d+1)^2\ge0$, so $\psiprobareject$ is convex on $[\mu_n,\mu_2]$.
Therefore, for each $k\ge2$,
\[
\psiprobarejectat{\mu_k}\le\max\!\left\{\psiprobarejectat{\mu_2},\ \psiprobarejectat{\mu_n}\right\}.
\]
Thus, we get:
\[
\max_{k\ge2}\psiprobarejectat{\mu_k}=\max\!\left\{\psiprobarejectat{\mu_2},\ \psiprobarejectat{\mu_n}\right\}.
\]
Moreover, since $\mu_1=1$, for every $k$:
\[
\psiprobarejectat{1}-\psiprobarejectat{\mu_k}
= \phiprobarejectat{d}-\phiprobarejectat{\nu_k}
= (d-\nu_k)\big(2b+c(d+\nu_k)\big)\ge0,
\]
because $\nu_k\in[-d,d]$ for a $d$-regular undirected adjacency matrix and $b,c\ge0$.
Hence $\lambda_1(Y)=\psiprobarejectat{1}$ and therefore
\[
\lambda_2(Y)=\max_{k\ge2}\psiprobarejectat{\mu_k}.
\]
Combining the two identities,
\[
\lambda_2\!\left(\E[(S^t)^\top S^t]\right)=\lambda_2(Y)
= \max\!\left\{\psiprobarejectat{\mu_2},\ \psiprobarejectat{\mu_n}\right\}.
\]
\end{proof}

\begin{figure}
    \centering
    \includegraphics[width=0.8\textwidth]{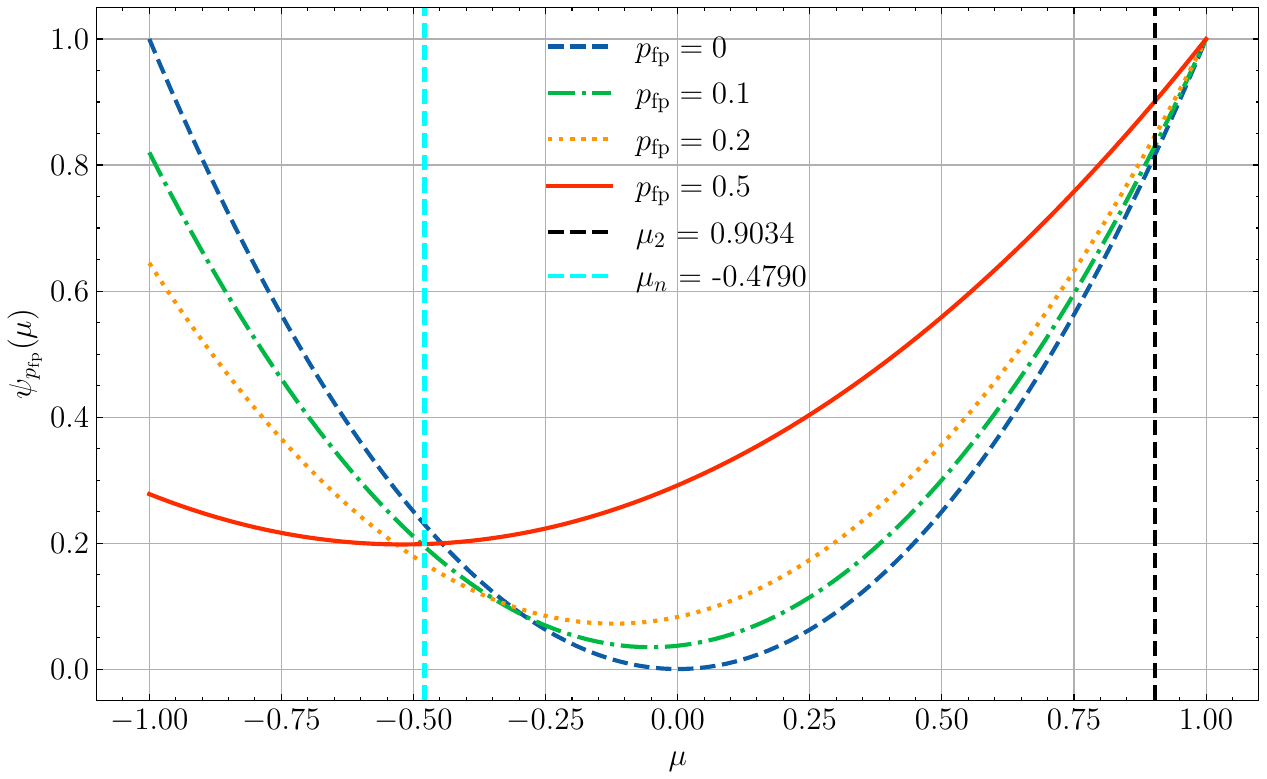}
    \caption{Evolution of the function $\psiprobareject$ for different values of $\probareject$. We represent the behavior on the eigenvalues $(\mu_2, \mu_n)$ for a given $3$-regular graph with 16 nodes.}%
    \label{fig:psi plot}
\end{figure}

In addition to~\cref{prop:spectral_gap_mixing_matrix}, we provide in~\cref{fig:psi plot} an analysis of the function $\psiprobareject$ as $\probareject$ varies, which controls the gossip factor of the random mixing matrix. We observe that as $\probareject$ increases, $\psiprobareject$ tends to favor positive eigenvalues.

\subsection{A mean-square contraction lemma}

We now connect~\cref{prop:spectral_gap_mixing_matrix} to the standard \emph{spectral gap} parameter used in~\ac{DL} analysis. Note that the relevant quantity is the second moment $\E\big[(S^t)^\top S^t\big]$~\cite{boydRandomizedGossipAlgorithms2006}.

\begin{lemma}[Mean-square consensus contraction]\label{lem:mean_square_contraction}
Assume the setting of~\cref{prop:spectral_gap_mixing_matrix} and define
\[
    \rho := \lambda_2\!\left(\E\big[(S^t)^\top S^t\big]\right)\in[0,1).
\]
Then, for any matrix $X\in\R^{n\times d}$ satisfying $\mathbf{1}^\top X = 0$, we have
\[
    \E\big\|S^t X\big\|_F^2\le \rho\,\|X\|_F^2.
\]
In particular,
\[
    \E\big\|\widetilde{\model}^{t+1}\big\|_F^2
    = \E\big\|\Pi\model^{t+1}\big\|_F^2
    \le \rho\,\E\big\|\widetilde{\model}^{t+\half}\big\|_F^2.
\]
\end{lemma}

\begin{proof}
Let $M := \E\big[(S^t)^\top S^t\big]$. For any $X\in\R^{n\times d}$,
\[
    \E\|S^t X\|_F^2
    = \E\,\mathrm{tr}\big(X^\top (S^t)^\top S^t X\big)
    = \mathrm{tr}\big(X^\top M X\big).
\]
Since $\mathbf{1}^\top X=0$, every column of $X$ lies in the subspace orthogonal to $\mathbf{1}$.
By~\cref{prop:spectral_gap_mixing_matrix}, $M$ is symmetric with eigenvectors given by those of $\mathcal{W}$, and the largest eigenvalue of $M$ on the subspace orthogonal to $\mathbf{1}$ is precisely $\rho$.
Therefore, for each column $x$ of $X$, we have $x^\top M x\le \rho\,\|x\|^2$, and summing over columns yields the claim.

Finally, using $\model^{t+1}=S^t\model^{t+\half}$ and the identity $\widetilde{\model}^{t+\half}=\Pi\model^{t+\half}$ with $\mathbf{1}^\top\widetilde{\model}^{t+\half}=0$, the last inequality follows by applying the first bound to $X=\widetilde{\model}^{t+\half}$.
\end{proof}

\subsection{Proof of Theorem~\ref{thm:convergence}}

We adapt the convergence analysis of \ac{EL}~\citep{devosEpidemicLearningBoosting2023} under our scenario. The main differences are that we consider a fixed underlying $d$-regular topology and not a randomized $d$-regular topology at each round.
We will thus perform the same assumption as~\ac{EL}, namely $L$-smoothness, bounded variance, and bounded heterogeneity.

The only difference in the proof lies in the mixing lemmas~\cite[Lemma 1]{devosEpidemicLearningBoosting2023}, where the contraction factor of the mixing step is controlled by a value defined by the authors. We revisit this lemma with constants adapted to our setting:

\begin{lemma}[Contraction Factor for \sys]\label{lem:EL lemma 1}
    Consider~\cref{alg:main}. Let $n\geq 2, d\geq 1, T\geq 1$ and $t\in [|0,T-1|]$. Assume the underlying communication graph is $d$-regular and undirected, and the filtering masks are independent of the stochastic gradient noise at round $t$. Let $\rho := \lambda_2\left(\E\big[(S^t)^\top S^t\big]\right)\in[0,1)$. Then, 
    \begin{enumerate}[label=(\alph*)]
        \item $\E\left[\overline{\model}^{t+1}\right] = \E\left[\overline{\model}^{t+\half}\right]$
        \item $\frac{1}{n}\sum_{i\in\nodeset}\E\left[\norm{\model_i^{t+1} - \overline{\model}^{t+1}}^2\right] \leq \frac{\rho}{n}\sum_{i\in\nodeset}\E\left[\norm{\model_i^{t+\half} - \overline{\model}^{t+\half}}^2\right]$
        \item $\E\left[\norm{\overline{\model}^{t+1} - \overline{\model}^{t+\half}}^2 \right] \leq \frac{\beta(d,\probareject)}{n} \frac{1}{n}\sum_{i\in\nodeset}\E\left[\norm{\model_i^{t+\half} - \overline{\model}^{t+\half}}^2\right]$
    \end{enumerate}
    with 
    \begin{align}
        \beta(d,\probareject) = \frac{1}{n}\E\norm{(S^t)^\top \mathbf{1} - \mathbf{1}}^2
        \label{eq:beta definition}
    \end{align}
\end{lemma}

\begin{proof}[Proof of~\cref{lem:EL lemma 1}]
Let $J := \frac{1}{n}\mathbf{1}\mathbf{1}^\top$ and $\Pi := I - J$ denote respectively the averaging operator and the orthogonal projection onto the disagreement subspace. We denote $\overline{\model}^t = J \model^t$.

\textbf{(a)} Since link dropout is symmetric across nodes, $\E[S^t]$ is a symmetric row-stochastic matrix, and thus it is doubly stochastic. Therefore,
\[
\E[\overline{\model}^{t+1}] = \E\left[\frac{1}{n}\mathbf{1}^\top S^t \model^{t+\half}\right] = \E\left[\frac{1}{n}\mathbf{1}^\top \E[S^t] \model^{t+\half}\right] = \E[\overline{\model}^{t+\half}].
\]

\textbf{(b)} $\Pi\model$ represents the disagreement. Using $\Pi S^tJ = 0$ and $\Pi J = 0$, we get:
\[
\Pi \model^{t+1} = \Pi S^t \model^{t+\half} = \Pi S^t \Pi \model^{t+\half}.
\]
Since $\mathbf{1}^\top \Pi\model^{t+\half}=0$, we can apply~\cref{lem:mean_square_contraction} with $X=\Pi\model^{t+\half}$ to obtain
\[
\E\norm{\Pi\model^{t+1}}^2_F
= \E\norm{\Pi S^t \Pi \model^{t+\half}}^2_F
\leq \E\norm{S^t X}^2_F
\leq \rho\,\E\norm{X}^2_F
= \rho\,\E\norm{\Pi\model^{t+\half}}^2_F.
\]
Rewriting $\norm{\Pi\model}_F^2=\sum_{i=1}^n\norm{\model_i-\overline{\model}}^2$ yields the claim.

\textbf{(c)}
We start from
\[
\overline{\model}^{t+1} - \overline{\model}^{t+\half}
= \frac{1}{n}\mathbf{1}^\top (S^t - I)\model^{t+\half}
= \frac{1}{n}\mathbf{1}^\top (S^t - I)\Pi \model^{t+\half},
\]
since $(S^t - I)\mathbf{1} = 0$. Therefore,
\[
\overline{\model}^{t+1} - \overline{\model}^{t+\half}
= \frac{1}{n}\bigl((S^t)^\top \mathbf{1} - \mathbf{1}\bigr)^\top \Pi \model^{t+\half}.
\]
By Cauchy--Schwarz,
\[
\norm{\overline{\model}^{t+1} - \overline{\model}^{t+\half}}^2
\le \frac{1}{n^2}
\norm{(S^t)^\top \mathbf{1} - \mathbf{1}}^2
\cdot
\norm{\Pi \model^{t+\half}}^2_F.
\]
Using $\norm{\Pi \model^{t+\half}}^2_F = \sum_{i\in\nodeset}\norm{\model_i^{t+\half} - \overline{\model}^{t+\half}}^2$
and taking expectation,
\[
\E\norm{\overline{\model}^{t+1} - \overline{\model}^{t+\half}}^2
\le
\underbrace{\frac{1}{n}\E\norm{(S^t)^\top \mathbf{1} - \mathbf{1}}^2}_{=:\beta(d,\probareject)}
\cdot
\frac{1}{n}
\E\sum_{i\in\nodeset}\norm{\model_i^{t+\half} - \overline{\model}^{t+\half}}^2.
\]
where we used the fact that the filtering masks $S^t$ are independent of the gradients computed in the same round. 

\end{proof}

Moreover, the uniform bound on the model drift and gradient drift remains the same as existing literature~\cite[Lemma 2]{devosEpidemicLearningBoosting2023}, with the mixing factor $\beta_s$ replaced by $\rho$ in our setting, as the proof only relies on contraction properties of the mixing step and not on the average drift.

\begin{lemma}[Uniform bound on model and gradient drift]\label{lem:uniform bounds}
    Consider a stepsize $\lrmain\leq\frac{1}{20L}$ and identical initial models. For any $t\geq 0$, we have:
    \begin{align}
        \frac{1}{n} \sum_{i\in\nodeset} \E\left[\norm{\model_i^t - \overline{\model}^t}^2\right] \leq 40 \frac{1+3\rho}{(1-\rho)^2} \rho \lrmain^2 (\boundheterogeneity^2 + \boundstochasticnoise^2),
    \end{align}
    and 
    \begin{align}
        \frac{1}{n^2}\sum_{i,j\in\nodeset} \E\left[\norm{g_i^t - g_j^t}^2\right] \leq 15 (\boundheterogeneity^2 + \boundstochasticnoise^2).
    \end{align}
\end{lemma}

\begin{proof}
    We unroll the proof of De Vos et al.~\cite[Lemma 2]{devosEpidemicLearningBoosting2023}, with $\beta_s$ replaced by $\rho$ following~\cref{lem:EL lemma 1}. This is the only part of the proof of the lemma where the communication graph comes into play, ensuring that a similar bound holds in our setting.
\end{proof}

Finally, the rest of the proof of~\cref{thm:convergence} follows by combining the descent lemma with the bounds on the average drift and gradient drift, in a similar fashion as De Vos et al.~\cite[Theorem 1]{devosEpidemicLearningBoosting2023}.

\ConvergenceTheorem*

For clarity, we write here the value of the stepsize $\lrmain$ before proceeding with the proof:
\begin{align}\label{eq:stepsize constant}
        \lrmain
        := \min\left\{
            \frac{1}{20L}
            ,
            \sqrt{
                \frac{
                    4n\,\Delta_0
                }{
                    T\, L\left(
                    4\boundstochasticnoise^2
                    +
                    \left(
                        60
                        + 320 C_{\rho}
                    \right)\beta(d,\probareject) (\boundheterogeneity^2 + \boundstochasticnoise^2)\right)
                }
            }
            ,
            \left(
                \frac{\Delta_0}{T\, 160L^2 C_{\rho} (\boundheterogeneity^2 + \boundstochasticnoise^2) }
            \right)^{\!\frac{1}{3}}
        \right\}.
\end{align}

\begin{proof}[Proof of~\cref{thm:convergence}]
    We derive a similar proof to~\ac{EL}~\citep{devosEpidemicLearningBoosting2023}. Using the usual convergence lemma, we derive~\cite[Equation 6]{devosEpidemicLearningBoosting2023}:
    \begin{align}
        \nonumber
        \frac{1}{n}\sum_{i\in\nodeset} \E\left[\norm{\nabla \locallosssampled{}(\model^t_i)}^2\right] 
        \leq& 
        \frac{4}{\lrmain}\E\left[\locallosssampled{}(\overline{\model}^t) - \locallosssampled{}(\overline{\model}^{t+1}) \right] + \frac{4L^2}{n}\sum_{i\in\nodeset} \E\left[\norm{\model_i^t - \overline{\model}^t}^2\right]
        \\&+4L\lrmain \frac{\boundstochasticnoise^2}{n} +  \frac{4L}{\lrmain}\E\left[\norm{\overline{\model}^{t+1} - \overline{\model}^{t+\half}}^2\right].\label{eq:descent lemma EL}
    \end{align}
    We can now bound terms individually: using~\cref{lem:EL lemma 1}, we have:
    \begin{align}
        \E\left[\norm{\overline{\model}^{t+1} - \overline{\model}^{t+\half}}^2\right]
        &\nonumber \leq \frac{\beta(d,\probareject)}{n}\frac{1}{n}\sum_{i\in\nodeset}\E\left[\norm{\model_i^{t+\half} - \overline{\model}^{t+\half}}^2\right]
        \\&\nonumber \stackrel{\eqref{eq:distance to average decomposition}}{=} \frac{\beta(d,\probareject)}{2n} \frac{1}{n^2}\sum_{i,j\in\nodeset}\E\left[\norm{\model_i^{t+\half} - \model_j^{t+\half}}^2\right]
        \\&\nonumber = \frac{\beta(d,\probareject)}{2n}\frac{1}{n^2}\sum_{i,j\in\nodeset}\E\left[\norm{\model_i^{t} - \model_j^{t} - \lrmain(g_i^t - g_j^t)}^2\right]
        \\&\nonumber \leq \frac{\beta(d,\probareject)}{n^3}\sum_{i,j\in\nodeset}\E\left[\norm{\model_i^{t} - \model_j^{t}}^2\right] + \frac{\beta(d,\probareject)\lrmain^2}{n^3}\sum_{i,j\in\nodeset}\E\left[\norm{g_i^t - g_j^t}^2\right]
        \\&\nonumber \stackrel{\eqref{eq:distance to average decomposition}}{=} \frac{2\beta(d,\probareject)}{n^2}\sum_{i\in\nodeset}\E\left[\norm{\model_i^{t} - \overline{\model}^{t}}^2\right] + \frac{\beta(d,\probareject)\lrmain^2}{n^3}\sum_{i,j\in\nodeset}\E\left[\norm{g_i^t - g_j^t}^2\right].
    \end{align}
    We can now plug this into \cref{eq:descent lemma EL} to get:
    \begin{align}
        \frac{1}{n}\sum_{i\in\nodeset}\E\left[\norm{\nabla \locallosssampled{}(\model^t_i)}^2\right]
        \leq& 
        \frac{4}{\lrmain}\E\left[\locallosssampled{}(\overline{\model}^t) - \locallosssampled{}(\overline{\model}^{t+1}) \right] 
        + \left(4L^2+\frac{8L\beta(d,\probareject)}{n\lrmain}\right)\frac{1}{n}\sum_{i\in\nodeset} \E\left[\norm{\model_i^t - \overline{\model}^t}^2\right]
        \nonumber\\&+
        4L\lrmain \frac{\boundstochasticnoise^2}{n} 
        + \frac{4L\beta(d,\probareject)\lrmain}{n} \frac{1}{n^2}\sum_{i,j\in\nodeset}\E\left[\norm{g_i^t - g_j^t}^2\right]
        \nonumber\\\stackrel{(a)}{\leq}&
        \frac{4}{\lrmain}\E\left[\locallosssampled{}(\overline{\model}^t) - \locallosssampled{}(\overline{\model}^{t+1}) \right]
        +4L\lrmain \frac{\boundstochasticnoise^2}{n}
        + \frac{4L\beta(d,\probareject)\lrmain}{n} 15 (\boundheterogeneity^2 + \boundstochasticnoise^2)
        \nonumber\\&  
        + \left(4L^2+\frac{8L\beta(d,\probareject)}{n\lrmain}\right)\left(40 \frac{1+3\rho}{(1-\rho)^2} \rho \lrmain^2 (\boundheterogeneity^2 + \boundstochasticnoise^2)\right) 
        \label{eq:descent lemma fully expanded}
    \end{align}
    where $(a)$ follows from~\cref{lem:uniform bounds}. Denoting 
    \begin{align}
        \Delta_0 := \left[\locallosssampled{}(\overline{\model}^0) - \locallosssampled{}^*\right]    
    \end{align}
    and 
    \begin{align}
        C_\rho := \frac{1+3\rho}{\left(1-\rho\right)^2}\rho,\label{eq:c rho definition}
    \end{align}
    and summing over $t=0$ to $T-1$, we get:

    \begin{align}
        \frac{1}{nT}\sum_{i\in\nodeset}\sum_{t=0}^{T-1}\E\left[\norm{\nabla \locallosssampled{}(\model^t_i)}^2\right]
        \leq&
        \frac{4}{T\lrmain}\Delta_0
        +4L\lrmain \frac{\boundstochasticnoise^2}{n}
        + \frac{4L\beta(d,\probareject)\lrmain}{n} 15 (\boundheterogeneity^2 + \boundstochasticnoise^2)
        \nonumber\\&  
        + \left(4L^2+\frac{8L\beta(d,\probareject)}{n\lrmain}\right)\left(40 C_{\rho} \lrmain^2 (\boundheterogeneity^2 + \boundstochasticnoise^2)\right)
        \nonumber\\=&
        \frac{4}{T\lrmain}\Delta_0
        +4L\boundstochasticnoise^2\frac{\lrmain}{n}
        + 60L\beta(d,\probareject)\frac{\lrmain}{n} (\boundheterogeneity^2 + \boundstochasticnoise^2)
        \nonumber\\&  
        + \left(160L^2C_{\rho} \lrmain^2+320L\beta(d,\probareject) C_{\rho} \frac{\lrmain}{n}\right)(\boundheterogeneity^2 + \boundstochasticnoise^2)
        \nonumber\\=&
        \frac{4}{T\lrmain}\Delta_0
        + \lrmain^2\left(160L^2C_{\rho} \right)(\boundheterogeneity^2 + \boundstochasticnoise^2) 
        \nonumber\\&  
        + \frac{\lrmain}{n} \left(
            4L\boundstochasticnoise^2
            +
            \left(
                60
                + 320 C_{\rho}
            \right)L\beta(d,\probareject) (\boundheterogeneity^2 + \boundstochasticnoise^2)
        \right)
        \nonumber\\=&
        \frac{4}{T\lrmain}\Delta_0
        + \lrmain^2\left(160L^2C_{\rho} \right)(\boundheterogeneity^2 + \boundstochasticnoise^2) 
        \nonumber\\&  
        + \frac{\lrmain}{n}L \left(
            4\boundstochasticnoise^2
            +
            \left(
                60
                + 320 C_{\rho}
            \right)\beta(d,\probareject) (\boundheterogeneity^2 + \boundstochasticnoise^2)
        \right)
        \label{eq:telescoping sum}
    \end{align}
    
    We now select the stepsize $\lrmain$ as in~\cref{eq:stepsize constant}.
    Plugging this into~\cref{eq:telescoping sum}, we absorb constants and use that $\frac{1}{\lrmain}$ is bounded by the sum of the inverse of each of the three terms in the minimum above, yielding that $\frac{1}{nT}\sum_{i\in\nodeset}\sum_{t=0}^{T-1}\E\left[\norm{\nabla \locallosssampled{}(\model^t_i)}^2\right]$ is bounded by:
    \begin{align}
        \mathcal{O}\left(
            \frac{L}{T}\Delta_0
            + \sqrt{\frac{L\Delta_0C_{\rho}\beta(d,\probareject) (\boundheterogeneity^2 + \boundstochasticnoise^2)}{nT}}
            + \sqrt[3]{\frac{L^2\Delta_0^2 C_{\rho} (\boundheterogeneity^2 + \boundstochasticnoise^2)}{T^2}}
        \right)
        .\label{eq:final convergence bound}
    \end{align}

\end{proof} 
\section{Ejection probability bounds}\label{appendix:generic_ejection_theorem}
We prove~\cref{prop:node_ejection_bounds} by first deriving a generic ejection bound in~\cref{thm:generic_ejection_bounds} and then instantiating it for malicious and honest nodes in~\cref{appendix:proof_node_ejection_bounds}.

\begin{theorem}[Ejection probability bounds]\label{thm:generic_ejection_bounds}
    Fix $T\ge 1$ and integers $k_1\ge1$, $1\le k_2\le k_3$.
    Let $Z_1,\dots,Z_T$ be i.i.d.\ Bernoulli$(p)$ rejection indicators.
    Assume the state rule is: enter \Suspected{} after $k_1$ consecutive detections, then become \Ejected{}
    if at least $k_2$ detections occur in the next $k_3$ rounds; otherwise, the node returns to \Trusted{}.

    Define
    \[
        A_T:=\mathbb{I}_{\{\text{the node is \Ejected{} by the end of round }T\}},
    \]
    \[
        \pi(p):=p^{k_1}\sum_{r=k_2}^{k_3}\binom{k_3}{r}p^r(1-p)^{k_3-r},
    \]
    \[
        \hat{\pi}(p):=(1-p)\,\pi(p),
        \qquad
        M_T:=\left\lfloor\frac{T}{1+k_1+k_3}\right\rfloor,
        \qquad
        N_T:=\max\{T-k_1-k_2+1,0\}.
    \]
    Then
    \[
        1-(1-\hat{\pi}(p))^{M_T}
        \le \mathbb{P}(A_T=1)
        \le N_T\,\pi(p),
    \]
    where the lower bound is stated for $k_2=1$, the setting of our experiments (\cref{subsec:trust_state_machine_details}); \cref{prop:markov_ejection_bounds} gives a lower bound valid for all thresholds.
\end{theorem}

\subsection{Proof of Theorem~\ref{thm:generic_ejection_bounds}}\label{appendix:proof_generic_ejection_bounds}
\begin{proof}[Proof of~\cref{thm:generic_ejection_bounds}]
    We prove both upper and lower bounds separately.

    \textbf{Upper bound:}
    For the upper bound, if $N_T=0$ there is no full window of length $k_1+k_3$ and the claim is immediate.
    Assume $N_T>0$.
    For each $i\in\{1,\dots,N_T\}$, define
    \[
        C_i:=\left\{Z_i=\cdots=Z_{i+k_1-1}=1,\;\sum_{j=i+k_1}^{i+k_1+k_3-1}Z_j\ge k_2\right\}.
    \]
    Any ejection by round $T$ must be triggered by such a window, so
    \[
        \{A_T=1\}\subseteq \bigcup_{i=1}^{N_T}C_i.
    \]
    Hence, by the union bound,
    \[
        \mathbb{P}(A_T=1)
        \le \sum_{i=1}^{N_T}\mathbb{P}(C_i)
        = N_T\,\pi(p),
    \]
    where the last equality uses i.i.d. Bernoulli$(p)$ indicators.

    \textbf{Lower bound:}
    We partition time into disjoint blocks of length $k' = 1+k_1+k_3$.
    For each $\ell\in\{0,\dots,M_T-1\}$, define
    \[
        E_\ell:=\left\{
        Z_{\ell k'+1}=0,\;
        Z_{\ell k'+2}=\cdots=Z_{\ell k'+1+k_1}=1,\;
        \sum_{j=\ell k'+2+k_1}^{(\ell+1) k'}Z_j\ge k_2
        \right\}.
    \]

    Let $k_2=1$. We claim that $E_\ell$ implies ejection by the end of block $\ell$, whatever the state of the node at the start of the block.
    If the node is \Suspected{} when the first of the $k_1$ consecutive detections occurs, this detection ejects it.
    Otherwise it is \Trusted{} at that time, with no pending run of detections thanks to the initial zero (an evaluation window can only close on a round without detection).
    It then enters \Suspected{} after the $k_1$ detections, and the first detection among the next $k_3$ rounds ejects it.
    For $k_2\ge2$ this argument no longer applies, as detections of block $\ell$ may fall in an evaluation window opened in the previous block and be discarded when that window closes; this is why the lower bound is restricted to $k_2=1$.
    Hence,
    \[
        \bigcup_{\ell=0}^{M_T-1}E_\ell \subseteq \{A_T=1\}.
    \]

    The events $E_\ell$ are independent (they depend on disjoint time blocks), and
    \[
        \mathbb{P}(E_\ell)=\hat{\pi}(p)\qquad\text{for all }\ell.
    \]
    Therefore,
    \[
        \mathbb{P}(A_T=1)
        \ge \mathbb{P}\!\left(\bigcup_{\ell=0}^{M_T-1}E_\ell\right)
        = 1-(1-\hat{\pi}(p))^{M_T}.
    \]
\end{proof}

\subsection{Proof of Proposition~\ref{prop:node_ejection_bounds}}\label{appendix:proof_node_ejection_bounds}

\NodeEjectionBounds*
\begin{proof}[Proof of~\cref{prop:node_ejection_bounds}]
    Under \cref{ass:two_rate_detection}, the per-round rejection indicators are \iid Bernoulli, with
    \[
        1-\delta_{i,j}^t\sim\operatorname{Bernoulli}(1-\probafalsenegative)\ \ (j\in\maliciousset),
        \qquad
        1-\delta_{i,j}^t\sim\operatorname{Bernoulli}(\probareject)\ \ (j\in\honestset).
    \]

    Let
    \[
        A_T^{\mathrm{mal}}:=\mathbb{I}_{\{\text{an attacker node is \Ejected{} by the end of round }T\}},
        \qquad
        B_T^{\mathrm{mal}}:=1-A_T^{\mathrm{mal}}.
    \]
    Applying the lower bound in \cref{thm:generic_ejection_bounds} with $p=1-\probafalsenegative$ gives
    \[
        \mathbb{P}(A_T^{\mathrm{mal}}=1)
        \ge 1-\left(1-\probafalsenegative\pi(1-\probafalsenegative)\right)^{\left\lfloor T/(k_1+k_3+1)\right\rfloor},
    \]
    hence
    \[
        \mathbb{P}(B_T^{\mathrm{mal}}=1)
        \le \left(1-\probafalsenegative\pi(1-\probafalsenegative)\right)^{\left\lfloor T/(k_1+k_3+1)\right\rfloor}.
    \]

    For the second bound, let:
    \[
        A_T^{\mathrm{hon}}:=\mathbb{I}_{\{\text{an honest node is \Ejected{} by the end of round }T\}}.
    \]
    Applying the upper bound in \cref{thm:generic_ejection_bounds} with $p=\probareject$ gives
    \[
        \mathbb{P}(A_T^{\mathrm{hon}}=1)
        \le \max\{T-k_1-k_2+1,0\}\,\pi(\probareject).
    \]
\end{proof}

\subsection{Extension: correlated detections via a two-state Markov chain}
\label{appendix:markov_ejection_bounds}

\Cref{ass:two_rate_detection} treats the per-round rejection indicators as independent across rounds whereas, as noted in \cref{sec:theory}, model rejections are temporally correlated in practice.
We relax independence to a stationary first-order Markov chain, which adds a single correlation parameter per edge type, and extend \cref{prop:node_ejection_bounds} to this setting.
For $i\in\honestset$ and $j\in N(i)$, write $Z_{i,j}^t$ for the detection indicator at round $t$, \ie the outcome of collaborative verification that drives the trust state machine.
In the model of \cref{ass:two_rate_detection}, $Z_{i,j}^t=1-\delta_{i,j}^t$.

\begin{assumption}[Markov two-rate detection model]\label{ass:markov_two_rate_detection}
    For each edge $(i,j)$, $(Z_{i,j}^t)_{t\ge0}$ is a stationary first-order Markov chain on $\{0,1\}$ with marginal $\mathbb{P}(Z_{i,j}^t=1)=p\in(0,1)$ ($p=\probareject$ for $j\in\honestset$, $p=1-\probafalsenegative$ for $j\in\maliciousset$) and transition probabilities
    \[
        a:=\mathbb{P}(Z_{i,j}^t=1\mid Z_{i,j}^{t-1}=1) = p+(1-p)\varrho,
        \qquad
        b:=\mathbb{P}(Z_{i,j}^t=1\mid Z_{i,j}^{t-1}=0) = p(1-\varrho),
    \]
    where
    \[
        \varrho=a-b=\mathrm{Corr}(Z_{i,j}^t,Z_{i,j}^{t+1})\in[\varrho_{\min}(p),1),
        \qquad
        \varrho_{\min}(p):=\max\!\left(-\frac{1-p}{p},\,-\frac{p}{1-p}\right),
    \]
    is the (edge-type-specific) lag-$1$ correlation $\corrreject$ (honest edges) or $\corrfalseneg$ (attacker edges).
    Chains on distinct edges are independent.
\end{assumption}

Taking $\varrho=0$ gives $a=b=p$, \ie \cref{ass:two_rate_detection}.
To handle the correlation, we track the rejection indicator jointly with the trust state machine.
Let $\mathcal{X}$ be the finite set of configurations of the state machine of \cref{thm:generic_ejection_bounds}: \Trusted{} with a current run of fewer than $k_1$ consecutive detections, \Suspected{} with the number of rounds elapsed in the evaluation window (fewer than $k_3$) and the number of detections so far (fewer than $k_2$), and \Ejected{}.
The state rule defines a deterministic update $x_{t}=\tau(x_{t-1},Z_{t})$, so that under \cref{ass:markov_two_rate_detection} the pair $(x_t,Z_t)_{t\ge0}$ is a time-homogeneous Markov chain on $\mathcal{X}\times\{0,1\}$, in which \Ejected{} is absorbing and $(x_0,Z_0)$ is the state before the first round.
We denote its transition matrix by $\mathbf{Q}$ and, for $L\ge1$, define
\[
    q_{\min}(L):=\min_{\substack{x\in\mathcal{X}\setminus\{\Ejected\}\\z\in\{0,1\}}}\ \sum_{z'\in\{0,1\}}\big(\mathbf{Q}^L\big)_{(x,z),(\Ejected,z')},
\]
the probability of reaching \Ejected{} within $L$ rounds from the least favorable configuration.
We also let
\[
    w_{k_2,k_3}(a,b) := \mathbb{P}\!\left(\textstyle\sum_{s=1}^{k_3}Z_s\ge k_2 \,\middle|\, Z_0=1\right)
    \qquad\text{and}\qquad
    \pi_\varrho(p) := p\,a^{k_1-1}\,w_{k_2,k_3}(a,b).
\]

\begin{proposition}[Ejection bounds under Markov-correlated detections]\label{prop:markov_ejection_bounds}
    Fix $T\ge1$ and integers $k_1\ge1$, $1\le k_2\le k_3$, and let $A_T$ and $N_T$ be as in \cref{thm:generic_ejection_bounds}.
    Under \cref{ass:markov_two_rate_detection},
    \[
        \mathbb{P}(A_T=1)\le N_T\,\pi_\varrho(p),
    \]
    and, for any integer $1\le L\le T$, letting $K_L:=\lfloor T/L\rfloor$,
    \[
        \mathbb{P}(A_T=1)\ \ge\ 1-\big(1-q_{\min}(L)\big)^{K_L},
        \qquad\text{hence}\qquad
        \mathbb{P}(A_T=1)\ \ge\ \max_{1\le L\le T}\Big[1-\big(1-q_{\min}(L)\big)^{K_L}\Big].
    \]
\end{proposition}
\begin{proof}
    \textbf{Upper bound.}
    As in \cref{appendix:proof_generic_ejection_bounds}, any ejection by round $T$ is triggered by one of the windows $C_i$, $i\in\{1,\dots,N_T\}$, so $\{A_T=1\}\subseteq\bigcup_{i=1}^{N_T}C_i$.
    This inclusion only depends on the state rule, not on the law of the indicators.
    By stationarity and the Markov property, $\mathbb{P}(C_i)=p\,a^{k_1-1}\,w_{k_2,k_3}(a,b)=\pi_\varrho(p)$ for every $i$, and the union bound, which requires no independence, concludes.

    \textbf{Lower bound.}
    Partition the first $K_LL$ rounds into $K_L$ consecutive blocks of length $L$, and let $\mathcal{F}_{\ell L}$ be the history up to the start of block $\ell\in\{0,\dots,K_L-1\}$.
    By the Markov property of $(x_t,Z_t)$, the conditional probability of reaching \Ejected{} during block $\ell$ given $\mathcal{F}_{\ell L}$ only depends on the configuration at the start of the block, and is at least $q_{\min}(L)$ whenever the node is not yet ejected.
    Hence $\mathbb{P}\big(x_{(\ell+1)L}\ne\Ejected\mid\mathcal{F}_{\ell L}\big)\le\big(1-q_{\min}(L)\big)\,\mathbb{I}_{\{x_{\ell L}\ne\Ejected\}}$, and applying the tower property over $\ell=K_L-1,\dots,0$ yields $\mathbb{P}(A_T=0)\le\mathbb{P}(x_{K_LL}\ne\Ejected)\le(1-q_{\min}(L))^{K_L}$.
\end{proof}

Applying \cref{prop:markov_ejection_bounds} with $(p,\varrho)=(\probareject,\corrreject)$ and $(p,\varrho)=(1-\probafalsenegative,\corrfalseneg)$ yields the counterpart of \cref{prop:node_ejection_bounds}:
\[
    \mathbb{P}(A_T^{\mathrm{hon}}=1)\le\max\{T-k_1-k_2+1,0\}\,\pi_{\corrreject}(\probareject),
    \quad
    \mathbb{P}(B_T^{\mathrm{mal}}=1)\le\min_{1\le L\le T}\big(1-q_{\min}(L)\big)^{\lfloor T/L\rfloor}.
\]
We make three observations on this result.
\begin{enumerate*}[label=(\roman*)]
    \item First, for $k_2=1$, the value used in our experiments, $w_{1,k_3}(a,b)=1-(1-a)(1-b)^{k_3-1}$, so that $\pi_\varrho(p)$ is explicit and $\pi_0(p)=\pi(p)$: the honest bound of \cref{prop:node_ejection_bounds} is recovered at $\corrreject=0$.
    For general $k_2$, $w_{k_2,k_3}(a,b)$ is computed by a recursion over the rounds of the window that tracks the last indicator and the number of detections so far.
    \item Second, the bound on $\mathbb{P}(B_T^{\mathrm{mal}}=1)$ only requires powers of the matrix $\mathbf{Q}$, whose size is at most $2(k_1+k_2k_3+1)$, and holds for all thresholds since $q_{\min}(L)$ accounts for every configuration in which a block may start.
    \item Third, this bound is never looser than the one of \cref{prop:node_ejection_bounds} when $k_2=1$ and $\corrfalseneg=0$: the proof of \cref{thm:generic_ejection_bounds} shows that ejection occurs within $T$ rounds with probability at least $1-(1-\hat\pi(p))^{M_T}$ from any configuration, hence $1-q_{\min}(T)\le(1-\hat\pi(p))^{M_T}$.
\end{enumerate*}

\begin{remark}[Measured correlations]\label{rem:markov_memory_validation}
    We evaluate \cref{ass:markov_two_rate_detection} on our \cifar experiments ($\alpha=0.5$, three seeds), taking $Z_{i,j}^t$ to be the outcome of collaborative verification on the rounds before an edge reaches \Ejected{}, and obtain $\corrreject\approx0.09$ on honest edges ($17{,}112$ transitions) and $\corrfalseneg\approx0.31$ on attacker edges ($643$ transitions, standard error $0.07$).
    Note that $\probareject$ and $\probafalsenegative$ are aggregate rates over the whole run, which include the rounds in which a neighbor is already \Suspected{} or \Ejected{}, whereas the correlations are fitted on the rounds before ejection; $\corrfalseneg$ in particular rests on few transitions, since attackers are ejected early, and should be read as indicative.
\end{remark}

\begin{remark}[Numerical instance]\label{rem:node_ejection_bounds_markov}
    With $\probareject=1.1\%$, $\probafalsenegative=20.0\%$, $(k_1,k_2,k_3)=(2,1,3)$ and $T=50$ as in \cref{rem:node_ejection_bounds}, the two bounds above evaluate to:
    \begin{center}
    \begin{tabular}{l c c}
        \toprule
        Correlation & $\mathbb{P}(A_T^{\mathrm{hon}}=1)\le$ & $\mathbb{P}(B_T^{\mathrm{mal}}=1)\le$ \\
        \midrule
        $\corrreject=\corrfalseneg=0$ (\cref{ass:two_rate_detection}) & $0.019\%$ & $<0.01\%$ \\
        $\corrreject=0.09$, $\corrfalseneg=0.31$ (measured) & $0.62\%$ & $<0.01\%$ \\
        $\corrreject=\corrfalseneg=0.3$ & $5.2\%$ & $<0.01\%$ \\
        $\corrreject=\corrfalseneg=0.6$ & $19.4\%$ & $<0.01\%$ \\
        \bottomrule
    \end{tabular}
    \end{center}
    At the measured correlations, the honest-ejection bound moves from $0.019\%$ to $0.62\%$, still well under $1\%$, and the attacker-survival bound is unchanged.
    Correlation mostly affects honest edges, as bursts of false positives make the $k_1$ consecutive rejections that open an evaluation window more likely.
\end{remark}

\section{End-to-end convergence under attack}\label{appendix:end_to_end}

This section states and proves the formal version of \cref{cor:end_to_end}, which combines the ejection bound of \cref{prop:node_ejection_bounds} with the convergence rate of \cref{thm:convergence}.

\EndToEndConvergence*

\subsection{Setting and additional assumption}

Following the assumptions of \cref{thm:convergence} that all attackers have been ejected, all quantities of that theorem are understood here when only honest nodes remain: $n:=\len{\honestset}$, average model over $\honestset$, $\losssampled$ is the average of the honest local losses, and $d$, $\rho$, $\beta(d,\probareject)$ and $C_\rho$ are those of the graph induced by the honest nodes.
We write $\mathcal{E}_{\maliciousset}:=\{(i,j): i\in\honestset,\ j\in N(i)\cap\maliciousset\}$ for the set of honest--attacker edges and $\mathcal{E}_{\honestset}:=\{(i,j): i\in\honestset,\ j\in N(i)\cap\honestset\}$ for the set of ordered pairs of honest neighbors; each such pair carries its own trust state machine.
As in the proofs of \cref{prop:spectral_gap_mixing_matrix,thm:convergence}, the indicators of \cref{ass:two_rate_detection} are taken independent across edges and rounds, and independent of the stochastic gradients, which are unbiased, and of the attackers' behavior.
Rounds are indexed from $0$, $\mathcal{F}_{t}$ denotes the history of the first $t$ rounds, and we use:
\[
    \bar\sigma^2:=\boundstochasticnoise^2+\boundheterogeneity^2,
    \qquad
    \Xi_t:=\frac{1}{n}\sum_{i\in\honestset}\norm{\model_i^t-\overline{\model}^t}^2,
    \qquad
    \Delta_t:=\losssampled\big(\overline{\model}^t\big)-\losssampled^*.
\]

\Cref{thm:convergence} describes the system once attackers are gone.
Before that, attackers may have arbitrary models accepted by their honest neighbors, so that the state reached by the honest nodes when the last attacker is ejected is not controlled by the assumptions made so far.
We therefore assume that this transient phase leaves the honest nodes in a bounded state.

\begin{assumption}[Bounded transient]\label{ass:bounded_transient}
    There exist $\Delta_{\max},\Xi_{\max}<\infty$ such that, at round $t_0=T_{\mathrm{ej}}(\delta)$ defined in \cref{cor:end_to_end_formal}, the honest models satisfy $\Delta_{t_0}\le\Delta_{\max}$ and $\Xi_{t_0}\le\Xi_{\max}$ almost surely.
\end{assumption}

The first condition holds as soon as the training loss is bounded.
Both conditions hold when the models produced during the transient phase remain in a bounded set, in which case $\Xi_{\max}$ is at most the squared diameter of this set.

\subsection{Convergence from an arbitrary starting point}

The proof of \cref{thm:convergence} relies on a uniform bound on the model drift (\cref{lem:uniform bounds}) that starts from identical models.
The following lemma removes this requirement, at the price of an additional term that depends on the initial disagreement $\Xi_{t_0}$.

\begin{lemma}[Convergence from an arbitrary start]\label{lem:restart}
    Assume that the setting of \cref{thm:convergence} holds from some round $t_0$ on, the models at round $t_0$ being arbitrary and the masks and gradient noise of rounds $t\ge t_0$ being independent of $\mathcal{F}_{t_0}$.
    Let $T'\ge1$ and consider a constant stepsize $\lrmain\le\frac{1-\rho}{5L}$ over rounds $t_0,\dots,t_0+T'-1$.
    Then
    \[
        \frac{1}{n}\sum_{i\in\honestset}\frac{1}{T'}\sum_{t=t_0}^{t_0+T'-1}\condesp{\norm{\nabla\losssampled(\model_i^t)}^2}{\mathcal{F}_{t_0}}
        \le \frac{4\widetilde\Delta_{t_0}}{\lrmain T'} + \frac{\lrmain}{n}\Gamma_1 + \lrmain^2\Gamma_2 + \frac{16L^2\,\Xi_{t_0}}{(1-\rho)T'},
    \]
    where $\widetilde\Delta_{t_0}:=\Delta_{t_0}+\frac{32L\beta(d,\probareject)}{n(1-\rho)}\Xi_{t_0}$ and
    \[
        \Gamma_1:=L\Big(4\boundstochasticnoise^2+\big(24+384C_\rho\big)\beta(d,\probareject)\,\bar\sigma^2\Big),
        \qquad
        \Gamma_2:=48L^2C_\rho\,\bar\sigma^2.
    \]
\end{lemma}

\begin{proof}
    All expectations are conditional on $\mathcal{F}_{t_0}$; for $t>t_0$ we still write $\Xi_t$ and $\Delta_t$ for the conditional expectations of these quantities.
    Since the masks of rounds $t\ge t_0$ are independent of $\mathcal{F}_{t_0}$ and of the gradient noise, \cref{lem:EL lemma 1} applies conditionally.

    \textbf{Gradient drift.}
    Let $\overline{g}^t:=\frac1n\sum_{i\in\honestset}g_i^t$. Since the average minimizes the sum of squared distances, and decomposing $g_i^t-\nabla\losssampled(\overline{\model}^t)$ into $g_i^t-\nabla\locallosssampled{i}(\model_i^t)$, $\nabla\locallosssampled{i}(\model_i^t)-\nabla\locallosssampled{i}(\overline{\model}^t)$ and $\nabla\locallosssampled{i}(\overline{\model}^t)-\nabla\losssampled(\overline{\model}^t)$,
    \begin{equation}\label{eq:restart gradient drift}
        \frac{1}{n}\sum_{i\in\honestset}\E\norm{g_i^t-\overline{g}^t}^2
        \le \frac{1}{n}\sum_{i\in\honestset}\E\norm{g_i^t-\nabla\losssampled(\overline{\model}^t)}^2
        \le 3\bar\sigma^2+3L^2\,\Xi_t,
    \end{equation}
    using the bounded variance, $L$-smoothness and bounded heterogeneity.

    \textbf{Model drift.}
    Since $\model_i^{t+\half}-\overline{\model}^{t+\half}=(\model_i^t-\overline{\model}^t)-\lrmain(g_i^t-\overline{g}^t)$, \cref{lem:EL lemma 1}(b), Young's inequality with a parameter $u>0$ and \cref{eq:restart gradient drift} give
    \[
        \Xi_{t+1}
        \le \rho(1+u)\,\Xi_t+\rho\Big(1+\frac1u\Big)\lrmain^2\big(3\bar\sigma^2+3L^2\,\Xi_t\big).
    \]
    If $\rho=0$ then $\Xi_{t+1}=0$. Otherwise, take $u=\frac{1-\rho}{2\rho}$, so that $\rho(1+u)=\frac{1+\rho}{2}$ and $\rho(1+\frac1u)=\frac{\rho(1+\rho)}{1-\rho}$.
    As $\lrmain L\le\frac{1-\rho}{5}$ and $\rho(1+\rho)\le2$, we have $\frac{3\rho(1+\rho)}{1-\rho}\lrmain^2L^2\le\frac{1-\rho}{4}$, hence
    \[
        \Xi_{t+1}\le\frac{3+\rho}{4}\,\Xi_t+\frac{3\rho(1+\rho)}{1-\rho}\lrmain^2\bar\sigma^2.
    \]
    Unrolling this recursion from $t_0$ and using $\rho(1+\rho)\le(1-\rho)^2C_\rho$ (\cref{eq:c rho definition}), we get, for all $s\ge0$,
    \begin{equation}\label{eq:restart model drift}
        \Xi_{t_0+s}\le\Big(\frac{3+\rho}{4}\Big)^{s}\,\Xi_{t_0}+12C_\rho\lrmain^2\bar\sigma^2,
        \qquad\text{thus}\qquad
        \frac{1}{T'}\sum_{t=t_0}^{t_0+T'-1}\Xi_t\le\frac{4\,\Xi_{t_0}}{(1-\rho)T'}+12C_\rho\lrmain^2\bar\sigma^2.
    \end{equation}

    \textbf{Descent.}
    The one-step inequality~\eqref{eq:descent lemma EL} holds at every round, irrespective of the initial models, for $\lrmain\le\frac{1}{2L}$~\cite[Lemma 3]{devosEpidemicLearningBoosting2023}.
    Its last term is bounded with \cref{lem:EL lemma 1}(c) and \cref{eq:restart gradient drift}:
    \[
        \E\norm{\overline{\model}^{t+1}-\overline{\model}^{t+\half}}^2
        \le\frac{\beta(d,\probareject)}{n}\Big(2\,\Xi_t+2\lrmain^2\big(3\bar\sigma^2+3L^2\,\Xi_t\big)\Big).
    \]
    Plugging this into \cref{eq:descent lemma EL}, and using $\lrmain L\le1$ to bound $24L^3\beta(d,\probareject)\lrmain/n$ by $24L\beta(d,\probareject)/(n\lrmain)$, we obtain
    \[
        \frac{1}{n}\sum_{i\in\honestset}\E\norm{\nabla\losssampled(\model_i^t)}^2
        \le\frac{4}{\lrmain}\big(\Delta_t-\Delta_{t+1}\big)
        +\Big(4L^2+\frac{32L\beta(d,\probareject)}{n\lrmain}\Big)\Xi_t
        +\frac{\lrmain}{n}L\big(4\boundstochasticnoise^2+24\beta(d,\probareject)\,\bar\sigma^2\big).
    \]
    Averaging over $t\in\{t_0,\dots,t_0+T'-1\}$, the first term telescopes to at most $\frac{4\Delta_{t_0}}{\lrmain T'}$ since $\Delta_{t_0+T'}\ge0$, and \cref{eq:restart model drift} bounds the second one by
    \[
        \frac{16L^2\,\Xi_{t_0}}{(1-\rho)T'}
        +\frac{4}{\lrmain T'}\cdot\frac{32L\beta(d,\probareject)}{n(1-\rho)}\Xi_{t_0}
        +48L^2C_\rho\lrmain^2\bar\sigma^2
        +\frac{\lrmain}{n}\,384LC_\rho\beta(d,\probareject)\,\bar\sigma^2.
    \]
    Collecting terms gives the claim.
\end{proof}

With identical initial models ($\Xi_{t_0}=0$), \cref{lem:restart} has the same form as \cref{eq:telescoping sum}: the transient only enters through the effective initial gap $\widetilde\Delta_{t_0}$ and a term of order $1/T'$.

\subsection{Formal statement and proof}

\begin{corollary}[Formal version of \cref{cor:end_to_end}]\label{cor:end_to_end_formal}
    Fix $\delta\in(0,1)$ and $\epsilon>0$.
    Let \cref{ass:two_rate_detection} hold with $0<\probafalsenegative<1$, let $k_2=1$ and $\hat p:=\hat\pi(1-\probafalsenegative)=\probafalsenegative\,\pi(1-\probafalsenegative)$ with $\pi$ as in \cref{prop:node_ejection_bounds}, and define
    \[
        T_{\mathrm{ej}}(\delta):=(k_1+k_3+1)\left\lceil\frac{\ln\!\big(2\len{\mathcal{E}_{\maliciousset}}/\delta\big)}{\ln\!\big(1/(1-\hat p)\big)}\right\rceil.
    \]
    Assume that the other hypotheses of \cref{thm:convergence} hold for the honest nodes, that honest neighbors always remain \Trusted{} (see \cref{rem:end_to_end_scope}), and that \cref{ass:bounded_transient} holds with $\widetilde\Delta>0$ defined below.
    Let $\Gamma_1$ and $\Gamma_2$ be as in \cref{lem:restart}, $\widetilde\Delta:=\Delta_{\max}+\frac{32L\beta(d,\probareject)}{n(1-\rho)}\Xi_{\max}$, and
    \[
        T_{\mathrm{cv}}(\delta,\epsilon):=\left\lceil\max\left\{
            \frac{24\big(5L\widetilde\Delta+4L^2\,\Xi_{\max}\big)}{(1-\rho)\,\delta\epsilon},\
            \frac{576\,\widetilde\Delta\,\Gamma_1}{n\,\delta^2\epsilon^2},\
            \frac{96\sqrt{3}\,\widetilde\Delta\sqrt{\Gamma_2}}{(\delta\epsilon)^{3/2}}
        \right\}\right\rceil.
    \]
    Run \cref{alg:main} for $T:=T_{\mathrm{ej}}(\delta)+T_{\mathrm{cv}}(\delta,\epsilon)$ rounds, using from round $T_{\mathrm{ej}}(\delta)$ on the constant stepsize
    \[
        \lrmain:=\min\left\{\frac{1-\rho}{5L},\ \sqrt{\frac{4n\widetilde\Delta}{\Gamma_1\,T_{\mathrm{cv}}(\delta,\epsilon)}},\ \left(\frac{4\widetilde\Delta}{\Gamma_2\,T_{\mathrm{cv}}(\delta,\epsilon)}\right)^{\!\frac13}\right\}.
    \]
    where a term with a null denominator is read as $+\infty$.
    Then
    \[
        \mathbb{P}\!\left(\frac{1}{n}\sum_{i\in\honestset}\frac{1}{T_{\mathrm{cv}}(\delta,\epsilon)}\sum_{t=T_{\mathrm{ej}}(\delta)}^{T-1}\norm{\nabla\losssampled(\model_i^t)}^2\le\epsilon\right)\ \ge\ 1-\delta.
    \]
    In particular, $T_{\mathrm{ej}}(\delta)=\mathcal{O}\big(\log(\len{\mathcal{E}_{\maliciousset}}/\delta)\big)$ and $T_{\mathrm{cv}}(\delta,\epsilon)=\mathcal{O}\big(\widetilde\Delta\,\Gamma_1/(n\delta^2\epsilon^2)\big)$ as $\delta\epsilon\to0$.
\end{corollary}

\begin{proof}
    Write $t_0:=T_{\mathrm{ej}}(\delta)$, $T':=T_{\mathrm{cv}}(\delta,\epsilon)$, and let $X$ be the average squared gradient norm appearing in the statement.

    \textbf{Step 1: ejection.}
    \Cref{prop:node_ejection_bounds} applies to the trust state machine of each edge $(i,j)\in\mathcal{E}_{\maliciousset}$: the probability that $i$ has not ejected $j$ within the first $t_0$ rounds is at most $(1-\hat p)^{\lfloor t_0/(k_1+k_3+1)\rfloor}\le\delta/(2\len{\mathcal{E}_{\maliciousset}})$ by the choice of $t_0$.
    Let $\Omega_{\mathrm{ej}}$ be the event that every attacker has been ejected by all its honest neighbors within the first $t_0$ rounds. By the union bound over $\mathcal{E}_{\maliciousset}$, $\mathbb{P}(\Omega_{\mathrm{ej}}^c)\le\delta/2$.

    \textbf{Step 2: convergence after ejection.}
    The event $\Omega_{\mathrm{ej}}$ belongs to $\mathcal{F}_{t_0}$.
    On $\Omega_{\mathrm{ej}}$, honest nodes no longer aggregate any attacker model from round $t_0$ on, and the masks and gradient noise of these rounds are independent of $\mathcal{F}_{t_0}$.
    \Cref{lem:restart} thus applies on $\Omega_{\mathrm{ej}}$. Its right-hand side is non-decreasing in $\Delta_{t_0}$ and $\Xi_{t_0}$, which are bounded by $\Delta_{\max}$ and $\Xi_{\max}$ (\cref{ass:bounded_transient}), so that, on $\Omega_{\mathrm{ej}}$,
    \[
        \condesp{X}{\mathcal{F}_{t_0}}\le R:=\frac{4\widetilde\Delta}{\lrmain T'}+\frac{\lrmain}{n}\Gamma_1+\lrmain^2\Gamma_2+\frac{16L^2\,\Xi_{\max}}{(1-\rho)T'}.
    \]

    \textbf{Step 3: stepsize.}
    The inverse of the stepsize is at most the sum of the inverses of the three terms of the minimum defining it, hence
    \[
        R\le\frac{20L\widetilde\Delta+16L^2\,\Xi_{\max}}{(1-\rho)T'}
        +4\sqrt{\frac{\widetilde\Delta\,\Gamma_1}{nT'}}
        +2\left(\frac{16\,\widetilde\Delta^2\,\Gamma_2}{T'^2}\right)^{\!\frac13}.
    \]
    Each of the three terms in the definition of $T'$ makes the corresponding term above at most $\delta\epsilon/6$, so $R\le\delta\epsilon/2$.

    \textbf{Step 4: tail bound.}
    As $X\ge0$, Markov's inequality applied conditionally on $\mathcal{F}_{t_0}$ gives
    \[
        \mathbb{P}\big(\{X>\epsilon\}\cap\Omega_{\mathrm{ej}}\big)
        =\E\big[\mathbb{I}_{\Omega_{\mathrm{ej}}}\,\mathbb{P}(X>\epsilon\mid\mathcal{F}_{t_0})\big]
        \le\frac{1}{\epsilon}\,\E\big[\mathbb{I}_{\Omega_{\mathrm{ej}}}\,\condesp{X}{\mathcal{F}_{t_0}}\big]
        \le\frac{R}{\epsilon}\le\frac{\delta}{2}.
    \]
    Therefore $\mathbb{P}(X>\epsilon)\le\mathbb{P}(\Omega_{\mathrm{ej}}^c)+\mathbb{P}(\{X>\epsilon\}\cap\Omega_{\mathrm{ej}})\le\delta$.
\end{proof}

\begin{remark}[Scope and extensions]\label{rem:end_to_end_scope}
    \emph{(i) Honest neighbors.}
    \Cref{thm:convergence,cor:end_to_end_formal} assume that honest neighbors always remain \Trusted{}: the masks $S^t$ model neither the systematic rejection of a \Suspected{} neighbor nor its ejection.
    \sys behaves exactly like this until an honest neighbor becomes \Suspected{}, which takes $k_1$ consecutive false rejections.
    By a union bound, this happens within $T$ rounds with probability at most $\len{\mathcal{E}_{\honestset}}\,T\,\probareject^{k_1}$, so the guarantee of \cref{cor:end_to_end_formal} holds for \sys with probability at least $1-\delta-\len{\mathcal{E}_{\honestset}}\,T\,\probareject^{k_1}$.
    This extra term grows with $T$ but decreases geometrically with $k_1$.

    \emph{(ii) Other ejection bounds.}
    The proof only needs a bound on the probability that an attacker edge survives $t_0$ rounds.
    $T_{\mathrm{ej}}(\delta)$ can thus be replaced by $\ell\left\lceil\ln(2\len{\mathcal{E}_{\maliciousset}}/\delta)/\ln(1/(1-q_{\min}(\ell)))\right\rceil$ for any block length $\ell$ with $0<q_{\min}(\ell)<1$ (\cref{prop:markov_ejection_bounds}), which holds for all thresholds and when attacker detections are correlated (\cref{ass:markov_two_rate_detection}).

    \emph{(iii) Topology.}
    \Cref{thm:convergence} requires a connected regular graph. Here, this applies to the graph induced by the honest nodes, which is restrictive: removing attackers from a regular graph generally breaks regularity.
\end{remark}

\newpage
\end{document}